\documentclass{article}

\usepackage[main, final, nonatbib]{neurips_2026}

\usepackage[numbers]{natbib}
\usepackage[utf8]{inputenc} 
\usepackage[T1]{fontenc}    
\usepackage{hyperref}       
\usepackage{url}            
\usepackage{booktabs}       
\usepackage{amsmath}        
\usepackage{amssymb}        
\usepackage{amsthm}         
\usepackage{amsfonts}       
\usepackage{nicefrac}       
\usepackage{microtype}      
\usepackage{xcolor}         
\usepackage{tikz}           
\usepackage{algorithm}      
\usepackage{algpseudocode}  
\algdef{SE}[WITH]{With}{EndWith}[1]{\textbf{with} #1 \textbf{do}}{\textbf{end with}}
\usepackage{bbm}
\usepackage{wrapfig}
\usepackage{subcaption}
\usepackage{pifont}
\usepackage{tcolorbox}
\usepackage{listings}
\usepackage{placeins}
\tcbuselibrary{listings}
\usetikzlibrary{arrows.meta}

\newcommand{\cmark}{\ding{51}}
\newcommand{\xmark}{\ding{55}}

\newtcblisting{qabox}[1][]{
  title={#1},
  fonttitle=\small,
  fontupper=\small,
  left=2mm,
  right=2mm,
  top=1mm,
  bottom=1mm,
  listing only,
  listing options={
    breaklines=true,
    breakatwhitespace=false,
    breakindent=0pt,
    columns=fullflexible,
    keepspaces=true,
    showstringspaces=false,
    basicstyle=\ttfamily\small
  },
}

\newtheorem{definition}{Definition}
\newtheorem{axiom}{Axiom}
\newtheorem{lemma}{Lemma}
\newtheorem{theorem}{Theorem}
\newtheorem{proposition}{Proposition}

\DeclareMathOperator*{\argmin}{arg\,min}

\title{Generalized Priority-Aware Shapley Value}

\author{
    Kiljae Lee\\
    The Ohio State University\\
    \texttt{lee.10428@osu.edu}
    \And
    Ziqi Liu\\
    Carnegie Mellon University\\
    \texttt{ziqiliu2@andrew.cmu.edu}
    \AND
    Weijing Tang\\
    Carnegie Mellon University\\
    \texttt{weijingt@andrew.cmu.edu}
    \And
    Yuan Zhang\thanks{Corresponding author.}\\
    The Ohio State University\\
    \texttt{yzhanghf@stat.osu.edu}
}

\usepackage{tabu}

\theoremstyle{plain}

\usepackage{accents}

\renewcommand{\hat}{\widehat}
\renewcommand{\tilde}{\widetilde}

\newcommand{\pr}{\mathbb{P}}

\usepackage{xcolor}

\definecolor{mccolor}{rgb}{0.3010, 0.7450, 0.9330}
\definecolor{edgecolor}{rgb}{0, 0.5, 0}
\definecolor{respcolor}{rgb}{0.4940, 0.1840, 0.5560}
\definecolor{subscolor}{rgb}{0.8500, 0.3250, 0.0980}

\usepackage{hyperref}
\hypersetup{
    colorlinks=true,
    citecolor=blue,
    citebordercolor=red
}

\begin{document}

\maketitle

\begin{abstract}
Shapley value and its priority-aware extensions are widely used for valuation in machine learning, but existing methods require pairwise priority to be binary and acyclic, a restriction spectacularly violated in real-data examples such as aggregated human preferences and multi-criterion comparisons.
We introduce the \emph{generalized priority-aware Shapley value} (GPASV), a random order value defined on arbitrary directed weighted priority graphs, in which pairwise edges penalize rather than forbid order violations.
GPASV covers a range of classical models as boundary cases.
We establish GPASV through an axiomatic characterization, develop the associated computational methods, and introduce a priority sweeping diagnostic extending PASV's.
We apply GPASV to LLM ensemble valuation on the cyclic Chatbot Arena preference graph, illustrating that {\bf priority-aware valuation is not a one-button operation}: different balances of pairwise graph priority versus individual soft priority produce substantively different valuations of the same data.
\end{abstract}

\section{Introduction}
\label{sec::intro}

\begin{wrapfigure}{r}{0.32\linewidth}
\centering
\setlength{\intextsep}{2pt}
\setlength{\columnsep}{6pt}
\setlength{\abovecaptionskip}{0pt}
\setlength{\belowcaptionskip}{0pt}
\vspace{-1.0em}
\hspace{-1.6em}
\resizebox{\linewidth}{!}{%
\definecolor{panelA}{HTML}{a1c9f4}
\definecolor{panelB}{HTML}{ffb482}
\definecolor{panelC}{HTML}{8de5a1}
\definecolor{panelD}{HTML}{ff9f9b}
\definecolor{cycleRed}{HTML}{c0392b}

\begin{tikzpicture}[
    >=Stealth,
    player/.style={
        circle,
        draw,
        thick,
        inner sep=0pt,
        minimum size=4.8mm,
        font=\scriptsize,
        fill=blue!10
    },
    smallplayer/.style={player, minimum size=2.5mm},
    medplayer/.style={player, minimum size=4.2mm},
    largeplayer/.style={player, minimum size=6mm},
    xlplayer/.style={player, minimum size=7.5mm},
    edge/.style={->, thick},
    edgedash/.style={->, thick, dashed}
]

\def\xgap{1.2}        
\def\ygap{0.4}        
\def\panelgap{2}    
\def\rowgap{2}      
\def\laby{-0.95}      

\begin{scope}[xshift=0cm, yshift=\rowgap cm, player/.append style={fill=panelA}]
    \node[medplayer] (a1) at (0,\ygap) {1};
    \node[medplayer] (a2) at (\xgap,\ygap) {2};
    \node[medplayer] (a3) at (0,-\ygap) {3};
    \node[medplayer] (a4) at (\xgap,-\ygap) {4};

    \draw[edge] (a1) -- (a2);
    \draw[edge] (a3) -- (a2);
    \draw[edge] (a3) -- (a4);

    \node[font=\scriptsize] at (0.5*\xgap,\laby) {(a) Hard priority};
\end{scope}

\begin{scope}[xshift=\panelgap cm, yshift=\rowgap cm, player/.append style={fill=panelB}]
    \node[medplayer]    (b1) at (0,\ygap) {1};
    \node[xlplayer]     (b2) at (\xgap,\ygap) {2};
    \node[largeplayer]  (b3) at (0,-\ygap) {3};
    \node[smallplayer]  (b4) at (\xgap,-\ygap) {4};

    \draw[edgedash] (b1) -- (b2);
    \draw[edgedash] (b3) -- (b2);
    \draw[edgedash] (b1) -- (b4);
    \draw[edgedash] (b3) -- (b4);

    \node[font=\scriptsize] at (0.5*\xgap,\laby) {(b) Soft priority};
\end{scope}

\begin{scope}[xshift=0cm, yshift=0cm, player/.append style={fill=panelC}]
    \node[medplayer]    (c1) at (0,\ygap) {1};
    \node[xlplayer]     (c2) at (\xgap,\ygap) {2};
    \node[largeplayer]  (c3) at (0,-\ygap) {3};
    \node[smallplayer]  (c4) at (\xgap,-\ygap) {4};

    \draw[edge] (c1) -- (c2);
    \draw[edge] (c3) -- (c2);
    \draw[edge] (c3) -- (c4);

    \node[font=\scriptsize] at (0.5*\xgap,\laby) {(c) PASV};
\end{scope}

\begin{scope}[xshift=\panelgap cm, yshift=0cm, player/.append style={fill=panelD}]
    \node[medplayer]    (d1) at (0,\ygap) {1};
    \node[xlplayer]     (d2) at (\xgap,\ygap) {2};
    \node[largeplayer]  (d3) at (0,-\ygap) {3};
    \node[smallplayer]  (d4) at (\xgap,-\ygap) {4};

    \draw[->, line width=1.5pt] (d1) -- (d2);
    \draw[->, line width=0.3pt] (d3) -- (d2);
    \draw[->, line width=2.3pt, cycleRed] (d3) -- (d4);
    \draw[->, line width=0.5pt, cycleRed] (d1) -- (d3);
    \draw[->, line width=1pt, cycleRed] (d4) -- (d1);

    \node[font=\scriptsize] at (0.5*\xgap,\laby) {(d) GPASV};
\end{scope}

\end{tikzpicture}%
}
\caption{Illustration of priority structures.  Dashed arrows in (b): \cite{nowak1995axiomatizations, zheng2025rethinking} require special precedence structures.}
\label{fig:priority}
\vspace{-1.0em}
\end{wrapfigure}
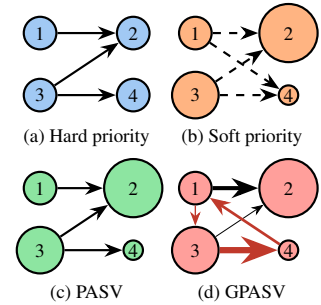

Shapley value \citep{shapley1953value} and its variants are widely used in machine learning to accredit content contributors, including training examples, data providers, features, or model components \citep{ghorbani2019data,lundberg2017unified,wang2023data}.
Their appeal lies in being both principled (uniquely determined by a small set of axioms) and model-agnostic.
A central axiom is symmetry: two players who contribute identical contents should receive the same credit.
While this is reasonable on its face, its key vulnerability is the {\bf copier attack}: symmetry rewards copiers equally to originators, when rationality says all credit should go to the originators.
Modern AI/ML problems contain many examples where for similar or related reasons, contributors should be valued differently for reasons beyond their contributed contents: data lineage, causal precedence among features, differing trust or cost across data providers.
Recent work has incorporated this contextual information into Shapley-type valuations \citep{weber1988probabilistic,derks2005new}, including precedence Shapley values (PSV) \citep{faigle1992shapley}, weighted Shapley values (WSV) \citep{kalai1987weighted,nowak1995axiomatizations}, machine-learning adaptations \citep{frye2020asymmetric,zheng2025rethinking}, and most recently the priority-aware Shapley value (PASV) \citep{lee2026priority}.

A common assumption across all these works is that the pairwise priority relations can be captured by a {\bf directed acyclic graph (DAG)}.
The meaning of DAG assumption is two-fold:
(i) it has {\bf no cycles} (e.g., $i\to j\to k\to i$);
(ii) precedences are {\bf unweighted};
both aspects are violated in many modern AI/ML applications.
A prominent example of cyclic priority is aggregated human preferences, where \emph{Condorcet cycles} are ubiquitous  \citep{condorcet1785essai,black1958theory,fishburn1977condorcet,liu2025statistical};
multi-criterion model comparison is another example in similar spirits \citep{xu2025investigating}.
Weighted priority arises whenever the strength of evidence varies, encompassing pairwise win counts in tournament analysis, Bradley-Terry log-odds \citep{bradley1952rank}, ELO ratings \citep{glickman1999parameter}, and Plackett-Luce-style worth scores \citep{plackett1975analysis,luce1959individual}; all of these carry magnitude that a DAG discards.
Some applications, such as Chatbot Arena \citep{zheng2023judging,chiang2024chatbot}, present cyclic, weighted priority graphs.

We propose {\bf generalized priority-aware Shapley value (GPASV)}, a valuation method that simultaneously handles cyclic and weighted priority graphs.
Technically speaking, Shapley-style valuation methods all build on permutations of the players (see Section \ref{section::preliminaries}).
Within this language, the DAG assumption amounts to rigidly restricting to admissible permutations.
Our GPASV instead, for each permutation, counts its total amount of precedence violations and inversely weights its probability accordingly, thereby naturally handling cycles and weighted pairwise priorities.

Apart from pairwise priority, represented by the priority graph, GPASV also incorporates the player-level ``\emph{soft priority}'', which encodes individual information such as trust, cost, or compliance risk.
While the idea of this component was inherited from the WSV-to-PASV line of prior work, integrating it with our ``upgraded'' priority-graph component requires rather nontrivial adaptation.

\subsection{Our Contributions}
	\paragraph{Method (core contribution): GPASV handles cyclic and weighted priority graphs while incorporating individual soft priority.}
	GPASV is the first valuation method that simultaneously handles cyclic pairwise priorities, weighted pairwise priorities, and individual soft priorities, while still containing the classical Shapley value, PSV, WSV, and PASV as special cases.

	\paragraph{Theory: a principled, expressive extension of PASV.}
    GPASV admits an axiomatic characterization that generalizes PASV’s, with GSCF fixing the canonical stage-wise form (Section \ref{subsec::method::axioms}). 
    Furthermore, on cyclic graphs GPASV admits limiting distributions that no PASV can express (Section \ref{subsec::method::limit}), marking GPASV as a highly nontrivial extension rather than a wrapper around PASV.

    \paragraph{Scalable computation and acceleration methods.}
    We derive a GPASV-specific local adjacent-swap MH ratio (Proposition \ref{prop:adjacent-swap-ratio}) and design a stage-wise greedy initialization for cyclic graphs (Algorithm \ref{alg:gpasv-init}), then integrate these sampling components with direct Monte Carlo estimation, utility caching (Section \ref{sec:utility-reuse}), SNIS reuse for priority sweeping (Section \ref{sec:snis-reuse}), and a surrogate-assisted estimator adapted from regression-based semi-value methods (Section \ref{sec:two-stage-surrogate}).
    While some acceleration components build on standard Monte Carlo and regression-surrogate ideas, their adaptation to the GPASV permutation distribution and priority-sweeping workflow significantly enhances scalability.

	\paragraph{Diagnostic: priority sweeping with a new hard-priority axis.}
	We extend PASV's priority sweeping diagnostic to GPASV.
    Beyond PASV's soft priority sweep, GPASV's weighted graph admits a new hard priority sweep over $\beta$.
    Since the balance between the two priority sources is the evaluator's preference rather than a tuning parameter, GPASV reports a family of valuations along the sweep together with sensitivity diagnostics (Section \ref{subsec::method::sweep}).
    
	\paragraph{Validation and application.}
	Extensive simulations validate GPASV's accuracy and confirm our theoretical predictions.
    A large-scale experiment on LLM ensemble valuation on MT-Bench and the cyclic Chatbot Arena preference graph demonstrates GPASV's practicality and scalability.

\section{Preliminaries}
\label{section::preliminaries}

\subsection{From Shapley Value to Random Order Values (ROV)}
\label{subsec::prelim::ROV}

{\bf Shapley value (SV).}
Let $[n]=\{1,\dots,n\}$ be a set of players.
For any $S\subseteq[n]$, let $U(S)$ be the revenue earned by the joint work of $S$.
We call $U$ a \emph{utility function} with $U(\emptyset)=0$.
The Shapley value \citep{shapley1953value}
\begin{equation}
    \nu_i(U)
    :=
    \sum_{S\subseteq[n]\setminus\{i\}}
    \frac{|S|!(n-|S|-1)!}{n!}
    \Bigl\{U(S\cup\{i\})-U(S)\Bigr\}.
    \label{eq:shapley-subset-prelim}
\end{equation}
is a payoff method that uniquely satisfies four widely-desired axioms (see Appendix~\ref{app:rov-axioms}).

{\bf Random order values (ROV).}
Shapley value is prone to {\bf copier attack} -- an attacker $i$ can unfairly split $j$'s pay by simply copying her.
To defend against this attack, it is essential to consider the {\bf order/rank} among players, thus if $i$ consulted $j$'s work, then $j$ should always be present in the prefix $S$ in \eqref{eq:shapley-subset-prelim} when evaluating $i$'s contribution.
This leads to a generic notion: {\bf random order value (ROV)}.
\begin{definition}[{\bf Random Order Value (ROV)}]
\label{def:rov}
    Let $\pi=(\pi_1,\dots,\pi_n)$ be a permutation of $[n]$ and $\pi^i$ be the set of players located before $i$ in $\pi$.
    The random order value (ROV) of $i$ is defined as
    \begin{equation}
        \nu_i^p(U)
        :=
        \mathbb{E}_{\pi\sim p}
        \big[U(\pi^i\cup\{i\})-U(\pi^i)\big],
        \label{eq:random-order-value-prelim}
    \end{equation}
    where $p$ is a distribution over $\Pi$, the set of all permutations of $[n]$.
\end{definition}
Different choices of $p$ yield different priority-aware extensions of Shapley value.
Examples include: 
(i) Shapley value, with $p=\text{Uniform}(\Pi)$;
(ii) precedence Shapley value (PSV) \cite{faigle1992shapley}, with $p$ uniform over a set of ``permitted permutations'' induced by a directed acyclic graph (DAG)  encoding hard, pairwise priority;
(iii) weighted Shapley value (WSV) \cite{kalai1987weighted,nowak1995axiomatizations} incorporates individual-level weights to encode soft priority.
Going forward, we will describe each method by its $p$, without repeating \eqref{eq:random-order-value-prelim}.

\subsection{Priority-Aware Shapley Value (PASV)}
\label{subsec::prelim::PASV}

{\bf Hard vs. soft priority.}
Priority information available to the user typically comes in two forms.

\emph{Hard priority} specifies pairwise constraints that must be respected at all times, e.g., a causal ancestor must precede its descendant.
Hard priority is encoded by a directed acyclic graph (DAG) $G=([n],\mathcal{E})$, where an edge $(i,j)\in {\cal E}$ means that $i$ must appear before $j$ in any $\pi$.

\emph{Soft priority}, in contrast, expresses individual-level preferences without forbidding any order; instead, it adjusts the ordering within the boundary permitted by the hard priority.
Players who are more trusted, less costly, or carry less legal/compliance risk than others should be rewarded by an earlier spot in $\pi$, thus face less peer competition when evaluated.
Soft priority is encoded by individual weights on each player: $\lambda=(\lambda_1,\ldots,\lambda_n)$, where $\lambda_i>0$.

{\bf Priority-aware Shapley value (PASV).}
\cite{lee2026priority} proposed PASV to combine hard and soft priority in one ROV.
For $\pi\in\Pi$ and prefix $S_t:=\{\pi_1,\ldots,\pi_t\}$, let $\max(S_t)$ denote the set of players $i \in S_t$ such that no $j \in S_t$ is a descendant of $i$ in $G$; these are the elements admissible to be $\pi_t$.
PASV is
\begin{equation}
	p^{(\lambda,G)}(\pi)
	\propto
	\mathbbm{1}_{[\pi\in\Pi^G]}\cdot 
	\prod_{t=1}^{n}
	\frac{\lambda_{\pi_t}\,|\max(S_t)|}{\sum_{k\in\max(S_t)}\lambda_k},
	\label{eq:pasv-prelim}
\end{equation}
where $\Pi^G:=\{\pi\in\Pi: \text{$i$ appears before $j$ in $\pi$ for all }(i,j)\in\mathcal{E}\}$ is the set of permutations consistent with $G$.
In \eqref{eq:pasv-prelim}, the indicator $\mathbbm{1}_{[\pi\in\Pi^G]}$ enforces hard priority.
The second part of \eqref{eq:pasv-prelim} uses a \emph{stage-wise} formulation to encode soft priority $\lambda$.
This part can be roughly understood as follows: starting from position $t=n$ down to $t=1$, $\pi_t$ is drawn from the admissible set $\max(S_t)$ with probability proportional to $\lambda_{\pi_t}$; the resulting $\pi$'s probability is then rescaled by the factor $\prod_{t=1}^n |\max(S_t)|$ --- this rescaling is essential for PASV to satisfy its axioms (see Remark 3.3 in \cite{lee2026priority}).
PASV recovers SV, PSV and WSV as special cases under proper choices of $\lambda$ and $G$.

\section{Our Method}
\label{section::method}

\subsection{Limitations of PASV and Motivation for a Generalization}
\label{subsec::method::motivation}

PASV encodes hard priority as a binary DAG, which can be restrictive in modern AI/ML applications.

{\bf Hard priority may not be that ``hard'' in practice.}
In \eqref{eq:pasv-prelim}, every edge of $G$ is enforced as an absolute constraint: any $\pi$ violating even one edge is ruled out.
But real pairwise priority is often a matter of degree rather than black/white.
For example, NeurIPS 2026 draws a line at March 1, 2026: papers appearing before are treated as prior work, those after as concurrent \citep{neurips2026}, even though two papers posted one day before/after carry nearly identical priority.
Similarly, a domain expert may believe feature $i$ causally impacts feature $j$ based on informative but non-determining evidence.
In both cases, PASV discards magnitude information and becomes sensitive to borderline edges.

{\bf Pairwise priority may be cyclic.}
PASV requires $G$ to be acyclic.
But cyclic pairwise signals arise naturally when priority is aggregated across a population or across criteria.
In aggregated human preferences, model $i$ may beat $j$, $j$ beats $k$, and yet $k$ beats $i$ -- the classical Condorcet cycle \citep{black1958theory,fishburn1977condorcet}, shown to occur with high probability in LLM preference networks \citep{liu2025statistical}.
Multi-criterion comparison can be similarly cyclic: $A$ beats $B$ on criterion 1, $B$ beats $C$ on criterion 2, $C$ beats $A$ on criterion 3.
Forcing such data into a DAG requires unfairly favoring certain players or biasing criteria.

\subsection{Generalized Priority-Aware Shapley Value (GPASV)}
\label{subsec::method::GPASV}

To address the two limitations of PASV in Section \ref{subsec::method::motivation}, we first replace the binary DAG $G$ with a weighted directed graph $(\omega_{ij})_{i,j\in[n], i\neq j}$, where each $\omega_{ij}\geq 0$ encodes the strength of the pairwise priority ``$i$ should precede $j$'':
$\omega_{ij}=0$ means no such priority; we set $\omega_{ii}\equiv0$ throughout.
This convention directly accommodates common forms of user data -- pairwise win counts, Bradley-Terry log-odds, ELO differences, or expert-assigned confidence scores -- without ad-hoc transformations.
Note that $\omega_{ij}$ and $\omega_{ji}$ are independent: in encoding a DAG-style precedence ``$i\prec j$'' at most one of them is nonzero, while in encoding pairwise comparison records (e.g., $i$ and $j$'s head-to-head wins), both can be positive.
Unlike PASV, we impose no acyclicity requirement on $\omega$.
In practice, we recommend rescaling the input edge weights and soft priority scores as a pre-processing step, since the temperature $\beta$ below (and its $\lambda$-side analogue $\alpha$ in Section~\ref{sec:application}) acts on their numerical scales.

For simplicity, we start with a simplified case where $\lambda_i\equiv 1$.
For any $\pi$, its \emph{total violation} against $\omega$ is
$
    V_\omega(\pi)
    :=
    \sum_{i\neq j}\omega_{ij}\cdot\mathbbm{1}_{[\text{$j$ appears before $i$ in $\pi$}]}.
$
A Gibbs-style distribution that penalizes total violation is
\begin{equation}
    p^{\omega}_\beta(\pi)
    \propto
    \exp\{-\beta\cdot V_\omega(\pi)\},
    \label{eq:mallows-form}
\end{equation}
where $\beta\geq 0$ is a \emph{temperature} controlling the overall penalty strength.
When $\omega$ is a DAG and $\beta\to\infty$, we have $p^{\omega}_\beta(\pi)=$~Uniform$(\Pi^G)$, recovering PSV, a special case of PASV when $\lambda_i\equiv$~constant.

Next, we incorporate $\lambda$ stage-wise, following PASV. Define the stage-wise violation cost
$
    V_\omega(k;S_t)
    :=
    \sum_{j\in S_t}\omega_{kj},
$
i.e., the total strength of priorities that $k$ would violate by being placed at position $t$. The {\bf Generalized Priority-Aware Shapley Value (GPASV)} is defined as:
\begin{equation}
    p^{(\lambda,\omega)}(\pi)
    \propto
    \prod_{t=1}^n
    \Bigg[
    \underbrace{
    \frac{\lambda_{\pi_t}\exp\{-\beta\cdot V_\omega(\pi_t;S_t)\}}{\sum_{k\in S_t}\lambda_k\exp\{-\beta\cdot V_\omega(k;S_t)\}}
    }_{\text{(Part 1)}}
    \cdot
    \underbrace{
    \sum_{k\in S_t}\exp\{-\beta\cdot V_\omega(k;S_t)\}
    }_{\text{(Part 2)}}
    \Bigg]
    .
    \label{eq:gpasv}
\end{equation}
At each stage $t$, Part 1 scores which player in $S_t$ occupies position $t$: each $k\in S_t$ competes through its soft priority $\lambda_k$ and its compliance with the pairwise priorities, measured by $V_\omega(k;S_t)$.
Part 1 generalizes PASV's stage-wise softmax: in PASV, every player in $\max(S_t)$ is equally eligible to be $\pi_t$ and competes solely through $\lambda$; in GPASV, the weighted graph $\omega$ no longer issues each player $k$ just a pass/fail ticket, but instead a continuous eligibility score $\exp\{-\beta\cdot V_\omega(k;S_t)\}$ that discounts $k$ by the total amount of pairwise priority violation it would cause if placed at position $t$; we naturally let this score rescale $\lambda_k$.
Part 2 is the soft generalization of $|\max(S_t)|$ in \eqref{eq:pasv-prelim}.
In PASV, $|\max(S_t)|$ is a key factor that ensures the reduction to PSV; under a weighted or cyclic $\omega$, $\max(S_t)$ is no longer well-defined, so Part 2 instead sums the eligibility scores over $S_t$, which tends to $|\max(S_t)|$ when $\omega$ is a DAG and $\beta\to\infty$.

When $\lambda_i\equiv c$, it is not difficult to verify that \eqref{eq:gpasv} reduces to \eqref{eq:mallows-form}.
When $\omega$ encodes a DAG with equal nonzero $\omega$'s and $\beta\to\infty$, \eqref{eq:gpasv} reduces to PASV \eqref{eq:pasv-prelim}; see Section \ref{subsec::method::limit} for formal limiting analysis.

\subsection{Connections to Existing Literature}
\label{subsec::method::connections}

{\bf Gibbs distribution \citep{cantwell2022belief}.}
GPASV adopts a Gibbs-style encoding of pairwise priority weights; but due to the incorporation of $\lambda$, GPASV does not exactly fall into the Gibbs family.

{\bf Mallows model \citep{mallows1957non}.}
The Mallows model specifies a reference permutation $\pi_0\in\Pi$ and a distribution $p_M(\pi)\propto\exp\{-d(\pi,\pi_0)\}$, where $d(\cdot,\cdot)$ is a dissimilarity measure, e.g., Kendall distance.
Mallows is a \emph{narrow special case} of GPASV: it does not consider $\lambda$ and implicitly assumes that the priority graph can be induced by a single reference $\pi_0$ (in particular, $\omega$ must be acyclic).

{\bf Plackett-Luce model \citep{plackett1975analysis,luce1959individual}.}
When $\omega\equiv 0$, \eqref{eq:gpasv} reduces to a Plackett-Luce distribution with ``worths'' $(\lambda_i)$, sampled stage-wise backward from position $t=n$ to $t=1$.

\subsection{Axiomatic Characterization of GPASV}
\label{subsec::method::axioms}

Here, we outline the axiomatization, using some acronyms that will be defined later and in Appendix \ref{app:rov-axioms}.
First, using Weber's axioms \citep{weber1988probabilistic}, we have \hyperref[axiom:e]{E} + \hyperref[axiom:l]{L} + \hyperref[axiom:np]{NP} + \hyperref[axiom:m]{M} $\Rightarrow$ \hyperref[def:rov]{ROV} -- notice that GPASV is supported on the entire $\Pi$ and does not use PASV's Maximal-Support (MS) axiom.
Second, GSCF fixes a canonical stage-wise form for ROV distributions, rather than by itself imposing a new payoff-fairness requirement.
Third, we engage boundary axioms: GWP and PVF to guarantee reduction to important special cases.
Here, PASV's SCF, WP and EWU axioms are generalized to GSCF, GWP and PVF to suit generalized hard priority.
Finally, \hyperref[def:rov]{ROV} + \hyperref[def:gscf]{GSCF} + \hyperref[axiom:gwp]{GWP} + \hyperref[axiom:pvf]{PVF} $\Rightarrow$ \hyperref[eq:gpasv]{GPASV} (see Theorem~\ref{thm:uniqueness}).

Now we carry out the outline.
The first step (E+L+NP+M$\Rightarrow$ROV) is due to \cite{weber1988probabilistic}. 
We start with GSCF.

{\bf Generalized state-choice factorization (GSCF).}
We use GSCF to specify this canonical stage-wise form, compatible with the sampling scheme in Section \ref{subsec::method::GPASV}.

\begin{definition}[GSCF]
	\label{def:gscf}
	A distribution family $p_{\lambda,\omega}$ on $\Pi$ satisfies GSCF if there exist a \emph{state factor} $s_\omega:2^{[n]}\to\mathbb{R}_{\geq 0}$ and a \emph{choice factor} $c_{\lambda,\omega}(\cdot\,;\,\cdot):[n]\times 2^{[n]}\to[0,1]$, s.t.
	\begin{equation}
		p_{\lambda,\omega}(\pi)
		\propto
		\prod_{t=1}^n
		s_\omega(S_t)\cdot c_{\lambda,\omega}(\pi_t;S_t),\quad\pi\in\Pi,
        \ \forall (\lambda,\omega),
		\label{eq:gscf}
	\end{equation}
	with $\sum_{i\in S}c_{\lambda,\omega}(i;S)=1$ for every nonempty $S\subseteq[n]$.
\end{definition}

GSCF is a canonical-form axiom: $s_\omega(S_t)$ is a prefix-dependent rescaling, and $c_{\lambda,\omega}(i;S_t)$ is the stage-wise choice weight of $i$ for position $t$.
Plackett-Luce is a familiar simple instance ($s\equiv 1$, $c=\lambda$-softmax over $S_t$); GPASV is another, where $\omega$ enters both factors.
Specifically, Part 1 of \eqref{eq:gpasv} is the choice factor and Part 2 is the state factor; the two boundary axioms below pin down each of them.
GSCF generalizes PASV's SCF \citep[Definition 3.9]{lee2026priority} along two directions that reflect GPASV's weighted-graph nature: the choice factor ranges over the full $S_t$ rather than the admissible subset $\max(S_t)$ (hard priority is no longer absolute), and the state factor depends on $\omega$ rather than on the DAG alone (recall Section \ref{subsec::method::GPASV}).

{\bf Boundary axioms.}
Within the canonical GSCF family, we identify two important boundary cases.

\begin{axiom}[Generalized Weight Proportionality (GWP)]
	\label{axiom:gwp}
	For every nonempty $S\subseteq[n]$,
	\begin{equation}
		\frac{c_{\lambda,\omega}(i;S)}{c_{\lambda,\omega}(j;S)}
		=
		\frac{\lambda_i\exp\{-\beta\cdot V_\omega(i;S)\}}{\lambda_j\exp\{-\beta\cdot V_\omega(j;S)\}},
        \quad 
        \forall\ i,j\in S.
		\label{eq:gwp}
	\end{equation}
\end{axiom}

Equivalently, the log-odds of choosing $i$ over $j$ at stage $t$ depends on $\lambda_i/\lambda_j$ and $V_\omega(i;S)-V_\omega(j;S)$; PASV's Weight Proportionality (WP) axiom only encodes the former -- the two coincide in the same limiting case under which GPASV reduces to PASV (Section \ref{subsec::method::GPASV}).
In PASV, WP is the key to containing WSV as a boundary case; together with $\sum_{i\in S}c_{\lambda,\omega}(i;S)=1$, GWP fixes the choice factor to Part 1 of \eqref{eq:gpasv}.

\begin{axiom}[Pairwise-Violation Factorization (PVF)]
	\label{axiom:pvf}
	If $\lambda_i\equiv 1$, then $p(\pi)\propto\exp\{-\beta\cdot V_\omega(\pi)\}$.
\end{axiom}
That is, when players' soft priority weights $\lambda_i$'s are equal, $p$ recovers the Gibbs-style form \eqref{eq:mallows-form}.
We use the Gibbs form as the ``no soft priority'' boundary for two reasons: it naturally extends PASV's EWU to weighted and cyclic graphs, and given the expected total violation, it is the maximum-entropy (least informative) distribution on $\Pi$.
Together with GWP, PVF fixes the state factor to Part 2.

\begin{theorem}[Uniqueness of GPASV]
	\label{thm:uniqueness}
	The only ROV satisfying GSCF + GWP + PVF is GPASV.
\end{theorem}

\subsection{Diagnostic Tool: Priority Sweeping}
\label{subsec::method::sweep}

{\bf Soft priority sweep (inherited from \cite{lee2026priority}).}
As pointed out by \cite{lee2026priority}, soft priority $\lambda$ is often unavailable in practice, since relevant metadata such as trustworthiness, compliance risk, or originality scores are not routinely collected.
Following \cite{lee2026priority}, we equip GPASV with \emph{priority sweeping}: one can vary a single $\lambda_i$ over $(0,\infty)$ to diagnose whether the unknown soft priority materially affects the valuation.
This inherits the \emph{idea} from \cite{lee2026priority} but not the algorithm: since GPASV's stage-wise softmax acts on the full $S_t$ rather than the admissible subset $\max(S_t)$, the computation needs adaptation (see Section \ref{sec:computation-main}).

{\bf Hard priority sweep (new to GPASV).}
GPASV's weighted graph $\omega$ opens a sweeping axis that PASV's binary DAG does not admit.
The two extremes are $\beta=0$ (no pairwise priority) and $\beta\to\infty$ (hard priority; see Section \ref{subsec::method::limit}), and sweeping intermediate $\beta$ traces the transition between them.

Overall, priority sweeping should report $\lambda$-only, $\omega$-only, and joint sweeping: this reveals robustness and flags possible double-counting when $\lambda$ aligns with priority graph;
see Sections \ref{sec:simulation}--\ref{sec:application}.

{\bf Choosing the sweep range.}
The sweeping parameters (e.g., $\beta$, and $\alpha$ in Section~\ref{sec:application}) are not tuning parameters: in Shapley-type valuation, the valuation itself is usually the final output, so there is no downstream task that can serve as an undisputed judge for choosing them \citep{wang2024rethinking}.
Instead, they encode the evaluator's preference on how to weigh the two priority sources, and GPASV's deliverable is the family of valuations along the sweep with its sensitivity diagnostics, rather than a single number.
In practice, we sweep each temperature from $0$ upward until the valuation stabilizes, and treat that value as a proxy for $\infty$.

\subsection{Limiting-Case Analysis}
\label{subsec::method::limit}

Let $G_\omega=([n],\mathcal{E}_\omega)$ with $\mathcal{E}_\omega=\{(i,j):\omega_{ij}>0\}$ denote the directed graph induced by $\omega$.
\begin{theorem}[Hard-penalty limit]
\label{thm:hard-limit}
	For every fixed $\lambda$ and $\omega$, we have
	\begin{align}
		\lim_{\beta\to\infty}
        p^{(\lambda,\omega)}(\pi)
		\propto&~
        \mathbbm{1}_{[\pi\in\widetilde{\Pi}^{G_\omega}]}\cdot\prod_{t=1}^n\frac{\lambda_{\pi_t}\,|M_\omega(S_t)|}{\sum_{k\in M_\omega(S_t)}\lambda_k},
        \label{eqn::gpasv-hard-limit}
	\end{align}

	where recall $V_\omega(\pi)$ and $V_\omega(k;S)$ from Section \ref{subsec::method::GPASV}, and define
	$
		\widetilde{\Pi}^{G_\omega}
		:=\arg\min_{\pi\in\Pi}V_\omega(\pi)
	$
	and
	$
		M_\omega(S):=\arg\min_{k\in S}V_\omega(k;S).
	$
\end{theorem}
In Theorem \ref{thm:hard-limit},
$\widetilde{\Pi}^{G_\omega}$ collects permutations with minimum total violation; 
while $M_\omega(S)$ is the soft analogue of $\max(S)$, it collects the members of $S$ whose placement at the end of $S$ incurs the least total violation within $S$.
If $G_\omega$ is a DAG: $\widetilde{\Pi}^{G_\omega}=\Pi^{G_\omega}$, $M_\omega(S_t)=\max(S_t)$ for prefix $S_t$, and Theorem \ref{thm:hard-limit} shows that GPASV reduces to PASV \eqref{eq:pasv-prelim} on $G_\omega$.
But if $G_\omega$ is cyclic, the limit in Theorem \ref{thm:hard-limit} generally does not reduce to any PASV.
A simple example is a big cycle (equal edge weights).
Appendix \ref{appendix::section::GPASV-reduce-examples} gives a sharper example: support becomes DAG, but distribution is still not PASV.

\section{Computation}
\label{sec:computation-main}

Recall that GPASV is defined by \eqref{eq:random-order-value-prelim} and \eqref{eq:gpasv}.
Naturally, we discuss three main topics:
(i) how to sample $\pi$ from \eqref{eq:gpasv};
(ii) how to compute \eqref{eq:random-order-value-prelim};
(iii) acceleration tricks.
Due to page limit, here we only present the gist and relegate detailed algorithms, and auxiliary derivations to Appendix~\ref{appendix::computation}.

\paragraph{Sampling $\pi$.}
We draw $\pi$ from $p^{(\lambda,\omega)}$ \eqref{eq:gpasv} using an adjacent-swap Metropolis--Hastings (MH) chain \citep{karzanov1991conductance,bubley1999faster,lee2026priority}. 
At each iteration, propose to swap $(\pi_i,\pi_{i+1})$ with acceptance probability $\min\{1, p^{(\lambda,\omega)}(\pi^{\rm swap})/p^{(\lambda,\omega)}(\pi)\}$. 
It turns out that for GPASV, almost all factors in this ratio cancel and only a local expression remains, making computation efficient (see Proposition~\ref{prop:adjacent-swap-ratio}).
Initialization for this MCMC requires extra care. 
PASV can conveniently start from any valid linear extension, but for GPASV, when $\omega$ is cyclic, no linear extension exists; a cold start can waste many burn-in iterations.
To address this challenge, we devise a dedicated stage-wise greedy method (see Algorithm~\ref{alg:gpasv-init}): 
sample players backward from position $n$, each with probability proportional to a term that turns out to be exactly the GSCF choice factor $c_{\lambda,\omega}$ in \eqref{eq:gscf} instantiated for GPASV, see  Appendix~\ref{sec:greedy-init}.
Mixing time bounds for this chain are generally intractable; we study mixing empirically in Section~\ref{sec:simulation}.

\paragraph{Estimating the expectation in \eqref{eq:random-order-value-prelim}.}
Given sampled permutations, the most straightforward estimator simply replaces the expectation in \eqref{eq:random-order-value-prelim} by a Monte Carlo average, which we use as our default. 
In Appendix~\ref{sec:two-stage-surrogate}, we additionally describe a surrogate-assisted variant adapted from \cite{liu2026first,lundberg2017unified,fumagalli2026polyshap,witter2025regressionadjusted}: first fit a cheap surrogate $\hat h\approx U$, whose GPASV is usually easier to compute; second, bias-correct for using $\hat h$ in lieu of $U$.
This method works when $\hat h$ well-captures $U$; otherwise, the extra surrogate-fitting cost can outweigh its merit.

\paragraph{Further acceleration.}
A simple trick allows substantial cost reduction with minimal implementation effort: we cache $U(S)$ the first time it is evaluated and reuse the cached value thereafter.
A separate acceleration is specific to priority sweeping (Section~\ref{subsec::method::sweep}), where the target $p^{(\lambda,\omega)}$ moves along a trajectory.
Since nearby targets along this trajectory are similar, we may largely reuse the previous permutation sample instead of redrawing, via \emph{self-normalized importance sampling} (SNIS)~\citep{hesterberg1995weighted}, and monitor to ensure sufficient \emph{effective sample size}.
See Appendix~\ref{app:acceleration} for details.

\section{Empirical Validation}
\label{sec:simulation}

\begin{wrapfigure}{r}{0.5\textwidth}
\vspace{-2em}
\centering
\includegraphics[width=0.5\textwidth]{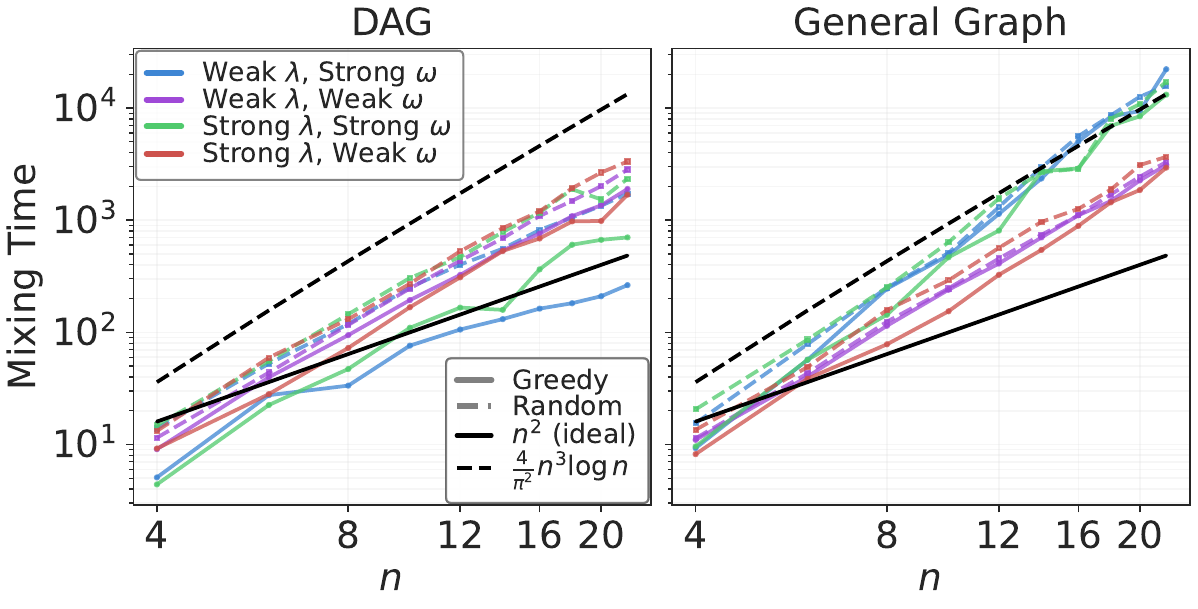}
\caption{Simulation 1: mixing time, greedy (solid) vs random (dashed) initializations.}
\label{fig:short-rmixing}
\vspace{-1.5em}
\end{wrapfigure}

We conduct three simulations to assess the accuracy of our GPASV, test speed-up tricks, and compare its priority sweeping results with its counterpart from the previous work PASV to deepen our understanding.

\paragraph{Simulation 1: mixing time of MCMC.}
Recall from Section \ref{sec:computation-main} that GPASV requires sampling from a non-uniform distribution on permutations; the first step is to burn-in the MH chain until stationarity; the number of iterations needed is called \emph{mixing time}.
We diagnose mixing by the accuracy of pairwise-order probabilities $\pr(\{\text{$i$ appears before $j$}\})$ and tested both DAG and general graphs.
Due to page limit, here we only plot for a representative setting and relegate all details to Appendix \ref{app:sim1-mixing}.
Figure \ref{fig:short-rmixing}  suggests that the mixing speed of our algorithm is competitive compared to literature.
In particular, we find that our greedy initialization method (Algorithm~\ref{alg:gpasv-init}) can significantly speed up mixing.

\begin{figure}[h!]
    \centering
      \includegraphics[width=0.47\linewidth]{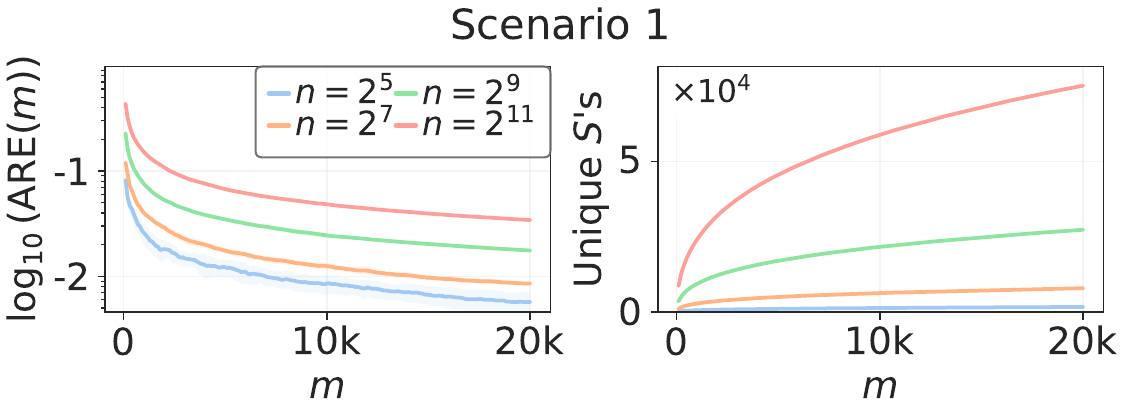}%
      \includegraphics[width=0.47\linewidth]{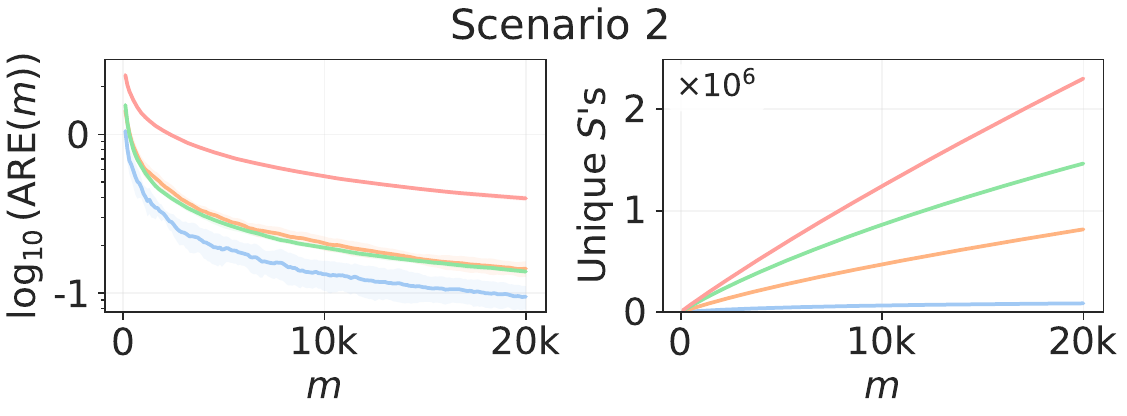}
      \caption{Accuracy and acceleration by caching (case 2); full grids are in Appendix \ref{app:sim2-accuracy}.}
      \label{fig:short-accuracy}
\end{figure}

\paragraph{Simulation 2: Monte Carlo accuracy and surrogate-assisted acceleration.}
After the MH chain mixes well, the next question is the accuracy of the Monte Carlo estimation of \eqref{eq:random-order-value-prelim}.
To separate sampling error from modeling error, we use two synthetic game families with closed-form GPASV values.
Scenario 1 uses a line DAG: $1\to \cdots\to n$;
Scenario 2 has the structure of a \emph{DAG of cycles}: players are partitioned into blocks, the blocks form a DAG among themselves, and each block is internally a big directed cycle.

\begin{wrapfigure}{r}{0.6\textwidth}
    \centering
          \includegraphics[width=\linewidth]{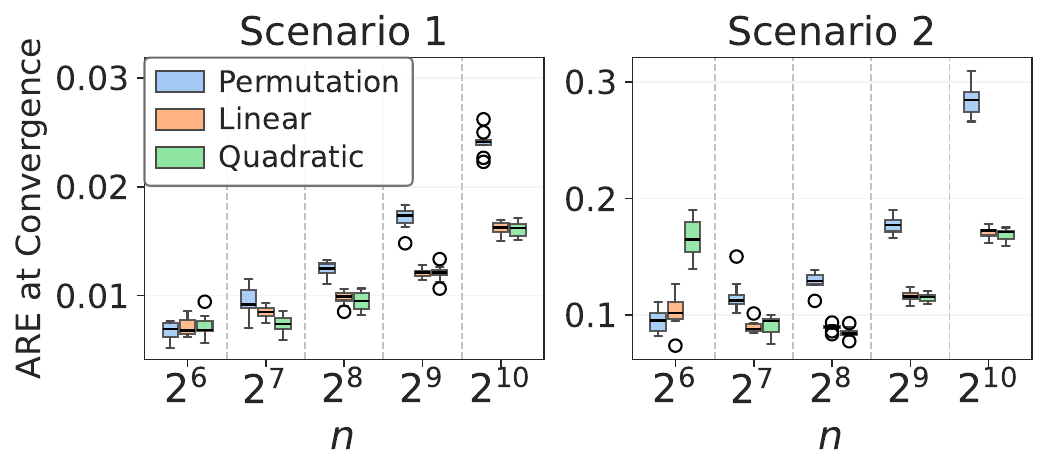}
      \caption{Surrogate-assisted methods (case 2).  Full results are in Appendix \ref{app:sim2-accuracy}.}
      \label{fig:short-surrogate}
    \vspace{-1em}
\end{wrapfigure}
Figure \ref{fig:short-accuracy} shows that the error of direct MC quickly decreases with sample budget; the problem's difficulty increases with $n$ (\# of players), decreases for graphs with stronger priority structures; caching $U(S)$ significantly reduces computational cost.
In Figure \ref{fig:short-surrogate} and Table \ref{tab:surrogate-short}, we compared (i) our direct MC estimator on \eqref{eq:random-order-value-prelim} to two surrogate-assisted estimators introduced in \ref{sec:two-stage-surrogate} using (ii) linear; and (iii) quadratic surrogates -- roughly speaking, (ii) fits $U(S)\approx U(\emptyset)+ \sum_{i\in S} a_i$, while (iii) fits $U(S)\approx U(\emptyset)+\sum_{i\in S} a_i+\sum_{i,j\in S} b_{i,j}$.

\begin{table}[h!]
    \centering
    \caption{Runtime (seconds) and training memory (GB) under matched utility budgets (case 2). 
    Values are mean (std.dev.). 
    Permutation reports runtime only.  
    Full results are in Appendix \ref{app:sim2-accuracy}.}
    \label{tab:surrogate-short}
    \scriptsize
    \setlength{\tabcolsep}{3.2pt}
    \resizebox{\textwidth}{!}{%
    \begin{tabular}{c|c|cc|cc|c|cc|cc}
    \toprule
    & \multicolumn{5}{c|}{\textbf{Scenario 1}} & \multicolumn{5}{c}{\textbf{Scenario 2}} \\
    \cline{2-6}\cline{7-11}
    \rule{0pt}{2.6ex}& Permutation & \multicolumn{2}{c|}{Linear} & \multicolumn{2}{c|}{Quadratic} & Permutation & \multicolumn{2}{c|}{Linear} & \multicolumn{2}{c}{Quadratic} \\
    $n$ & Time (s) & Time (s) & Mem (GB) & Time (s) & Mem (GB) & Time (s) & Time (s) & Mem (GB) & Time (s) & Mem (GB) \\
    \hline
    \rule{0pt}{2.6ex}64   & 4.94 (0.01)   & 26.51 (0.15)  & 0.001 (0.000) & 26.49 (0.10)  & 0.002 (0.000) & 4.69 (0.02)   & 35.50 (0.38)  & 0.093 (0.000) & 36.08 (0.24)   & 0.290 (0.001) \\
    128  & 8.53 (0.02)   & 46.72 (0.19)  & 0.002 (0.000) & 46.89 (0.26)  & 0.012 (0.000) & 8.15 (0.04)   & 61.27 (0.28)  & 0.194 (0.000) & 63.95 (0.22)   & 0.986 (0.003) \\
    256  & 15.94 (0.03)  & 91.73 (0.17)  & 0.009 (0.000) & 92.30 (0.28)  & 0.081 (0.001) & 15.99 (0.07)  & 120.17 (0.39) & 0.592 (0.000) & 138.65 (0.88)  & 5.620 (0.008) \\
    512  & 36.45 (0.16)  & 191.07 (0.65) & 0.031 (0.000) & 194.82 (0.74) & 0.566 (0.006) & 34.36 (0.06)  & 231.47 (1.18) & 1.158 (0.001) & 317.75 (3.33)  & 21.000 (0.027) \\
    1024 & 195.76 (2.46) & 663.81 (6.91) & 0.106 (0.001) & 700.01 (6.31) & 3.724 (0.031) & 205.55 (0.51) & 725.34 (4.64) & 2.312 (0.002) & 1378.02 (9.58) & 108.276 (0.155) \\
    \bottomrule
    \end{tabular}
    }
\end{table}

{\bf Key takeaway:} while surrogate methods improve estimation accuracy in most cases, they (especially the quadratic method) could introduce remarkable runtime and memory overhead.
These additional costs can be much more significant than PASV, since PASV can zero out $b_{i,j}$'s corresponding to DAG edges, while under GPASV this is impossible.
Practically, we recommend direct MC in most cases, and the quadratic surrogate only when (a) $n$ is small enough to afford its $O(n^2)$ terms $b_{i,j}$, or (b) $\omega$ is (nearly) a DAG and a large $\beta$ is preferred, i.e., close to PASV's setting.

\begin{wrapfigure}{r}{0.63\textwidth}
    \centering
    \includegraphics[width=0.63\textwidth]{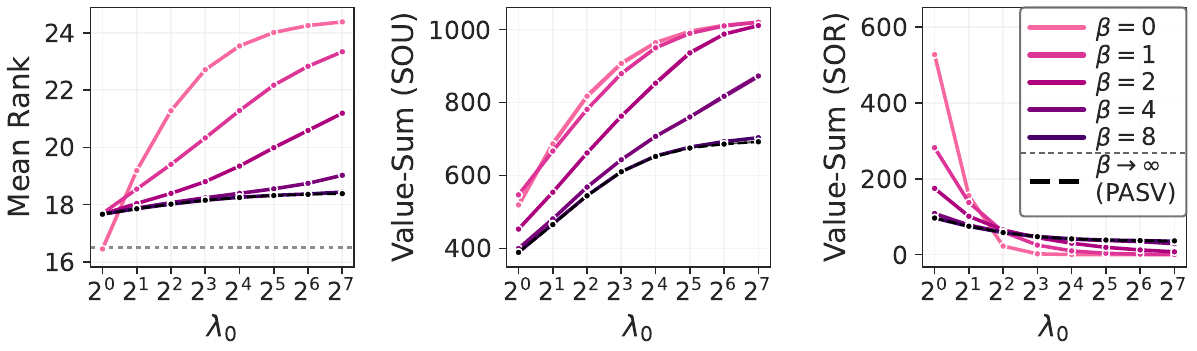}
    \caption{
    Simulation 3: priority sweeping.
    }
    \label{fig:sweep-value}
    \vspace{-1.5em}
\end{wrapfigure}
\paragraph{Simulation 3: priority sweeping.}
This simulation compares the impact of the soft priority $\lambda_i$ under GPASV and PASV.
Recall from \cite{lee2026priority} that under PASV, the impact of $\lambda_i$ is limited by the hard priority.
However, since GPASV has softened the hard priority, any permutation is possible.
To stabilize results, here we pick a subset of size $n/2$ and set their $\lambda_i$'s to a common value $\lambda_0$ and sweep.
Figure \ref{fig:sweep-value} reports the mean rank and valuation of this group.
The valuations use two utilities indexed by random player subsets $T$ (Appendix~\ref{app:sim3-sweep}): the sum-of-unanimity (SOU) utility credits $T$ if \emph{all} members of $T$ are in $S$, and the sum-of-race (SOR) utility if \emph{any} member is; SOU thus rewards late-positioned players and SOR early-positioned ones.
When $\beta$ is small, $\lambda_0$ moves the results more under GPASV than under PASV (black curve); this is a direct consequence of GPASV's full support on $\Pi$ rather than a separate finding, and the gap closes as $\beta\to\infty$.
When $\beta=0$, GPASV recovers Plackett-Luce, in which scenario  $\lambda$'s impact is maximized.

\section{Application: LLM Evaluation under Human Preference Graphs}
\label{sec:application}

We apply GPASV to LLM ensemble valuation, using Chatbot Arena's \citep{chiang2024chatbot} pairwise human preferences as a cyclic, weighted priority graph that PASV cannot handle.
Aggregated human preferences over LLMs are known to contain cycles with high probability~\citep{liu2025statistical}, so a cycle-aware method is not a corner case but an essential ingredient for this application.
The main takeaway is that valuation should not be treated as a single, best-chosen number automatically read off from the data.
Instead, the user must decide how to weigh soft, individual priorities against pairwise user-preference priorities, and this choice materially affects the eventual valuation.

\paragraph{Data and setup.}
Our goal is to apply GPASV to value each candidate LLM.
Running GPASV on this task requires three inputs, each sourced from a public benchmark: a coalition utility, a pairwise priority graph, and a soft priority.
We restrict attention to the $20$ models that appear in both MT-Bench \citep{zheng2023judging} and Chatbot Arena \citep{chiang2024chatbot}, so each player is one model and a coalition $S$ is a subset of these models.
	\paragraph{Coalition utility.}
    	The coalition utility is supplied by MT-Bench \citep{zheng2023judging}, which consists of $80$ two-turn prompts together with released responses from a large set of models scored by an LLM-as-a-judge protocol on a $1$--$10$ scale.
    	For a coalition $S$ and a prompt $q$, we first pass the first-turn responses of models in $S$ through a third-party LLM aggregator to synthesize a single ensemble answer, then ask a third-party LLM judge to score that answer against the MT-Bench protocol; we repeat the aggregator-judge pipeline for the second turn, conditioned on the first-turn dialogue so that coherence is scored as well.
    	The prompt-level utility is the mean of the two turn scores, and the coalition utility $U_{\mathrm{ens}}(S)$ is the mean over all $80$ prompts, with $U_{\mathrm{ens}}(\emptyset)=0$.
    	Both aggregator and judge are instantiated by \texttt{Qwen3.5-35B-A3B-fp8} \citep{qwen3.5}.
        To check the judge choice, we compare our judge scores with the released GPT-4 judge scores of MT-Bench on the $20$ singleton coalitions, where no aggregator is involved: the Pearson correlations are $0.984$, $0.983$, and $0.986$ for Turn 1, Turn 2, and their mean, respectively; see Appendix~\ref{section::limitations} for the scope of this check.
        Further details are relegated to Appendix \ref{app:application-llm-pipeline}; templates are in Appendix \ref{app:application-llm-prompts}.
    
	\paragraph{Pairwise priority graph.}
    	Chatbot Arena \citep{chiang2024chatbot} supplies pairwise human preferences.
    	For each model pair $(i,j)$ with at least $50$ recorded comparisons, let $i$ denote the majority-preferred model and set $\omega_{ij}=\widehat p_{ij}-1/2$ and $\omega_{ji}=0$, where $\widehat p_{ij}$ is the empirical win probability of $i$ over $j$.
        The graph has cycles and is not a DAG \citep{liu2025statistical}.
    	The temperature $\beta$ in \eqref{eq:gpasv} controls how strongly this graph shapes the GPASV order distribution: $\beta=0$ turns off the priority graph entirely, leaving only $\lambda$ (SV if all $\lambda_i$'s are equal); larger $\beta$ enforces the Arena majority directions more strongly.
    \paragraph{Soft priority.}
    	The soft priority encodes a deployment preference for open-source models over paid proprietary APIs.
    	Let $z_i=1$ if model $i$ is open-source and $z_i=0$ otherwise (the paid models in our set are \texttt{gpt-4}, \texttt{gpt-3.5-turbo}, \texttt{claude-v1}, \texttt{claude-instant-v1}, and \texttt{palm-2}).
    	A non-negative temperature $\alpha$ converts this raw preference into the player weights $\lambda_i=\exp(-\alpha z_i)$ that enter GPASV in \eqref{eq:gpasv}, so that $\alpha=0$ recovers $\lambda_i\equiv 1$ (no soft priority).
        Because GPASV samples backward, larger $\alpha$ places open-source models earlier in the sampled permutation.
    	This is not a claim that open-source models are intrinsically better; it simply reflects an evaluator who wants to know how much of the ensemble's performance open-source models are responsible for before reaching for proprietary APIs.

\paragraph{Experimental design.}
Our experiments scan $(\alpha,\beta)$ along three one-dimensional slices: $\alpha$ varies alone (with $\beta=0$), $\beta$ varies alone (with $\alpha=0$), and the two vary together with $\alpha=\beta$.
Along each slice the non-zero temperature takes the values $\{1,2,4,8,16,32\}$, and all three slices share the baseline $(\alpha,\beta)=(0,0)$, at which GPASV reduces to the classical Shapley value with no priority applied.
This gives $19$ settings in total.
The $\alpha$-only and $\beta$-only slices are ablations for diagnosing whether the soft priority and the priority graph encode overlapping signals, while the joint slice studies the evaluator’s deliberate decision to combine them.
Due to space, the main paper presents a few representative $(\alpha,\beta)$ pairs; full results across all $19$ settings are in Appendix \ref{app:application-llm-results}.

\begin{figure}[h!]
    \centering
    \includegraphics[width=\textwidth]{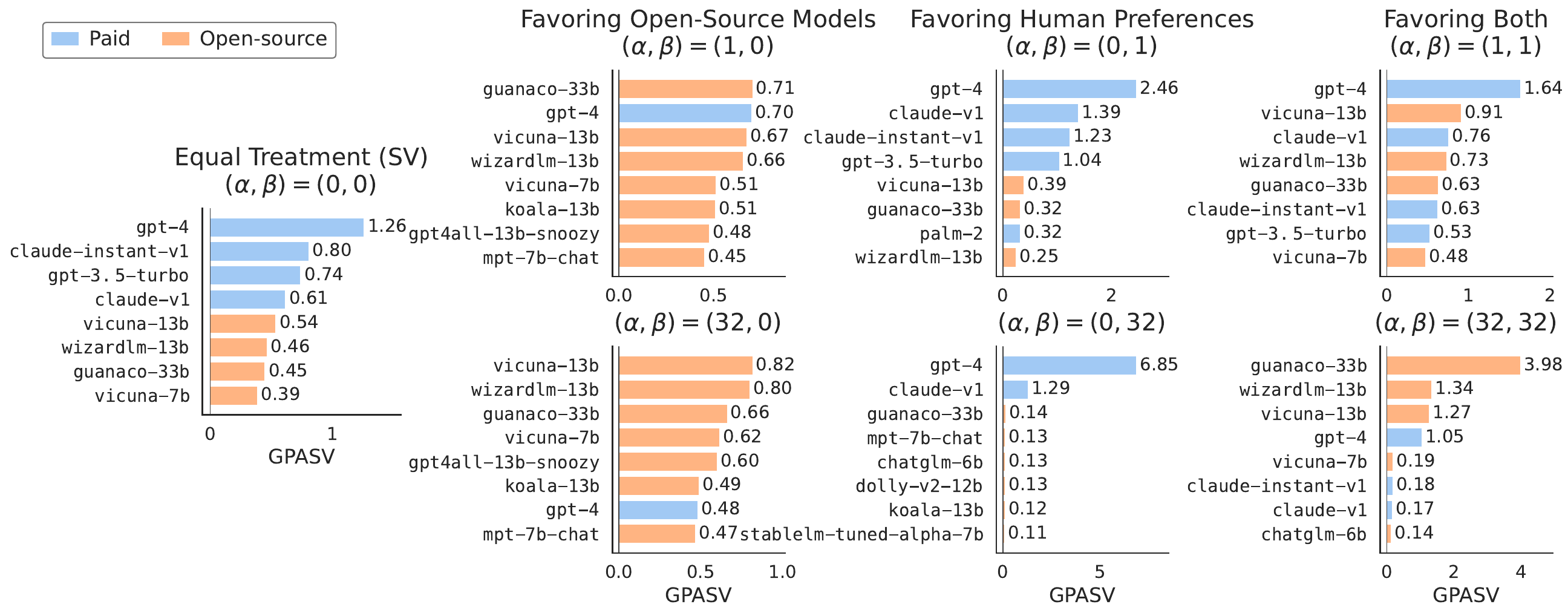}
    \caption{Top-valued models under different priority balances.
    Colors: paid vs open-source.}
    \label{fig:llm-top8}
\end{figure}

\paragraph{Computation.}
The LLM inference cost in utility evaluations dominates that of permutation sampling: one un-cached $U_{\mathrm{ens}}(S)$ evaluation costs $4\times 80=320$ calls.
Naively evaluating every prefix over 19 settings would require $O(10^8)$ LLM calls.
We use the two accelerations from Section \ref{sec:computation-main}: cache each $(S,q)$ aggregator-judge output, and reuse neighboring-setting permutations by SNIS, drawing fresh samples only when effective sample size is low.
Appendix \ref{app:application-llm-sampling} reports a realized 5.4$\times$ reduction in distinct utility evaluations.

\paragraph{Results.}
Figures \ref{fig:llm-top8} and \ref{fig:llm-group} report the scan outcomes: top-$8$ GPASV values under 7 representative $(\alpha,\beta)$ settings, and group-sum GPASV along the three slices.
Four observations.

\textit{1. Paid dominates the no-priority baseline.}
At $(0,0)$, GPASV reduces to the Shapley value, and paid models hold ranks $1$--$4$ led by \texttt{gpt-4} at $1.26$; the first open-source model only appears at rank $5$.
The paid group-sum exceeds the open-source group-sum, even though paid contains only 5 of the 20 models; equivalently, the paid per-model average is substantially higher.
This skew comes from the coalition utility alone --- the aggregator-and-judge pipeline simply scores paid responses higher.
All later observations are to be read against this baseline.

\begin{wrapfigure}{r}{0.5\textwidth}
\centering
\includegraphics[width=0.5\textwidth]{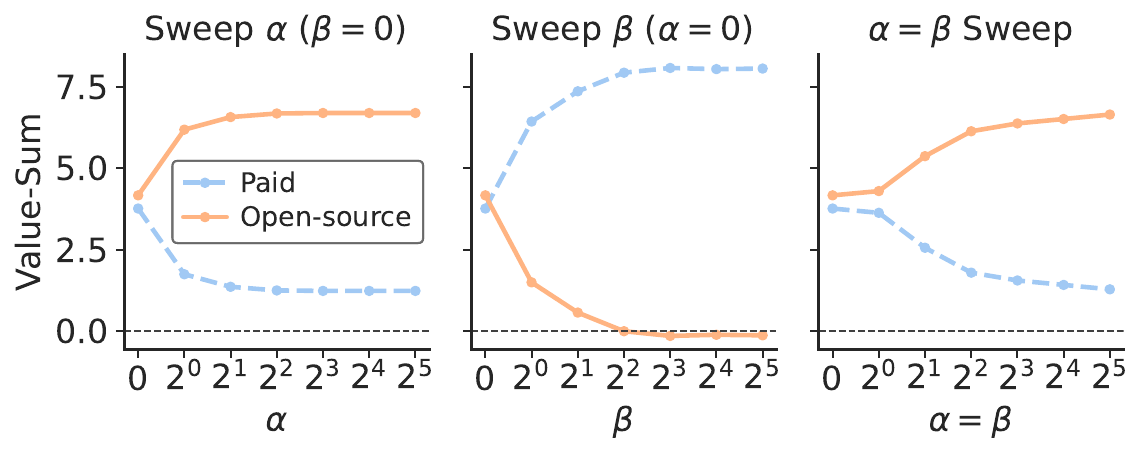}
\caption{Group-sum GPASV for paid and open-source models across priority sweeps.
The $\alpha$ sweep strengthens open-source soft priority, the $\beta$ sweep strengthens the preference graph priority, and the joint sweep increases both together.}
\label{fig:llm-group}
\vspace{-1.0em}
\end{wrapfigure}
\textit{2. $\alpha$ lifts the favored group uniformly.}
Along the $\beta=0$ slice, the paid and open-source group-sum curves cross near $\alpha\approx 2$ and then separate sharply, with open-source climbing and paid collapsing (Figure \ref{fig:llm-group}, left).
Even at $\alpha=1$, \texttt{gpt-4} has lost rank $1$ to \texttt{guanaco-33b} and the other paid models fall out of the top $8$; at $\alpha=32$ every open-source model is lifted by a similar multiplicative factor, because $\alpha$ enters each $\lambda_i=\exp(-\alpha z_i)$ independently.

\textit{3. $\beta$ concentrates GPASV on a single preference-graph hub.}
Along the $\alpha=0$ slice, increasing $\beta$ does not spread the gain across paid models.
At $\beta=32$, \texttt{gpt-4} alone attains GPASV value $6.85$ over $5\times$ the runner-up, while most other paid models fall to the open-source tail ($\sim 0.1$).
The paid group-sum in Figure \ref{fig:llm-group} (middle) rises, but is carried by this one model rather than spread across the five.
The structural reason is that $\beta$ acts through the pairwise violation $V_\omega$, which couples every player to every other, so graph-theoretic hubs absorb attribution rather than uniform category members.
For \texttt{gpt-4} specifically, two factors reinforce this: it wins most of its Arena comparisons, which places it early in almost all sampled permutations, and it is highly effective on its own, leaving little room for later players to contribute.

\textit{4. On the joint slice, $\alpha$ wins.}
The joint slice $\alpha=\beta$ could a priori track either axis or interpolate between them; empirically it tracks the $\alpha$-only axis (Figure \ref{fig:llm-group}, right).
At $(32,32)$, \texttt{guanaco-33b} leads at $3.98$, while \texttt{gpt-4} drops to rank $4$ at $1.05$ --- nowhere near the $6.85$ it commands under $\beta$-only.
The player-level soft priority, which operates through an additive log-factor for every player, dominates the edge-level graph priority, which concentrates into a small number of hubs.

Finally, under a bootstrap of the Arena comparisons, the GPASV rankings remain stable for $\beta\le 16$ (mean Kendall $\tau\ge 0.91$), whereas ranking-then-PASV baselines that first project the graph to a Bradley--Terry or round-robin total order disagree substantially with GPASV and with each other; see Appendix~\ref{app:application-llm-robustness}.

\section{Key Takeaway and Limitations}

{\bf Take-home message.}
Priority-aware valuation is {\bf not a one-button operation}: whenever both player-level and pairwise priority information are present (e.g., open-source deployment preference + Chatbot Arena preference graph), the user must decide how to balance the two sources, and the resulting valuation can materially depend on that decision.
This is {\bf not a method-level inconsistency but a natural property of the problem}: unless an external objective is specified, the data alone cannot identify a uniquely correct priority trade-off; different evaluator preferences define different coherent valuations.
The value of GPASV is to keep these modeling choices visible: it exposes the priority trade-off as an explicit, sweepable input and accepts cyclic priority graphs without first reducing them to a DAG.

{\bf Limitations.} We discuss two main limitations: utility calls and priority graph stability.
Due to page limit, we relegate the detailed discussion to Appendix \ref{section::limitations}.

\begin{ack}
Lee and Zhang were supported by NSF DMS-2311109 and DMS-2610424. Liu and Tang were supported by NSF DMS-2412853, DMS-2610423, and Jane Street Group, LLC.
\end{ack}

\bibliographystyle{abbrvnat}
\bibliography{references}

\clearpage

\appendix
\section*{Computer Code}
Computer code is available at: \url{https://github.com/KiljaeL/GPASV}

\section{Limitations}
\label{section::limitations}

\subsection{Utility Calls}
	The main computational cost of GPASV lies in repeated evaluations of the coalition utility $U(S)$.
	This cost is shared by essentially all model-agnostic data valuation methods in the random-order family, and is, in our view, unavoidable: any method that treats $U$ as a black box has no recourse but to query it, and accurate estimation of marginal contributions imposes a minimum query budget that scales with the number of players.
	The cost becomes especially salient when each query to $U$ is itself expensive, as in the LLM ensemble setting of Section~\ref{sec:application}, where evaluating $U(S)$ for a single coalition requires running an LLM-as-an-aggregator and an LLM-as-a-judge over all $80$ MT-Bench questions, with $20$ candidate models giving $2^{20}$ possible coalitions.

	To address this limitation, our framework incorporates three computational devices that work jointly with the GPASV definition rather than being bolted on after the fact.
	Subset caching (Section~\ref{sec:computation-main}) exploits the fact that the same subset $S$ recurs many times across both the permutation-based estimator and the surrogate-based estimator, so each distinct $U(S)$ is evaluated at most once.
	Self-normalized importance sampling (Section~\ref{sec:snis-reuse}) reuses permutations already drawn under one $(\lambda,\omega)$ to estimate GPASV at neighboring parameter values, which is what makes the priority sweeps in Section~\ref{sec:application} affordable.
	The two-stage surrogate-adjusted residual estimator (Section~\ref{sec:two-stage-surrogate}) further reduces variance per utility query by absorbing most of the signal into a fitted surrogate and spending the remaining query budget on residual correction.
	Together, these devices allow the full LLM application of Section~\ref{sec:application}, including its priority sweeps, to be carried out within a fixed and modest LLM-call budget.

	The same cost also limits how far we can vary the utility pipeline itself.
	The judge check in Section~\ref{sec:application} uses singleton utilities only, which isolates the choice of judge but not that of aggregator, and its conclusion is specific to our aggregator--judge pipeline.
	Full component-replacement ablations, i.e., recomputing coalition utilities and GPASV under alternative judges and aggregators, require recomputing all utilities, and we leave them as an important direction for future work.

\subsection{Priority Graph}

A second limitation concerns the priority graph itself. 
Throughout this paper we treat the priority graph as a given input and study how it shapes the resulting valuation, but in practice the graph is often estimated from finite, noisy, or partially observed data, as is the case for the Chatbot Arena pairwise preferences used in Section~\ref{sec:application}. 
At finite $\beta$, GPASV remains a smooth function of the edge weights, but the $\beta\to\infty$ hard-penalty boundary can amplify near-ties in violation cost. 
Hence priority sweeps can diagnose empirical instability, but they are not formal statistical error bars. Deriving finite-sample error bounds or sharper stability diagnostics for noisy estimated graphs is beyond the present work and is an important direction for future research.
As a first step, Appendix~\ref{app:application-llm-robustness} reports a graph bootstrap on the Chatbot Arena graph: the GPASV rankings are stable for moderate $\beta$, and become less stable at large $\beta$, where the effective sample size of importance reweighting is also low.

\section{Background for the Shapley Value and Random Order Values}
\label{app:rov-axioms}

This appendix collects the standard axioms invoked throughout the paper.
We fix the player set $[n]$ and consider utility functions $U:2^{[n]}\to\mathbb{R}$ with $U(\emptyset)=0$.
A \emph{value} is a mapping $\psi$ assigning to each such $U$ a payoff vector $\psi(U)=(\psi_i(U))_{i\in[n]}$.

The following four axioms are originally due to \citet{shapley1953value}.

\begin{axiom}[Efficiency (E)]
\label{axiom:e}
$\sum_{i\in[n]}\psi_i(U)=U([n])$.
\end{axiom}

\begin{axiom}[Linearity (L)]
\label{axiom:l}
For utilities $U,V$ and scalars $\alpha,\beta\in\mathbb{R}$, $\psi(\alpha U+\beta V)=\alpha\psi(U)+\beta\psi(V)$.
\end{axiom}

\begin{axiom}[Null Player (NP)]
\label{axiom:np}
If player $i$ contributes nothing, i.e., $U(S\cup\{i\})=U(S)$ for all $S\subseteq[n]\setminus\{i\}$, then $\psi_i(U)=0$.
\end{axiom}

\begin{axiom}[Symmetry (S)]
\label{axiom:s}
If two players $i,j$ are interchangeable, i.e., $U(S\cup\{i\})=U(S\cup\{j\})$ for all $S\subseteq[n]\setminus\{i,j\}$, then $\psi_i(U)=\psi_j(U)$.
\end{axiom}

E says the grand-coalition revenue $U([n])$ is fully distributed; L says the value commutes with affine combinations of utilities; NP rules out paying a player whose marginal contribution is identically zero; S enforces anonymity by treating exchangeable players identically.
\citet{shapley1953value} shows that E + L + NP + S uniquely characterize the Shapley value $\psi^{\mathrm{SV}}$ in \eqref{eq:shapley-subset-prelim}.

For random order values, anonymity is precisely what one wants to relax in order to encode priority.
\citet{weber1988probabilistic} replaces S with the following monotonicity (M) axiom and obtains a strictly larger class.

\begin{axiom}[Monotonicity (M)]
\label{axiom:m}
If $U$ is monotone in the sense $S\subseteq T\Rightarrow U(S)\leq U(T)$, then $\psi_i(U)\geq 0$ for every $i\in[n]$.
\end{axiom}

M rules out negative payoffs whenever larger coalitions never reduce revenue.
\citet{weber1988probabilistic} shows that a value satisfies E + L + NP + M if and only if it admits the ROV form \eqref{eq:random-order-value-prelim} for some distribution $p$ over $\Pi$.
In other words, ROV is exactly the family obtained by deliberately relaxing Symmetry: it preserves the additive, contribution-based fairness ideals (E, L, NP, M) while opening room to treat players asymmetrically through the choice of permutation distribution $p$.

\section{Two Equivalent Representations of GPASV}
\label{app:representation}

The definition of GPASV in \eqref{eq:gpasv} encodes pairwise priorities based on a Gibbs-style representation with \emph{additive} priorities $\omega_{ij}\ge 0$.
However, the same family of distributions can be expressed in an alternative, \emph{multiplicative} representation.
This is only a shorthand that connects more directly to the numerical implementation and the algorithms in Appendix~\ref{appendix::computation}.
Define
$$
\tilde\omega_{ij}=\exp(-\beta\omega_{ij}).
$$
Note that the temperature $\beta$ is absorbed in the multiplicative representation.
By one-to-one correspondence, $\tilde\omega_{ij}\in(0,1]$ and $\tilde\omega_{ij}=1$ corresponds to the absence of pairwise priority.
In this representation, the stepwise violation factor is the product
\begin{equation}
    \tilde\omega_k^S
    :=
    \prod_{j\in S\backslash\{k\}}\tilde\omega_{kj},
    \quad S\subseteq[n],\ k\in[n],
    \label{eq:app-param-V-mult}
\end{equation}
with $\tilde\omega_{kk}=1$,
so that we have the equivalence
\begin{equation}
    \exp\{-\beta\cdot V_\omega(k;S)\}
    =
    \prod_{j\in S}\exp(-\beta\omega_{kj})
    =
    \prod_{j\in S}\tilde\omega_{kj}
    =
    \tilde\omega_k^S.
    \label{eq:app-param-equivalence}
\end{equation}
Under this equivalence, the GPASV distribution \eqref{eq:gpasv} admits the equivalent product form
\begin{equation}
    p^{(\lambda,\omega)}(\pi)
    \propto
    \prod_{t=1}^n
    \left[
    \frac{\lambda_{\pi_t}\tilde\omega_{\pi_t}^{S_t}}
         {\sum_{k\in S_t}\lambda_k\tilde\omega_k^{S_t}}
    \cdot
    \sum_{k\in S_t}\tilde\omega_k^{S_t}
    \right].
    \label{eq:app-param-gpasv-mult}
\end{equation}

Reading off \eqref{eq:app-param-gpasv-mult}, GPASV's GSCF factorization (Definition~\ref{def:gscf}) instantiates as
\begin{align}
	c_{\lambda,\omega}(i;S) 
	=&~ \frac{\lambda_i\tilde\omega_i^S}{\sum_{k\in S}\lambda_k\tilde\omega_k^S}, 
	\quad 
	s_\omega(S) = \sum_{k\in S}\tilde\omega_k^S.
	\label{eq:gpasv-gscf-instantiated}
\end{align}
That is, the choice factor is a $\lambda$-weighted softmax over $S$ with weights $\tilde\omega_k^S$, and the state factor is the corresponding partition sum.

Table \ref{tab:param-conversion} summarizes the relationship between the two representations.

\begin{table}[h]
    \centering
    \caption{Correspondence between the multiplicative representation $\tilde\omega_{ij}=\exp(-\beta\omega_{ij})$ and the additive representation $\omega_{ij}$ used in the main text.}
    \label{tab:param-conversion}
    \begin{tabular}{lcc}
    \toprule
    Quantity/Case & Correspondence \\
    \midrule
    Pairwise priority & $\tilde\omega_{ij}=\exp(-\beta\omega_{ij})$ \\
    Stepwise violation & $\tilde\omega_k^S=\exp(-\beta\cdot V_{\omega}(k;S))$ \\
    No priority & $(\tilde\omega_{ij}=1) \sim (\omega_{ij}=0)$ \\
    Hard priority & $(\tilde\omega_{ij}\to0) \sim (\omega_{ij}\text{ or }\beta\to\infty)$ \\
    \bottomrule
    \end{tabular}
\end{table}

\section{A Detailed Example of the GPASV under Cyclic Graph}
\label{appendix::section::GPASV-reduce-examples}

Here we provide a detailed analysis of the cyclic counterexample where GPASV does not reduce to PASV, as mentioned in Section~\ref{subsec::method::limit}.
Consider five players with node weights $\lambda_1=1, \lambda_2=2, \lambda_3=3,\lambda_4=4,\lambda_5=5$ and directed edge weights $\omega_{12}=3, \omega_{23}=4, \omega_{31}=1, \omega_{34}=6$ (player 5 isolated), see Figure \ref{fig:reduction-to-pasv}.
\begin{figure}[h!]
    \centering
    \input{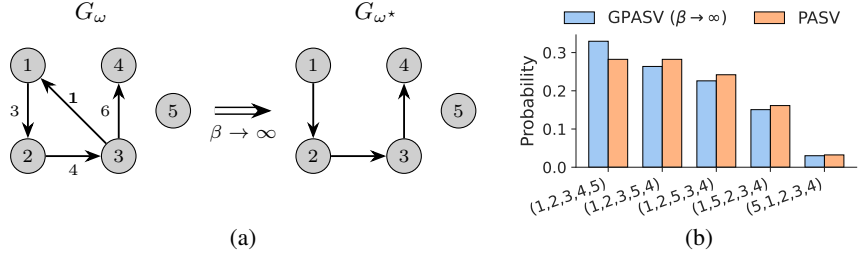}
    \caption{Example showing that cyclic GPASV does not reduce to PASV; (a) the support of the limiting GPASV distribution collapses to the feasible permutations of the DAG obtained by deleting the minimum-violation edge; (b) Yet the unnormalized distribution of the limiting GPASV differs from that of the PASV on the reduced DAG.}
    \label{fig:reduction-to-pasv}
\end{figure}

If we delete the edge $3\to1$, the graph becomes acyclic, and the permutations obeying the remaining edges achieve the minimum total violation of $\omega_{31}=1$.
Write $\omega^\star$ for the edge weights of the reduced DAG, so that $\omega_{31}^\star=0$ and $\omega_{ij}^\star=\omega_{ij}$ for all other pairs.
Then, the support of the limiting GPASV matches the support of the PASV on $G_\omega^\star$:
$$
    \widetilde\Pi^{G_\omega}\;=\;\Pi^{G_\omega^\star}\;=\;\bigl\{(1,2,3,4,5),\,(1,2,3,5,4),\,(1,2,5,3,4),\,(1,5,2,3,4),\,(5,1,2,3,4)\bigr\}.
$$
It is enough to compare with $G_\omega^\star$: if the limiting GPASV law coincided with PASV on some DAG $H$, then its support would satisfy $\Pi^H=\widetilde\Pi^{G_\omega}=\Pi^{G_\omega^\star}$. The set of linear extensions determines the same partial order up to transitive closure, and PASV depends only on this induced partial order. Thus PASV on $H$ coincides with PASV on $G_\omega^\star$, and it suffices to show that the limiting GPASV law differs from PASV on $G_\omega^\star$.

Since the two supports agree, we may compare the unnormalized distributions of the limit of GPASV in Theorem~\ref{thm:hard-limit} with the PASV \eqref{eq:pasv-prelim} on $G_{\omega^\star}$, which are given by:
$$
    \lim_{\beta\to\infty}\tilde{p}^{(\lambda,\omega)}(\pi)
    \;=\;\prod_{t=1}^n\frac{\lambda_{\pi_t}\,|M_\omega(S_t)|}{\sum_{k\in M_\omega(S_t)}\lambda_k},
    \quad
    \tilde{p}^{(\lambda,G_{\omega^\star})}(\pi)
    \;=\;\prod_{t=1}^n\frac{\lambda_{\pi_t}\,|\max(S_t)|}{\sum_{k\in \max(S_t)}\lambda_k}.
$$

For a permutation $\pi=(1,2,3,4,5)$,  their ratio (PASV/GPASV) is 1.
However, on the permutation $\pi=(1,2,3,5,4)$, their ratio is
$$
    \frac{\rm PASV}{\rm GPASV}
    =
    \frac{2\lambda_5}{\lambda_3+\lambda_5}.
$$
Their key discrepancy happens when discussing $\pi_4$.
Under GPASV, $M_\omega(S_4)=\{5\}$, while under PASV, $\max(S_4)=\{3,5\}$, leading to the crucial difference between these two schemes.

From this observation, we emphasize following understandings:
\begin{itemize}
    \item In terms of the support of the distribution of $\pi$, a GPASV under a cyclic graph may reduce to a PASV.
    However, their $\pi$ distributions are still different;
    \item In PASV, at each stage $t$, the maximal set $\max(S_t)$ is the range of candidate players eligible for competing for $\pi_t$;
    however, in GPASV, this is completely different: the set of eligible players for $\pi_t$ is {\bf not} determined by $M_\omega(S_t)$, moreover, it cannot be determined locally, but only through the global rule $\pi\in\tilde\Pi^{G_\omega}$.
\end{itemize}

\section{Proofs of Theoretical Results}

\subsection{Proof of Theorem~\ref{thm:uniqueness}}

We first record the reduction from the four background axioms to the ROV class; this is the entry point for the GSCF/GWP/PVF analysis below.

\begin{lemma}[\citet{weber1988probabilistic}]
\label{lem:weber}
A value $\psi$ satisfies E + L + NP + M (Appendix~\ref{app:rov-axioms}) if and only if $\psi=\nu^p$ for some probability distribution $p$ on $\Pi$, where $\nu^p$ is defined in \eqref{eq:random-order-value-prelim}.
\end{lemma}

A proof of Lemma~\ref{lem:weber} is given by \citet{weber1988probabilistic}; we take it as a starting point and focus on the GSCF/GWP/PVF analysis below.

\begin{proof}[Proof of Theorem~\ref{thm:uniqueness}]
We prove the two implications $\boldsymbol{(\Rightarrow)}$ and $\boldsymbol{(\Leftarrow)}$ separately.

\noindent$\boldsymbol{(\Rightarrow)}$
Suppose $p$ is a ROV satisfying GSCF, GWP and PVF.
By GSCF \eqref{eq:gscf}, there exist $s_\omega:2^{[n]}\to\mathbb{R}_{\geq 0}$ and $c_{\lambda,\omega}(\cdot;\cdot)$ with $\sum_{i\in S}c_{\lambda,\omega}(i;S)=1$ for every nonempty $S\subseteq[n]$ such that
\begin{equation}
    p(\pi)\propto\prod_{t=1}^{n}s_\omega(S_t)\,c_{\lambda,\omega}(\pi_t;S_t),\quad\pi\in\Pi.
    \label{eq:proof-thm1-gscf}
\end{equation}

Fix any nonempty $S\subseteq[n]$ and any reference $j\in S$.
For every $i\in S$, GWP \eqref{eq:gwp} yields
$$
    c_{\lambda,\omega}(i;S)
    =
    c_{\lambda,\omega}(j;S)\cdot\frac{\lambda_i\exp\{-\beta\cdot V_\omega(i;S)\}}{\lambda_j\exp\{-\beta\cdot V_\omega(j;S)\}}.
$$
Summing over $i\in S$ and using the GSCF normalization $\sum_{i\in S}c_{\lambda,\omega}(i;S)=1$ identifies
\begin{equation}
    c_{\lambda,\omega}(i;S)
    =
    \frac{\lambda_i\exp\{-\beta\cdot V_\omega(i;S)\}}{\sum_{k\in S}\lambda_k\exp\{-\beta\cdot V_\omega(k;S)\}},
    \quad i\in S,
    \label{eq:proof-thm1-choice}
\end{equation}
which determines the GSCF choice factor uniquely from $(\lambda,\omega,\beta)$.

Substituting \eqref{eq:proof-thm1-choice} into \eqref{eq:proof-thm1-gscf} gives, for every $\pi\in\Pi$,
\begin{equation}
    p(\pi)
    \propto
    \left(\prod_{t=1}^{n}\frac{s_\omega(S_t)}{\sum_{k\in S_t}\lambda_k\exp\{-\beta\cdot V_\omega(k;S_t)\}}\right)
    \cdot
    \prod_{t=1}^{n}\lambda_{\pi_t}\exp\{-\beta\cdot V_\omega(\pi_t;S_t)\}.
    \label{eq:proof-thm1-rewrite}
\end{equation}
The exponents in the second product 
simplify by re-indexing the double sum
\begin{equation}
    \sum_{t=1}^{n}V_\omega(\pi_t;S_t)
    =
    \sum_{t=1}^{n}\sum_{s=1}^{t}\omega_{\pi_t,\pi_s}
    =
    \sum_{1\leq s<t\leq n}\omega_{\pi_t,\pi_s}
    =
    V_\omega(\pi),
    \label{eq:proof-thm1-telescope}
\end{equation}
where the last equality re-indexes by the ordered pairs $(\pi_s,\pi_t)$ for which $\pi_s$ appears before $\pi_t$.
Specializing to $\lambda_i\equiv 1$, \eqref{eq:proof-thm1-rewrite} then reduces to
$$
    p(\pi)
    \propto
    \exp\{-\beta\cdot V_\omega(\pi)\}\cdot
    \prod_{t=1}^{n}\frac{s_\omega(S_t)}{\sum_{k\in S_t}\exp\{-\beta\cdot V_\omega(k;S_t)\}}.
$$
Comparing with PVF, which asserts $p(\pi)\propto\exp\{-\beta\cdot V_\omega(\pi)\}$, the prefix product must be a $\pi$-independent constant: there exists $K>0$ such that
\begin{equation}
    \prod_{t=1}^{n}\frac{s_\omega(S_t)}{\sum_{k\in S_t}\exp\{-\beta\cdot V_\omega(k;S_t)\}}
    =
    K
    \quad\text{for every }\pi\in\Pi.
    \label{eq:proof-thm1-pathconst}
\end{equation}
The factor on the left is independent of $\lambda$, so \eqref{eq:proof-thm1-pathconst} continues to hold for arbitrary $\lambda$.

Returning to \eqref{eq:proof-thm1-rewrite} for arbitrary $\lambda$, we factor each stage to expose the GPASV form:
\begin{align*}
    p(\pi)
    &\propto
    \prod_{t=1}^{n}\frac{s_\omega(S_t)\,\lambda_{\pi_t}\exp\{-\beta\cdot V_\omega(\pi_t;S_t)\}}{\sum_{k\in S_t}\lambda_k\exp\{-\beta\cdot V_\omega(k;S_t)\}}\\
    &=
    \left(\prod_{t=1}^{n}\frac{s_\omega(S_t)}{\sum_{k\in S_t}\exp\{-\beta\cdot V_\omega(k;S_t)\}}\right)\\
    &\quad\times
    \prod_{t=1}^{n}\Bigg[
    \frac{\lambda_{\pi_t}\exp\{-\beta\cdot V_\omega(\pi_t;S_t)\}}{\sum_{k\in S_t}\lambda_k\exp\{-\beta\cdot V_\omega(k;S_t)\}}
    \cdot
    \sum_{k\in S_t}\exp\{-\beta\cdot V_\omega(k;S_t)\}
    \Bigg]\\
    &\propto
    \prod_{t=1}^{n}\Bigg[
    \frac{\lambda_{\pi_t}\exp\{-\beta\cdot V_\omega(\pi_t;S_t)\}}{\sum_{k\in S_t}\lambda_k\exp\{-\beta\cdot V_\omega(k;S_t)\}}
    \cdot
    \sum_{k\in S_t}\exp\{-\beta\cdot V_\omega(k;S_t)\}
    \Bigg],
\end{align*}
which is exactly $p^{(\lambda,\omega)}(\pi)$ in \eqref{eq:gpasv}.

\noindent$\boldsymbol{(\Leftarrow)}$
Conversely, take $p=p^{(\lambda,\omega)}$ defined in \eqref{eq:gpasv}.
Setting
$$
    s_\omega(S):=\sum_{k\in S}\exp\{-\beta\cdot V_\omega(k;S)\},
    \quad
    c_{\lambda,\omega}(i;S):=\frac{\lambda_i\exp\{-\beta\cdot V_\omega(i;S)\}}{\sum_{k\in S}\lambda_k\exp\{-\beta\cdot V_\omega(k;S)\}},
$$
we have $s_\omega:2^{[n]}\to\mathbb{R}_{\geq 0}$, $\sum_{i\in S}c_{\lambda,\omega}(i;S)=1$, and $p^{(\lambda,\omega)}(\pi)\propto\prod_t s_\omega(S_t)\,c_{\lambda,\omega}(\pi_t;S_t)$, so GSCF holds.
The ratio $c_{\lambda,\omega}(i;S)/c_{\lambda,\omega}(j;S)$ matches \eqref{eq:gwp} by construction, so GWP holds.
When $\lambda_i\equiv 1$, the stage-wise product Part~1$\cdot$Part~2 in \eqref{eq:gpasv} collapses to $\exp\{-\beta\cdot V_\omega(\pi_t;S_t)\}$, and \eqref{eq:proof-thm1-telescope} gives $\prod_t \exp\{-\beta\cdot V_\omega(\pi_t;S_t)\}=\exp\{-\beta\cdot V_\omega(\pi)\}$, recovering \eqref{eq:mallows-form}; hence PVF holds.
\end{proof}

\subsection{Proof of Theorem~\ref{thm:hard-limit}}

\begin{proof}
We will first isolate the dominant exponential factor in $\beta$ that controls the limiting support, and then evaluate the residual factor on that support.

By \eqref{eq:gpasv}, for every $\pi\in\Pi$,
$$
    p^{(\lambda,\omega)}(\pi)
    \propto
    \prod_{t=1}^{n}
    \lambda_{\pi_t}\exp\{-\beta\cdot V_\omega(\pi_t;S_t)\}
    \cdot
    \frac{\sum_{k\in S_t}\exp\{-\beta\cdot V_\omega(k;S_t)\}}{\sum_{k\in S_t}\lambda_k\exp\{-\beta\cdot V_\omega(k;S_t)\}}.
$$
Using the identity \eqref{eq:proof-thm1-telescope}, $\prod_{t=1}^{n}\exp\{-\beta\cdot V_\omega(\pi_t;S_t)\}=\exp\{-\beta\cdot V_\omega(\pi)\}$, so
\begin{equation}
    p^{(\lambda,\omega)}(\pi)
    \propto
    \exp\{-\beta\cdot V_\omega(\pi)\}\cdot A_\beta(\pi),
    \quad
    A_\beta(\pi):=
    \prod_{t=1}^{n}\lambda_{\pi_t}\cdot
    \frac{\sum_{k\in S_t}\exp\{-\beta\cdot V_\omega(k;S_t)\}}{\sum_{k\in S_t}\lambda_k\exp\{-\beta\cdot V_\omega(k;S_t)\}}.
    \label{eq:proof-thm2-rewrite}
\end{equation}

We now study the limit of $A_\beta(\pi)$ for a fixed $\pi\in\Pi$.
For each stage $t$, set
$$
    m_\omega(S_t):=\min_{k\in S_t}V_\omega(k;S_t),
    \quad
    M_\omega(S_t)=\arg\min_{k\in S_t}V_\omega(k;S_t),
$$
and factor out $e^{-\beta\cdot m_\omega(S_t)}$ from both the numerator and the denominator inside $A_\beta(\pi)$ to obtain
$$
    \frac{\sum_{k\in S_t}\exp\{-\beta\cdot V_\omega(k;S_t)\}}{\sum_{k\in S_t}\lambda_k\exp\{-\beta\cdot V_\omega(k;S_t)\}}
    =
    \frac{\sum_{k\in S_t}\exp\{-\beta\cdot[V_\omega(k;S_t)-m_\omega(S_t)]\}}{\sum_{k\in S_t}\lambda_k\exp\{-\beta\cdot[V_\omega(k;S_t)-m_\omega(S_t)]\}}.
$$
Each exponent in the rewritten sums is nonpositive and equals zero precisely when $k\in M_\omega(S_t)$, so as $\beta\to\infty$ every term with $k\notin M_\omega(S_t)$ vanishes and the ratio tends to $|M_\omega(S_t)|/\sum_{k\in M_\omega(S_t)}\lambda_k\in(0,\infty)$.
Taking the product over $t=1,\dots,n$ yields
\begin{equation}
    A_\beta(\pi)\xrightarrow{\beta\to\infty}
    A_\infty(\pi):=\prod_{t=1}^{n}\frac{\lambda_{\pi_t}\,|M_\omega(S_t)|}{\sum_{k\in M_\omega(S_t)}\lambda_k}\in(0,\infty),
    \quad\pi\in\Pi.
    \label{eq:proof-thm2-Ainf}
\end{equation}

It remains to handle the exponential factor in \eqref{eq:proof-thm2-rewrite}.
Recall $\widetilde{\Pi}^{G_\omega}=\arg\min_{\pi\in\Pi}V_\omega(\pi)$ and let
$$
    V^\star:=\min_{\pi\in\Pi}V_\omega(\pi).
$$
Multiplying every unnormalized weight in \eqref{eq:proof-thm2-rewrite} by the common factor $e^{\beta\cdot V^\star}$ leaves the induced probability distribution unchanged, so
\begin{equation}
    p^{(\lambda,\omega)}(\pi)
    \propto
    \exp\{-\beta\cdot[V_\omega(\pi)-V^\star]\}\cdot A_\beta(\pi),
    \quad\pi\in\Pi.
    \label{eq:proof-thm2-recenter}
\end{equation}
For $\pi\in\widetilde{\Pi}^{G_\omega}$, the exponential factor in \eqref{eq:proof-thm2-recenter} is identically $1$, and by \eqref{eq:proof-thm2-Ainf} the unnormalized weight converges to $A_\infty(\pi)\in(0,\infty)$.
For $\pi\notin\widetilde{\Pi}^{G_\omega}$, we have $V_\omega(\pi)-V^\star>0$ while $A_\beta(\pi)$ stays bounded by \eqref{eq:proof-thm2-Ainf}, so the unnormalized weight tends to $0$.
Consequently the total normalizing mass converges to the strictly positive quantity $\sum_{\pi\in\widetilde{\Pi}^{G_\omega}}A_\infty(\pi)$, and
$$
    \lim_{\beta\to\infty}p^{(\lambda,\omega)}(\pi)
    =
    \begin{cases}
    \dfrac{A_\infty(\pi)}{\sum_{\pi'\in\widetilde{\Pi}^{G_\omega}}A_\infty(\pi')}, & \pi\in\widetilde{\Pi}^{G_\omega},\\[4pt]
    0, & \pi\notin\widetilde{\Pi}^{G_\omega},
    \end{cases}
$$
which, by the definition of $A_\infty$ in \eqref{eq:proof-thm2-Ainf}, is exactly \eqref{eqn::gpasv-hard-limit}.
\end{proof}

\subsection{Proof of Proposition~\ref{prop:adjacent-swap-ratio}}

\begin{proof}
Recall $S_t=\{\pi_1,\dots,\pi_t\}$, so the prefix shared by $\pi$ and $\pi'$ before the swapped pair is $S_{i-1}=\{\pi_1,\dots,\pi_{i-1}\}$, and
$$
S_i=S_{i-1}\cup\{a\},\quad S_i'=S_{i-1}\cup\{b\},\quad S_{i+1}=S_{i+1}'=S_{i-1}\cup\{a,b\}.
$$
For any stage $t\notin\{i,i+1\}$, swapping positions $i$ and $i+1$ leaves both $\pi_t$ and $S_t$ unchanged, so the corresponding factor in \eqref{eq:gpasv} is identical for $\pi$ and $\pi'$ and cancels in the ratio $p^{(\lambda,\omega)}(\pi')/p^{(\lambda,\omega)}(\pi)$.
It therefore suffices to compare the contributions at stages $t=i$ and $t=i+1$.

Write the stagewise factor of \eqref{eq:gpasv} at $(x,S)$ as
\begin{equation}
\label{eq:proof-prop1-Ffactor}
F(x;S):=
\frac{\lambda_x \exp\{-\beta\cdot V_\omega(x;S)\}}{\sum_{k\in S}\lambda_k\exp\{-\beta\cdot V_\omega(k;S)\}}
\cdot
\sum_{k\in S}\exp\{-\beta\cdot V_\omega(k;S)\}
=
\frac{\lambda_x\exp\{-\beta\cdot V_\omega(x;S)\}}{\zeta_{\lambda,\omega}(S)},
\end{equation}
where the second equality follows from the definition of $\zeta_{\lambda,\omega}(S)$ in the proposition.
With this notation,
\begin{equation}
\label{eq:proof-prop1-ratio-F}
\frac{p^{(\lambda,\omega)}(\pi')}{p^{(\lambda,\omega)}(\pi)}
=
\frac{F(b;S_i')\,F(a;S_{i+1}')}{F(a;S_i)\,F(b;S_{i+1})}.
\end{equation}
Since $S_{i+1}=S_{i+1}'$, the factor $\zeta_{\lambda,\omega}(S_{i+1})$ appears once in the numerator and once in the denominator of \eqref{eq:proof-prop1-ratio-F} and cancels.
The $\lambda$ factors $\lambda_a$ and $\lambda_b$ likewise appear once in each, and cancel.
Substituting \eqref{eq:proof-prop1-Ffactor} and collecting the surviving terms yields
\begin{equation}
\label{eq:proof-prop1-collected}
\frac{p^{(\lambda,\omega)}(\pi')}{p^{(\lambda,\omega)}(\pi)}
=
\exp\!\Big\{-\beta\big[V_\omega(b;S_i')+V_\omega(a;S_{i+1})-V_\omega(a;S_i)-V_\omega(b;S_{i+1})\big]\Big\}
\cdot
\frac{\zeta_{\lambda,\omega}(S_i)}{\zeta_{\lambda,\omega}(S_i')}.
\end{equation}

It remains to evaluate the bracketed exponent.
Using $V_\omega(x;S)=\sum_{j\in S}\omega_{xj}$ and the convention $\omega_{xx}=0$,
\begin{align*}
V_\omega(b;S_i')&=\sum_{j\in S_{i-1}}\omega_{bj},\quad
V_\omega(a;S_{i+1})=\sum_{j\in S_{i-1}}\omega_{aj}+\omega_{ab},\\
V_\omega(a;S_i)&=\sum_{j\in S_{i-1}}\omega_{aj},\quad
V_\omega(b;S_{i+1})=\sum_{j\in S_{i-1}}\omega_{bj}+\omega_{ba}.
\end{align*}
Hence
$$
V_\omega(b;S_i')+V_\omega(a;S_{i+1})-V_\omega(a;S_i)-V_\omega(b;S_{i+1})
=
\omega_{ab}-\omega_{ba}.
$$
Substituting into \eqref{eq:proof-prop1-collected} gives \eqref{eq:mh-local-ratio-main}.
\end{proof}

\subsection{Proof of Proposition~\ref{prop:scenario1}}
\label{app:proof-scenario1}

\begin{proof}
Throughout we work in the multiplicative representation $\tilde\omega_{ij}=\exp(-\beta\omega_{ij})$ and $\tilde\omega_k^S=\prod_{j\in S}\tilde\omega_{kj}$ of Appendix~\ref{app:representation}, and abbreviate
$$
\mu_{\lambda,\omega}(S):=\sum_{k\in S}\lambda_k\tilde\omega_k^S,
\qquad
s_\omega(S):=\sum_{k\in S}\tilde\omega_k^S.
$$
Under the Scenario~1 total order, $\tilde\omega_{ij}=\tilde\omega_0$ for $i<j$ and $\tilde\omega_{ij}=1$ otherwise (with the convention $\tilde\omega_{ii}=1$), so for every $k\in[n]$ and $S\subseteq[n]$,
\begin{equation}
\label{eq:proof-prop2-omegakS}
\tilde\omega_k^S=\tilde\omega_0^{\,|\{j\in S:\,j>k\}|}.
\end{equation}

By \eqref{eq:app-sum-of-unanimity} and linearity of the ROV,
\begin{equation}
\label{eq:proof-prop2-reduction}
\psi_i(U)
=\sum_{j:\,i\in T_j}c_j\cdot
\mathbb{P}_{\pi\sim p^{(\lambda,\omega)}}\!\big(\,i\text{ is the last member of }T_j\text{ in }\pi\,\big),
\end{equation}
so it suffices to fix one interval $T=T_j=\{\ell,\ldots,r\}$ (writing $\ell=\ell_j$, $r=r_j$) and prove
\begin{equation}
\label{eq:proof-prop2-target}
\mathbb{P}_{\pi\sim p^{(\lambda,\omega)}}\!\big(\,i\text{ is the last of }T\text{ in }\pi\,\big)
=
\frac{\lambda_i\tilde\omega_0^{\,r-i}}
     {\sum_{k=\ell}^{r}\lambda_k\tilde\omega_0^{\,r-k}},
\qquad i\in T.
\end{equation}

We first observe that the GPASV distribution admits a particularly clean form under Scenario~1.
By \eqref{eq:proof-prop2-omegakS}, for any $S\subseteq[n]$ with $|S|=t$,
$$
s_\omega(S)
=\sum_{k\in S}\tilde\omega_0^{\,|\{j\in S:\,j>k\}|}
=\sum_{q=0}^{t-1}\tilde\omega_0^{\,q},
$$
because the exponent $|\{j\in S:\,j>k\}|$ ranges over $\{0,1,\ldots,t-1\}$ as $k$ ranges over $S$, regardless of which elements of $[n]$ populate $S$.
Hence $s_\omega(S_t)$ depends on $\pi$ only through $|S_t|=t$, and $\prod_{t=1}^{n}s_\omega(S_t)$ is a $\pi$-independent constant.
Substituting into \eqref{eq:gpasv} (equivalently \eqref{eq:app-param-gpasv-mult}) yields
\begin{equation}
\label{eq:proof-prop2-backward}
p^{(\lambda,\omega)}(\pi)
\;\propto\;
\prod_{t=1}^{n}\frac{\lambda_{\pi_t}\tilde\omega_{\pi_t}^{S_t}}{\mu_{\lambda,\omega}(S_t)},
\end{equation}
which is exactly the joint law of the following \emph{backward sequential sampler}: set $S_n=[n]$, and for $t=n,n-1,\ldots,1$, draw
\begin{equation}
\label{eq:proof-prop2-backward-step}
\mathbb{P}(\pi_t=k\mid S_t)
=
\frac{\lambda_k\tilde\omega_k^{S_t}}{\mu_{\lambda,\omega}(S_t)},
\qquad k\in S_t,
\end{equation}
then set $S_{t-1}:=S_t\setminus\{\pi_t\}$.

Under this sampler, the last member of $T$ to appear in $\pi$ (in the forward ordering) is the \emph{first} $T$-element drawn in reverse time. Define the reverse-time stopping stage
$$
\tau:=\max\{t\in[n]:\,\pi_t\in T\}.
$$
At stage $\tau$ no $T$-element has yet been drawn, so $T\subseteq S_\tau$, and \eqref{eq:proof-prop2-backward-step} gives
\begin{equation}
\label{eq:proof-prop2-cond}
\mathbb{P}(\pi_\tau=i\mid S_\tau,\,\pi_\tau\in T)
=
\frac{\lambda_i\tilde\omega_i^{S_\tau}}{\sum_{k\in T}\lambda_k\tilde\omega_k^{S_\tau}},
\qquad i\in T.
\end{equation}
By \eqref{eq:proof-prop2-omegakS}, for $k\in T=\{\ell,\ldots,r\}$,
$$
|\{j\in S_\tau:\,j>k\}|
=|\{j\in T:\,j>k\}|+|\{j\in S_\tau\setminus T:\,j>k\}|
=(r-k)+|\{j\in S_\tau\setminus T:\,j>k\}|.
$$
Since $[n]\setminus T=\{1,\ldots,\ell-1\}\cup\{r+1,\ldots,n\}$ and $k\in[\ell,r]$, the set $\{j\in[n]\setminus T:\,j>k\}$ equals $\{r+1,\ldots,n\}$, which is independent of the choice of $k\in T$. Writing
$$
C(S_\tau):=\tilde\omega_0^{\,|S_\tau\cap\{r+1,\ldots,n\}|},
$$
we therefore have $\tilde\omega_k^{S_\tau}=\tilde\omega_0^{\,r-k}\cdot C(S_\tau)$ for every $k\in T$, and the factor $C(S_\tau)$ cancels between numerator and denominator of \eqref{eq:proof-prop2-cond}:
$$
\mathbb{P}(\pi_\tau=i\mid S_\tau,\,\pi_\tau\in T)
=
\frac{\lambda_i\tilde\omega_0^{\,r-i}}{\sum_{k=\ell}^{r}\lambda_k\tilde\omega_0^{\,r-k}}.
$$
The right-hand side does not depend on $S_\tau$, and $\tau$ is well-defined almost surely (some $T$-element is eventually drawn), so marginalization yields \eqref{eq:proof-prop2-target}. Substituting into \eqref{eq:proof-prop2-reduction} completes the proof.
\end{proof}

\subsection{Proof of Proposition~\ref{prop:scenario2}}
\label{app:proof-scenario2}

\begin{proof}
We retain the notation $\tilde\omega_k^S$, $s_\omega(S)$, and $\mu_{\lambda,\omega}(S)$ from the proof of Proposition~\ref{prop:scenario1}.
Specializing \eqref{eq:app-sum-of-unanimity} to $T_j=B_j$ and using linearity of the ROV, for $i\in B_j$,
\begin{equation}
\label{eq:proof-prop3-reduction}
\psi_i(U)
=c_j\cdot\mathbb{P}_{\pi\sim p^{(\lambda,\omega)}}\!\big(\,i\text{ is the last of }B_j\text{ in }\pi\,\big),
\end{equation}
so it suffices to show that
\begin{equation}
\label{eq:proof-prop3-target}
\mathbb{P}_{\pi\sim p^{(\lambda,\omega)}}\!\big(\,i\text{ is the last of }B_j\text{ in }\pi\,\big)
=\frac{1}{|B_j|},\qquad i\in B_j.
\end{equation}

We will prove \eqref{eq:proof-prop3-target} by exhibiting a bijection of $[n]$ that leaves the GPASV distribution invariant and acts as a single cycle on $B_j$.
Let $B_j=\{b_1,b_2,\ldots,b_m\}$ with $m=|B_j|$, labeled so that the within-block directed cycle is $b_1\to b_2\to\cdots\to b_m\to b_1$, and adopt the cyclic convention $b_{m+1}:=b_1$ throughout this proof. Thus $\tilde\omega_{b_s\, b_{s+1}}=\tilde\omega^{\mathrm{cyc}}_j$ for $s=1,\ldots,m$ (the common within-block cycle weight), while every other ordered within-block pair is a non-edge with $\tilde\omega$-value $1$.
Define the bijection $\phi:[n]\to[n]$ by
$$
\phi(b_s)=b_{s+1}\ (s=1,\ldots,m),
\qquad
\phi(x)=x\ \text{for } x\in[n]\setminus B_j.
$$

We first establish the pairwise invariance
\begin{equation}
\label{eq:proof-prop3-pairinv}
\tilde\omega_{\phi(k)\phi(\ell)}=\tilde\omega_{k\ell},
\qquad k,\ell\in[n].
\end{equation}
If $k,\ell\in B_j$, the shift $s\mapsto s+1$ preserves the cycle edge set $\{(b_s,b_{s+1}):s=1,\ldots,m\}$ (using the convention $b_{m+1}=b_1$) and its common weight $\tilde\omega^{\mathrm{cyc}}_j$, so both sides of \eqref{eq:proof-prop3-pairinv} agree (either both equal $\tilde\omega^{\mathrm{cyc}}_j$ or both equal $1$).
If $k\in B_j$ and $\ell\in B_{m'}$ with $B_{m'}\ne B_j$, then by assumption $\tilde\omega_{k\ell}$ is constant in $(k,\ell)\in B_j\times B_{m'}$ (the block-pair common weight if $B_j\to B_{m'}$ is a block edge, otherwise $1$), and $\phi(\ell)=\ell$, hence $\tilde\omega_{\phi(k)\phi(\ell)}=\tilde\omega_{\phi(k)\ell}=\tilde\omega_{k\ell}$; the case $k\notin B_j,\ell\in B_j$ is symmetric, and the case $k,\ell\notin B_j$ is trivial since $\phi$ fixes both arguments.
Moreover $\lambda_{\phi(k)}=\lambda_k$ for all $k\in[n]$, since $\lambda$ is constant on $B_j$ and $\phi$ fixes every element of $[n]\setminus B_j$.

From \eqref{eq:proof-prop3-pairinv} we obtain, for every $k\in[n]$ and $S\subseteq[n]$,
\begin{equation}
\label{eq:proof-prop3-omega-invariance}
\tilde\omega_{\phi(k)}^{\phi(S)}
=\prod_{\ell\in S}\tilde\omega_{\phi(k)\phi(\ell)}
=\prod_{\ell\in S}\tilde\omega_{k\ell}
=\tilde\omega_k^S,
\end{equation}
where the first equality uses the reindexing $\ell\mapsto\phi(\ell)$ of the product over $\phi(S)$.
Consequently $s_\omega(\phi(S))=s_\omega(S)$ and, since $\lambda_{\phi(k)}=\lambda_k$, $\mu_{\lambda,\omega}(\phi(S))=\mu_{\lambda,\omega}(S)$.

For any permutation $\pi=(\pi_1,\ldots,\pi_n)\in\Pi$, write $\phi(\pi):=(\phi(\pi_1),\ldots,\phi(\pi_n))$ for the componentwise action of $\phi$ on the tuple $\pi$; since $\phi$ is a bijection on $[n]$, $\phi(\pi)\in\Pi$ as well. Denote its prefix sets by $\pi':=\phi(\pi)$, so that $\pi'_t=\phi(\pi_t)$ and $S_t^{\pi'}=\phi(S_t^{\pi})$. Combining the three invariances above, each stagewise factor of \eqref{eq:gpasv} is preserved under $\pi\mapsto\pi'$:
$$
\frac{\lambda_{\pi'_t}\tilde\omega_{\pi'_t}^{S_t^{\pi'}}}{\mu_{\lambda,\omega}(S_t^{\pi'})}\cdot s_\omega(S_t^{\pi'})
=
\frac{\lambda_{\pi_t}\tilde\omega_{\pi_t}^{S_t^{\pi}}}{\mu_{\lambda,\omega}(S_t^{\pi})}\cdot s_\omega(S_t^{\pi}).
$$
Taking the product over $t=1,\ldots,n$ yields
\begin{equation}
\label{eq:proof-prop3-dist-invariance}
p^{(\lambda,\omega)}(\phi(\pi))=p^{(\lambda,\omega)}(\pi),\qquad\pi\in\Pi.
\end{equation}

It remains to convert \eqref{eq:proof-prop3-dist-invariance} into uniformity of the last-of-$B_j$ element.
Define $t^\star(\pi):=\max\{t\in[n]:\,\pi_t\in B_j\}$ and $L(\pi):=\pi_{t^\star(\pi)}$.
Because $\phi(B_j)=B_j$, for $\pi'=\phi(\pi)$ we have $\{t:\pi'_t\in B_j\}=\{t:\pi_t\in B_j\}$, so $t^\star(\pi')=t^\star(\pi)$ and
$$
L(\phi(\pi))=\pi'_{t^\star(\pi)}=\phi(\pi_{t^\star(\pi)})=\phi(L(\pi)).
$$
Combining this equivariance with \eqref{eq:proof-prop3-dist-invariance} and bijectivity of $\phi$, for every $i\in B_j$,
$$
\mathbb{P}\!\big(L(\pi)=i\big)
=\mathbb{P}\!\big(L(\phi(\pi))=\phi(i)\big)
=\mathbb{P}\!\big(L(\pi)=\phi(i)\big).
$$
Since $\phi$ acts on $B_j$ as a single $m$-cycle, iterating this identity shows that $\mathbb{P}(L(\pi)=i)$ is constant in $i\in B_j$, so it equals $1/|B_j|$. This proves \eqref{eq:proof-prop3-target}, and substitution into \eqref{eq:proof-prop3-reduction} gives $\psi_i(U)=c_j/|B_j|$ for $i\in B_j$.
\end{proof}

\section{Computational Details for GPASV}
\label{appendix::computation}

Section~\ref{sec:computation-main} outlines the high-level computational pipeline. This appendix collects the detailed derivations and algorithms.

\subsection{Permutation Sampling}

GPASV is defined as an expectation under $p^{(\lambda,\omega)}$, so the principal computational primitive is sampling permutations from this distribution.
We describe an adjacent-swap Metropolis--Hastings sampler with a local acceptance ratio (Section~\ref{sec:adjacent-swap-sampler}), and a greedy initialization that mirrors the GPASV choice factor (Section~\ref{sec:greedy-init}).

\subsubsection{Adjacent-Swap Metropolis--Hastings Sampler}
\label{sec:adjacent-swap-sampler}

GPASV is defined as an expectation with respect to a permutation distribution on the full space $\Pi$.
Accordingly, the main computational task is to sample permutations approximately from $p^{(\lambda,\omega)}$.
Adjacent-swap Metropolis-Hastings (MH) is a canonical local approach for sampling permutations, both in uniform and non-uniform settings \citep{karzanov1991conductance,bubley1999faster,lee2026priority}.
We adopt this proposal because it preserves the permutation structure while allowing the target distribution to be the GPASV distribution on the full space $\Pi$.

Let $p$ be any target distribution on $\Pi$ and let $q(\pi,\pi')$ be a proposal kernel.
The general MH acceptance probability is
\begin{equation}
\label{eq:mh-general-main}
A_{\mathrm{MH}}(\pi,\pi')
:=
\min\left\{
1,\,
\frac{p(\pi')q(\pi',\pi)}{p(\pi)q(\pi,\pi')}
\right\}.
\end{equation}
In our implementation, given the current permutation $\pi=(\pi_1,\dots,\pi_n)$, we choose an index $i\in[n-1]$ uniformly at random and propose $\pi'$ by swapping the adjacent pair $a:=\pi_i$, $b:=\pi_{i+1}$.
This adjacent-swap proposal is symmetric, so $q(\pi,\pi')=q(\pi',\pi)$ and the proposal term in \eqref{eq:mh-general-main} cancels.
Therefore the acceptance probability reduces to
$$
A_{\mathrm{MH}}(\pi,\pi')
=
\min\left\{1,\frac{p^{(\lambda,\omega)}(\pi')}{p^{(\lambda,\omega)}(\pi)}\right\}.
$$

\begin{proposition}[Local Adjacent-Swap Ratio for GPASV]
\label{prop:adjacent-swap-ratio}
Let $\pi'\in\Pi$ be obtained from $\pi\in\Pi$ by swapping the adjacent pair $(a,b)=(\pi_i,\pi_{i+1})$ at positions $(i,i+1)$.
Write
$$
S_i=\{\pi_1,\dots,\pi_{i-1},a\},
\quad
S_i'=\{\pi_1,\dots,\pi_{i-1},b\},
$$
and define, for any nonempty $S\subseteq[n]$,
$$
\zeta_{\lambda,\omega}(S)
:=
\frac{\sum_{k\in S}\lambda_k\exp\{-\beta\cdot V_\omega(k;S)\}}{\sum_{k\in S}\exp\{-\beta\cdot V_\omega(k;S)\}}.
$$
Then
\begin{equation}
\label{eq:mh-local-ratio-main}
\frac{p^{(\lambda,\omega)}(\pi')}{p^{(\lambda,\omega)}(\pi)}
=
\exp\{-\beta(\omega_{ab}-\omega_{ba})\}
\cdot
\frac{\zeta_{\lambda,\omega}(S_i)}{\zeta_{\lambda,\omega}(S_i')}.
\end{equation}
\end{proposition}

Proposition~\ref{prop:adjacent-swap-ratio} shows that the acceptance ratio depends only on the swapped pair $(a,b)$ and the two competing sets $S_i,S_i'$, not on the entire permutation.
This locality is the key structural reason why adjacent-swap MH remains practical for GPASV.

Algorithm~\ref{alg:gpasv-adjacent-mh} summarizes the resulting sampler.

\begin{algorithm}[t]
    \caption{Adjacent-swap MH for sampling from $p^{(\lambda,\omega)}$}
    \label{alg:gpasv-adjacent-mh}
    \begin{algorithmic}[1]
        \Require Parameters $(\lambda,\omega)$, lazy probability $\xi\in[0,1)$, burn-in $B$, thinning $\tau$, target sample size $N_{\mathrm{MC}}$
        \State Initialize $\pi^{(0)}\in\Pi$ via Algorithm~\ref{alg:gpasv-init}
        \State Initialize the sample set $\mathcal{T}\gets\emptyset$
        \For{$t=1$ to $B+\tau(N_{\mathrm{MC}}-1)+1$}
            \State Set $\pi^{(t)} \gets \pi^{(t-1)}$
            \With{probability $1 - \xi$}
                \State Draw $i$ uniformly from $[n-1]$
                \State Let $a\gets \pi^{(t-1)}_i$ and $b\gets \pi^{(t-1)}_{i+1}$
                \State Form $\pi'$ by swapping the adjacent pair $(a,b)$
                \State Compute the local ratio in \eqref{eq:mh-local-ratio-main}
                \State Accept $\pi'$ with probability $\min\{1,p^{(\lambda,\omega)}(\pi')/p^{(\lambda,\omega)}(\pi^{(t-1)})\}$
                \If{accepted}
                    \State $\pi^{(t)}\gets \pi'$
                \EndIf
            \EndWith
            \If{$t>B$ and $(t-B-1)\bmod\tau=0$}
                \State Append $\pi^{(t)}$ to $\mathcal{T}$
            \EndIf
        \EndFor
        \State \Return $\mathcal{T}$
    \end{algorithmic}
\end{algorithm}

\subsubsection{Greedy Initialization}
\label{sec:greedy-init}

Algorithm~\ref{alg:gpasv-adjacent-mh} is valid for any initial permutation $\pi^{(0)}\in\Pi$, but in practice, initialization can matter because GPASV lives on the full permutation space $\Pi$ while the adjacent-swap proposal is purely local.
This is especially relevant in concentrated regimes, where a poor starting permutation can increase the transient cost before the chain reaches a representative high-probability region.

In a DAG, this issue is comparatively mild.
A zero-violation initialization is easy to obtain, since finding one linear extension is computationally feasible (e.g.\ Kahn's algorithm \citep{kahn1962topological}).
By contrast, on a general directed graph a linear extension may not exist at all; a good starting permutation must instead balance unavoidable edge violations against soft priorities.
Rather than solving a separate global optimization problem, we use a stagewise greedy initialization that mirrors the GPASV choice factor itself.

For a current remaining set $R\subseteq [n]$ and a player $i\in R$, the quantity $\exp\{-\beta\cdot V_\omega(i;R)\}=\exp\{-\beta\sum_{j\in R}\omega_{ij}\}$ is exactly the stage-wise edge factor incurred when $i$ is placed at the current rightmost position relative to the other players still in $R$.
Larger values of this factor (equivalently, smaller $V_\omega(i;R)$) correspond to fewer or weaker immediate priority violations, while neutral non-edges ($\omega_{ij}=0$) contribute factor $1$.
The strategy constructs the permutation backward.
At each step it samples a player with probability proportional to the local GPASV choice factor $\lambda_i\exp\{-\beta\cdot V_\omega(i;R)\}$, places that player at the rightmost unfilled position, and removes it from the remaining set.

Algorithm~\ref{alg:gpasv-init} summarizes the resulting initialization rule.

\begin{algorithm}[t]
\caption{Greedy initialization for GPASV}
\label{alg:gpasv-init}
\begin{algorithmic}[1]
\Require Parameters $(\lambda,\omega)$
\State Initialize the remaining set $R\gets [n]$
\For{$t=n$ down to $1$}
    \State For each $i\in R$, compute the weight $w(i;R)\gets \lambda_i\exp\{-\beta\cdot V_\omega(i;R)\}$
    \State Sample $i^\star\in R$ with probability proportional to $w(i;R)$
    \State Set $\pi^{(0)}_t\gets i^\star$
    \State Update $R\gets R\setminus \{i^\star\}$
\EndFor
\State \Return $\pi^{(0)}$
\end{algorithmic}
\end{algorithm}

When $G_\omega$ is a DAG, every maximal node in the induced subgraph on $R$ has $V_\omega(i;R)=0$, whereas every non-maximal node has $V_\omega(i;R)>0$.
Thus Algorithm~\ref{alg:gpasv-init} systematically favors maximal nodes (whose stage-wise factor equals $1$) while still retaining the node-weight term $\lambda_i$ in the local sampling weights.
In the hard-priority limit $\beta\to\infty$, non-maximal nodes receive vanishing relative weight, so the rule collapses to a $\lambda$-weighted randomized linear extension.

\subsection{GPASV Estimators}

We use two estimators for the GPASV target in \eqref{eq:random-order-value-prelim}: a direct permutation-based Monte Carlo estimator and a subset-reweighted surrogate estimator.

\subsubsection{Direct Permutation Monte Carlo}
\label{appendix::subsec::direct-MC}

The most direct estimator follows from the random order value representation in \eqref{eq:random-order-value-prelim}.
A natural choice is to sample permutations $\pi^{(1)},\dots,\pi^{(N_{\mathrm{MC}})}$ from $p^{(\lambda,\omega)}$ and average the corresponding marginal contributions along those permutations.
This mirrors the standard permutation-based estimator for the classical Shapley value, where the order distribution is uniform over $\Pi$ \citep{castro2009polynomial}.
Define the marginal contribution
$$
\Delta_i(\pi;U):=U(\pi^i\cup\{i\})-U(\pi^i).
$$
Then, if $\pi^{(1)},\dots,\pi^{(N_{\mathrm{MC}})}\sim p^{(\lambda,\omega)}$, the direct estimator is
\begin{equation}
\label{eq:perm-mc-main}
\widehat{\psi}_i^{\mathrm{MC}}(U)
:=
\frac{1}{N_{\mathrm{MC}}}\sum_{m=1}^{N_{\mathrm{MC}}} \Delta_i(\pi^{(m)};U).
\end{equation}
Its finite-sample behavior is controlled by the quality and effective size of the sampled permutation pool.

\subsubsection{Surrogate-Assisted Subset Estimator}
\label{sec:two-stage-surrogate}

The second route provides a bridge between GPASV computation and a line of semivalue estimation methods familiar in data valuation and feature attribution.
Semivalues are a class of attribution rules that generalize the Shapley value by relaxing the efficiency (E in Appendix~\ref{app:rov-axioms}) \citep{dubey1981value}: fixing $n$, a semivalue with weights $a_0,\dots,a_{n-1}$ satisfying $a_s\geq 0$ and $\sum_{s=0}^{n-1}\binom{n-1}{s}a_s=1$ is defined by $\varphi_i(U) = \sum_{S\subseteq[n]\setminus\{i\}} a_{|S|}\bigl\{U(S\cup\{i\})-U(S)\bigr\}$.
Unlike a random order value, semivalues average subset-based marginal contributions rather than marginal contributions along sampled permutations.

Existing semivalue estimators differ in their observation design and in the rule used to recover the target from sampled utilities; these include regression-based methods \citep{lundberg2017unified, fumagalli2026polyshap,fumagalli2026odd}, direct weighted-average methods \citep{li2024one,wang2023data}, and surrogate-adjusted residual estimators \citep{witter2025regressionadjusted}.
A recent work \citep{liu2026first} shows that these can be described by a unified framework.
We adapt \cite{liu2026first} and extend its method to the setting where the priority relationship between data points should be considered.

The main obstacle is that GPASV is not a semivalue in general: random order values including GPASV may break symmetry, whereas semivalues need not satisfy efficiency.
Nevertheless, the subset-linear viewpoint remains useful.
Following \citep{li2024one,liu2026first}, we first rewrite a semivalue as a pure weighted sum over subsets:
$$
\varphi_i(U) = \sum_{S\subseteq[n]}\rho_i(S)U(S),
\quad
\rho_i(S)= \mathbbm{1}\{i\in S\}\rho_i^{+}(S) - \mathbbm{1}\{i\notin S\}\rho_i^{-}(S),
$$
with $\rho_i^{+}(S)=a_{|S|-1}$, $\rho_i^{-}(S)=a_{|S|}$.
The signed coefficient $\rho_i(S)$ depends on whether $i$ belongs to $S$, but the positive and negative subset weights $\rho^+_i(S)$ and $\rho^-_i(S)$ are anonymous functions of coalition size, and known in closed form for several semivalues including the (beta) Shapley value \citep{shapley1953value,kwon2021beta} and the (weighted) Banzhaf value \citep{banzhaf1964weighted,li2023robust}.

Random order values admit an analogous subset-based representation, but the positive and negative coefficients depend on the player $i$ and the actual subset $S$.
If a random order value is induced by a distribution $p$ on $\Pi$,
$$
\rho_i^+(S)
=
\mathbb{P}_{\pi\sim p}\bigl(\pi^i=S\setminus\{i\}\bigr),
\quad
\rho_i^-(S)
=
\mathbb{P}_{\pi\sim p}\bigl(\pi^i=S\bigr).
$$

For GPASV, $p=p^{(\lambda,\omega)}$, so these coefficients depend on the player, the actual subset, and the priority structure encoded by $(\lambda,\omega)$; they are not cardinality-based closed-form weights.
To make the semivalue-style construction usable, we first draw and store utility-free permutations $\tilde\pi^{(1)},\dots,\tilde\pi^{(M)}\sim p^{(\lambda,\omega)}$ and estimate these coefficients by
\begin{equation}
\label{eq:rho-hat-def-main}
\hat\rho_i(S)
:=
\frac{1}{M}\sum_{r=1}^{M}
\Bigl[
\mathbbm{1}\{i\in S\}\mathbbm{1}\{(\tilde\pi^{(r)})^i=S\setminus\{i\}\}
-
\mathbbm{1}\{i\notin S\}\mathbbm{1}\{(\tilde\pi^{(r)})^i=S\}
\Bigr].
\end{equation}

To build a shared proposal over subsets, we use
$$
\hat A(S):=\sum_{i=1}^n |\hat\rho_i(S)|,
\quad
\widehat{\mathcal S}_\rho:=\{S\subseteq[n]:\hat A(S)>0\},
\quad
\hat q(S):=
\frac{\hat A(S)}{\sum_{T\in\widehat{\mathcal S}_\rho}\hat A(T)},
\quad S\in\widehat{\mathcal S}_\rho.
$$
This allocates more mass to subsets that carry larger aggregate coefficient magnitude across players.
Assume the proposal $\hat q$ satisfies $\hat q(S)>0$ whenever $\rho_i(S)\neq 0$ for some $i$. With $\gamma_i(S):=\rho_i(S)/\hat q(S)$ and $\hat{\gamma}_i(S):=\hat{\rho}_i(S)/\hat q(S)$, GPASV admits the rewritten form
$$
\psi_i(U)
=
\mathbb{E}_{S\sim \hat q}\left[\gamma_i(S)U(S)\right].
$$

Since random order values still obey linearity, for any fixed surrogate $h$, $\psi_i(U)=\psi_i(h)+\psi_i(U-h)$.
Hence, conditional on a fixed proposal $\hat q$ with valid support and using exact $\rho_i$, the estimator
$$
\psi_i(h)+\frac{1}{K_{\mathrm{adjust}}} \sum_{k=1}^{K_{\mathrm{adjust}}} \gamma_i(S^{(k)})\bigl(U(S^{(k)})-h(S^{(k)})\bigr),
\quad S^{(k)}\sim \hat q,
$$
is unbiased for $\psi_i(U)$, where $K_{\mathrm{adjust}}$ is the number of residual-correction subsets.
In practice we use the plug-in form below with $\hat\rho_i$:
\begin{equation}
\label{eq:two-stage-surrogate-main}
\widehat{\psi}_i^{\mathrm{2stage}}
:=
\widehat{\psi}_i(\hat h)
+
\frac{1}{K_{\mathrm{adjust}}}\sum_{k=1}^{K_{\mathrm{adjust}}}
\hat{\gamma}_i(S^{(k)})
\Big(U(S^{(k)})-\hat h(S^{(k)})\Big),
\quad S^{(k)}\sim \hat q.
\end{equation}

It remains to specify how $\hat h$ is obtained.
We fit the surrogate on training subsets sampled independently from $\hat q$, separately from the subsets later used for residual correction.
Given $K_{\mathrm{train}}$ evaluated training subsets $T^{(1)},\dots,T^{(K_{\mathrm{train}})}\sim \hat{q}$ and a surrogate class $\mathcal H$, we adapt the weighted least squares procedure in \citep{liu2026first} to the GPASV setting by defining
$$
\hat h\in
\argmin_{h\in\mathcal H}
\sum_{\ell=1}^{K_{\mathrm{train}}}
\hat W(T^{(\ell)})
\left\{U(T^{(\ell)})-h(T^{(\ell)})\right\}^2,
\quad
\hat W(S):=
\sum_{i=1}^n
\frac{\hat\rho_i(S)^2}{\hat q(S)^2}.
$$
The weight $\hat W(S)$ emphasizes subsets that contribute more strongly to the vector of GPASV coordinates.

For a linear surrogate \citep{lundberg2017unified} with fitted form $\hat h_{\mathrm{lin}}(S)= U(\emptyset)+\sum_{k=1}^n \hat a_k\,\mathbbm{1}\{k\in S\}$, we have $\widehat{\psi}_i(\hat h_{\mathrm{lin}})=\hat a_i$.
For a quadratic surrogate \citep{fumagalli2026polyshap} with selected interaction set $\mathcal I\subseteq \{\{k,\ell\}:1\le k<\ell\le n\}$,
$$
\hat h_{\mathrm{quad}}(S)
=
U(\emptyset)+\sum_{k=1}^n \hat a_k\,\mathbbm{1}\{k\in S\}
+\sum_{\{k,\ell\}\in\mathcal I}\hat b_{k\ell}\,\mathbbm{1}\{k,\ell\in S\},
$$
where $\hat b_{k\ell}=\hat b_{\ell k}$ and $\hat b_{ij}=0$ if $\{i,j\}\notin\mathcal I$.
For brevity in the calculations below, let $\pi^{-1}_i\in[n]$ denote the position of player $i$ in $\pi$, i.e., $\pi^{-1}_i=t$ iff $\pi_t=i$; the prefix set $\pi^i$ from \eqref{eq:random-order-value-prelim} can then be written as $\pi^i=\{j\in[n]:\pi^{-1}_j<\pi^{-1}_i\}$.
With this notation, the one-step marginal term satisfies
$\Delta_i(\pi;\hat h_{\mathrm{quad}}) = \hat a_i + \sum_{j\neq i}\hat b_{ij}\,\mathbbm{1}\{\pi^{-1}_j<\pi^{-1}_i\}$.
Taking expectation over $\pi\sim p^{(\lambda,\omega)}$ yields
$$
\widehat{\psi}_i(\hat h_{\mathrm{quad}})
=
\hat a_i+\sum_{j\neq i}\eta_{ji}\hat b_{ij},
\quad
\eta_{ji}:=\mathbb{P}_{\pi\sim p^{(\lambda,\omega)}}\bigl(\pi^{-1}_j<\pi^{-1}_i\bigr).
$$
In practice these pairwise probabilities are estimated from previously stored utility-free permutations:
$$
\hat\eta_{ji}
:=
\frac{1}{M}\sum_{r=1}^{M}
\mathbbm{1}\left\{(\tilde\pi^{(r)})^{-1}_j<(\tilde\pi^{(r)})^{-1}_i\right\}.
$$

We use a fixed training/correction split for simplicity; optimizing this allocation for the two-stage estimator is orthogonal to GPASV and left to future work.

\subsection{Acceleration via Computational Reuse}
\label{app:acceleration}

The sampler above addresses the distributional side of GPASV, but in many applications the dominant cost is not the MH move itself.
Rather, it is often the repeated evaluation of the cooperative game.
This imbalance is well known in the broader Shapley literature: in feature attribution, utility evaluation may require many model inferences \citep{lundberg2017unified}; in data valuation, it may require repeated model retraining \citep{ghorbani2019data,wang2023data}.
This motivates two complementary acceleration mechanisms: reusing utility evaluations on repeated subsets, and reusing permutations across nearby target distributions.

\subsubsection{Utility Reuse Across Repeated Subsets}
\label{sec:utility-reuse}

The utility-caching mechanism described here applies to both estimators in Sections \ref{appendix::subsec::direct-MC} and \ref{sec:two-stage-surrogate}.
Given sampled permutations $\pi^{(1)},\dots,\pi^{(N_{\mathrm{MC}})}$, estimating GPASV requires the marginal contributions $\Delta_i(\pi^{(m)};U)$ from \eqref{eq:perm-mc-main}:
$$
\Delta_i(\pi^{(m)};U),
\quad
i\in[n],\ m\in[N_{\mathrm{MC}}],
$$
Each term is determined by a subset induced by the relative order in $\pi^{(m)}$.
Because adjacent-swap proposals modify permutations only locally, and because many sampled permutations share long initial segments, the same subset often appears repeatedly across MCMC iterations and across players.

We therefore maintain a cache of utility evaluations indexed by the subset itself.
Whenever a subset $S$ is encountered for the first time, we evaluate $U(S)$ once and store the result; every later appearance of the same subset reuses the cached value.
The same cache applies whether $S$ arises as a permutation prefix or as a subset sampled from $\hat q$ in the surrogate estimator.
This does not change the estimator at all, but it can reduce the number of expensive utility calls by a large factor.

The practical consequence is that the total runtime separates naturally into two parts:
$$
\text{runtime}
\approx
(\#\text{ unique subsets})\times(\text{cost per utility evaluation})
\;+\;
(\text{sampling overhead}).
$$
This decomposition is particularly important in applications such as LLM evaluation, where utility computation can be much more expensive than sampling permutations.

\subsubsection{SNIS for Permutation Reuse Across Priority Sweeps}
\label{sec:snis-reuse}

The second acceleration mechanism is useful when the target distribution changes but the underlying attribution problem remains fixed.
A canonical case is priority sweeping.
Section~\ref{subsec::method::sweep} sweeps a single $\lambda_i$ over $(0,\infty)$ to probe the sensitivity of one player's valuation.
A useful complementary view, when the relative scales among players are themselves meaningful, is to fix a \emph{latent} soft priority $\widetilde\lambda$ (with $\widetilde\lambda_i\geq 0$) and control its overall strength through a scalar temperature $\alpha\geq 0$ via
\begin{equation}
\label{eq:snis-reparam}
\lambda_i:=\exp\{-\alpha\widetilde\lambda_i\}.
\end{equation}
$\alpha$ is its $\lambda$-side analogue of $\beta$: $\alpha=0$ removes the soft-priority contrast (all $\lambda_i\equiv 1$), while $\alpha\to\infty$ pushes players with larger $\widetilde\lambda_i$ progressively earlier in $\pi$.
Substituting \eqref{eq:snis-reparam} into \eqref{eq:gpasv} yields a two-parameter family
$$
p_{\alpha,\beta}^{(\lambda,\omega)}(\pi):=p^{(\lambda,\omega)}(\pi)\Big|_{\lambda_i=\exp\{-\alpha\widetilde\lambda_i\}},
$$
and a sweep traces a continuous curve in $(\alpha,\beta)$-space rather than along a single coordinate.
This is the parameterization used, e.g., in our LLM application (Section~\ref{sec:application}).

Within such a sweep, instead of redrawing a full Monte Carlo sample at every new $(\alpha,\beta)$, we may reuse previously sampled permutations through importance weighting.
Let
$$
p(\pi):=p_{\alpha,\beta}^{(\lambda,\omega)}(\pi),
\quad
p'(\pi):=p_{\alpha',\beta'}^{(\lambda,\omega)}(\pi),
$$
and let $\widetilde p$ and $\widetilde p'$ denote the corresponding unnormalized PMFs.
Equivalently, we may write
$$
p(\pi)=\frac{\widetilde p(\pi)}{Z_p},
\quad
p'(\pi)=\frac{\widetilde p'(\pi)}{Z_{p'}},
$$
where the normalizing constants $Z_p$ and $Z_{p'}$ are intractable in general.
Thus, although evaluating the ratio $\widetilde p'(\pi)/\widetilde p(\pi)$ is straightforward, directly using the ordinary importance-sampling correction would still require the unknown ratio $Z_p/Z_{p'}$.
Self-normalized importance sampling avoids this obstacle by estimating that ratio from the same weighted sample \citep{hesterberg1995weighted}.
If $\pi^{(1)},\dots,\pi^{(N)}\sim p$, define the importance weights
$$
w^{(m)}:=\frac{\widetilde p'(\pi^{(m)})}{\widetilde p(\pi^{(m)})},
\quad m\in[N].
$$
Before introducing reuse, applying the direct estimator \eqref{eq:perm-mc-main} under the target distribution $p'$ would require fresh samples $\pi^{(1)},\dots,\pi^{(N)}\sim p'$.
With the importance weights above, however, the same pool $\pi^{(1)},\dots,\pi^{(N)}\sim p$ delivers a valid self-normalized estimator of GPASV at the new target $p'$:
\begin{equation}
\label{eq:snis-gpasv}
\widehat{\psi}^{\mathrm{SNIS}}_i(U)
:=
\frac{
\sum_{m=1}^N
w^{(m)}\Delta_i(\pi^{(m)};U)
}{
\sum_{m=1}^N w^{(m)}
},
\quad i\in[n].
\end{equation}
This justifies permutation reuse across nearby temperature settings.
As an aside, computing only the numerator of \eqref{eq:snis-gpasv} and then rescaling the resulting vector so that its coordinates sum to $U([n])-U(\emptyset)$ recovers exactly \eqref{eq:snis-gpasv} itself; this follows immediately from the efficiency axiom.

To monitor the quality of this reuse, we use the usual effective sample size \citep{kong1994sequential}, defined as
\begin{equation}
\label{eq:ess-main}
\mathrm{ESS}
:=
\frac{\left(\sum_{m=1}^N w^{(m)}\right)^2}{\sum_{m=1}^N (w^{(m)})^2}.
\end{equation}
For example, suppose that one first samples permutations at $(\alpha,\beta)=(0,0)$ and then moves to a nearby target such as $(1,0)$.
Once \eqref{eq:ess-main} is computed and the desired Monte Carlo budget at the new target is $N$, one may draw only
$$
N_{\mathrm{new}}:=\max\{0,\,N-\lfloor \mathrm{ESS}\rfloor\}
$$
fresh permutations $\widetilde{\pi}^{(1)},\dots,\widetilde{\pi}^{(N_{\mathrm{new}})}\sim p'$.
One may then combine them with the reused sample through the hybrid estimate below:
$$
\widehat{\psi}^{\mathrm{hyb}}_i(U)
:=
\frac{\mathrm{ESS}}{\mathrm{ESS}+N_{\mathrm{new}}}
\frac{\sum_{m=1}^N w^{(m)}\Delta_i(\pi^{(m)};U)}{\sum_{m=1}^N w^{(m)}}
+
\frac{N_{\mathrm{new}}}{\mathrm{ESS}+N_{\mathrm{new}}}
\frac{1}{N_{\mathrm{new}}}\sum_{r=1}^{N_{\mathrm{new}}}\Delta_i(\widetilde{\pi}^{(r)};U).
$$
When $N_{\mathrm{new}}=0$, the second term is omitted.
This hybrid estimate interprets the reused sample as contributing roughly $\mathrm{ESS}$ effective draws and fills the remaining budget with fresh target samples.
In this way, ESS serves as an operational measure of how much of the next budget can be inherited from the previous point in the sweep.

\section{Additional Details and Results for Empirical Studies in Section~\ref{sec:simulation}}
\label{app:simulation}

This appendix collects the implementation details and full empirical results
of the three simulation studies in Section~\ref{sec:simulation}.
All experiments in this appendix are CPU-only and were run on the Unity high-performance computing cluster, provided by the
College of Arts and Sciences at The Ohio State University across multiple Intel Xeon CPUs;
the runtime-reporting experiments were pinned to
Intel Xeon Platinum 8468 CPU with 128 GB RAM to keep wall-clock times comparable across configurations.

\subsection{Simulation 1: Mixing Behavior of MCMC}
\label{app:sim1-mixing}

Sampling $\pi\sim p^{(\lambda,\omega)}$ from \eqref{eq:gpasv} relies on an adjacent-swap Metropolis--Hastings chain (Algorithm~\ref{alg:gpasv-adjacent-mh}), whose mixing time is not known in closed form.
We provide a practical way to declare the chain mixed, describe the setup used for the grid of graphs in Section~\ref{sec:simulation}, explain the exact pairwise-order target against which the chain is compared, and report the full mixing-time grid that the main text only shows for one representative panel.

Throughout Section~\ref{app:sim1-mixing} we work in the multiplicative representation of GPASV introduced in Appendix~\ref{app:representation}, which absorbs the temperature $\beta$ into the edge weights and makes the weak/strong-priority regimes easier to parametrize.

\subsubsection{Mixing Diagnostic and Threshold Protocol}

In the classical theory of Markov chains, the mixing time is defined as the first $t$ at which the total variation between the chain's distribution at step $t$ and its stationary distribution falls below a target threshold; once that is met, any functional of the chain is guaranteed to be close to its stationary expectation up to a uniform constant.
On the permutation space $\Pi$, however, directly tracking this total variation is combinatorially out of reach: even representing the chain's current distribution over $\Pi$ already requires $O(n!)$ numbers, and the stationary distribution $p^{(\lambda,\omega)}$ in \eqref{eq:gpasv} does not admit a closed-form normalizer.
We therefore follow \citet{talvitie2017mixing} and diagnose mixing empirically through pairwise-order probabilities and compare the resulting practical mixing times against reference growth rates from the literature.

Let $\widehat P_t^{(\mathrm{rand})}(i\prec j)$ and $\widehat P_t^{(\mathrm{greedy})}(i\prec j)$ denote the empirical probability that player $i$ appears before player $j$ after $t$ proposal steps under random and greedy initialization, respectively.
For each initialization scheme, we track the worst-case pairwise-order deviation from the stationary target $P^\star(i\prec j):=\mathbb{P}_{\pi\sim p^{(\lambda,\omega)}}(\text{$i$ appears before $j$})$:
\begin{equation}
\label{eq:app-mixing-dt}
D_t^{(\mathrm{init})}
:=
\max_{i,j\in[n], i\neq j}
\bigl|\widehat P_t^{(\mathrm{init})}(i\prec j)-P^\star(i\prec j)\bigr|,
\quad\mathrm{init}\in\{\mathrm{rand},\mathrm{greedy}\}.
\end{equation}
The quantity $D_t^{(\mathrm{init})}$ is the primary mixing diagnostic, and we report the first $t$ at which $D_t^{(\mathrm{init})}$ falls below the target threshold $\epsilon=1/4$ as the practical mixing time.
Comparing random against greedy initialization isolates how much of the practical speedup is attributable to initialization alone.

In practice, $D_t^{(\mathrm{init})}$ is evaluated at doubling checkpoints $t=1,2,4,\ldots$ up to a horizon $T(n):=\lceil n^3\log n\rceil$, and the first crossing is localized by binary search between consecutive checkpoints.
A guard band $\epsilon_0=0.02$ prevents Monte Carlo noise from producing spurious crossings.
That is, a crossing is certified only once $D_t^{(\mathrm{init})}\le\epsilon-\epsilon_0$, while values in $[\epsilon-\epsilon_0,\epsilon+\epsilon_0]$ are handled conservatively during the binary search.
If no crossing is observed before $T(n)$, the replicate is recorded as not mixed within the simulated horizon.
The greedy scheme follows Algorithm~\ref{alg:gpasv-init}, and the random scheme draws the starting permutation uniformly from $\Pi$.

\subsubsection{Experimental Setup}

Experiments cover $n\in\{4,6,8,\ldots,22\}$, which is the range where the ground-truth $P^\star$ in \eqref{eq:app-mixing-dt} can be computed exactly by the dynamic program described below; beyond that, exact pairwise probabilities are no longer tractable.

Graphs are drawn from two families.
In the \emph{DAG family}, each ordered pair $(i,j)$ with $i<j$ is included as an edge independently with probability $p_{\mathrm{edge}}$; the resulting graph is automatically acyclic.
In the \emph{general directed-graph family}, each ordered pair $(i,j)$ with $i\ne j$ is included independently with probability $p_{\mathrm{edge}}$, so the graph may contain cycles.
We run $p_{\mathrm{edge}}\in\{0.2,0.8,1.0\}$ to cover sparse, dense, and saturated regimes.

Node weights are sampled as $\lambda_i\stackrel{\mathrm{iid}}{\sim}\mathrm{Unif}(1,U_\lambda)$ with $U_\lambda\in\{1,100\}$; $U_\lambda=1$ gives a homogeneous regime $\lambda_i\equiv 1$, and $U_\lambda=100$ gives a heterogeneous regime.
On present edges, the multiplicative edge weight $\tilde\omega_{ij}$ is sampled as $\tilde\omega_{ij}\stackrel{\mathrm{iid}}{\sim}\mathrm{Unif}(L_{\tilde\omega},U_{\tilde\omega})$ with $(L_{\tilde\omega},U_{\tilde\omega})\in\{(0,0.5),(0.5,1)\}$, corresponding to a strong and a weak pairwise-priority regime respectively (recall $\tilde\omega_{ij}\to 0$ is the hard limit and $\tilde\omega_{ij}=1$ is no priority).
Non-edge pairs are assigned $\tilde\omega_{ij}=1$, so that they contribute a neutral factor to the stage-wise form \eqref{eq:gpasv-gscf-instantiated}.

For each parameter setting, we run $1000$ independent adjacent-swap MH chains in Algorithm~\ref{alg:gpasv-adjacent-mh} per initialization scheme, with lazy probability $1/2$ to remove periodicity of the adjacent-swap proposal.
Each setting is replicated over five graph draws and, within each graph draw, five initialization draws; the saturated graph ($p_{\mathrm{edge}}=1$) is deterministic, so only one graph draw is used there.
Within each initialization replication, the $1000$ chains share the starting permutation but use independent MCMC seeds.

\subsubsection{Exact Pairwise-Order Ground Truth}

\begin{algorithm}[t]
\caption{Dynamic programming computation of $P^\star(i\prec j)$ under GPASV.}
\label{alg:app-exact-pairwise-dp}
\begin{algorithmic}[1]
\Require Node weights $\lambda$, multiplicative pairwise weights $\tilde\omega$
\ForAll{nonempty $S\subseteq[n]$}
    \ForAll{$i\in S$}
        \State $\tilde\omega_i^S\gets \prod_{r\in S\setminus\{i\}}\tilde\omega_{ir}$
    \EndFor
    \ForAll{$i\in S$}
        \State $\ell_{\lambda,\tilde\omega}(i;S)\gets
        \lambda_i\,\tilde\omega_i^S\left(\sum_{k\in S}\tilde\omega_k^S\right)\!\big/\!\left(\sum_{k\in S}\lambda_k\,\tilde\omega_k^S\right)$
    \EndFor
\EndFor
\State $A(\emptyset)\gets 1$
\For{$s=1$ to $n$}
    \ForAll{$S\subseteq[n]$ with $|S|=s$}
        \State $A(S)\gets \sum_{i\in S} A(S\setminus\{i\})\,\ell_{\lambda,\tilde\omega}(i;S)$
    \EndFor
\EndFor
\State $Z\gets A([n])$ and $B([n])\gets 1$
\For{$s=n-1$ down to $0$}
    \ForAll{$S\subseteq[n]$ with $|S|=s$}
        \State $B(S)\gets \sum_{k\notin S}\ell_{\lambda,\tilde\omega}(k;S\cup\{k\})\,B(S\cup\{k\})$
    \EndFor
\EndFor
\State Initialize $P^\star(i\prec j)\gets 0$ for all $i,j$
\ForAll{$j\in[n]$}
    \ForAll{$S\subseteq[n]\setminus\{j\}$}
        \State $C\gets A(S)\,\ell_{\lambda,\tilde\omega}(j;S\cup\{j\})\,B(S\cup\{j\})/Z$
        \ForAll{$i\in S$}
            \State $P^\star(i\prec j)\gets P^\star(i\prec j)+C$
        \EndFor
    \EndFor
\EndFor
\State Set the diagonal to zero and rescale complementary pairs to sum to one
\State \Return $P^\star$
\end{algorithmic}
\end{algorithm}

The diagnostic \eqref{eq:app-mixing-dt} requires the stationary pairwise-order probabilities $P^\star(i\prec j)$ under $p^{(\lambda,\omega)}$.
Computing all $n(n-1)$ pairwise probabilities by enumeration over $\Pi$ costs $O(n^2 n!)$ time, which is already infeasible at $n=22$.
We instead adapt the subset dynamic programming (DP) of \citet{kangas2016counting}, originally designed for the uniform distribution on linear extensions of a DAG, to the non-uniform GPASV distribution on the full $\Pi$.

Working in the multiplicative representation from Appendix~\ref{app:representation}, the GPASV mass function can be written as
$$
p^{(\lambda,\omega)}(\pi)=\frac{\widetilde p(\pi)}{Z},
\quad
\widetilde p(\pi)=\prod_{t=1}^n \ell_{\lambda,\tilde\omega}(\pi_t;S_t(\pi)),
$$
where the stage-wise factor at step $t$ is, reading off \eqref{eq:app-param-gpasv-mult},
$$
\ell_{\lambda,\tilde\omega}(i;S)
:=\frac{\lambda_i\,\tilde\omega_i^S}{\sum_{k\in S}\lambda_k\,\tilde\omega_k^S}\cdot\sum_{k\in S}\tilde\omega_k^S,
\quad i\in S,
\quad \tilde\omega_k^S:=\prod_{r\in S\setminus\{k\}}\tilde\omega_{kr}.
$$
Intuitively, $\ell_{\lambda,\tilde\omega}(i;S)$ is the contribution of placing $i$ last in prefix $S$: the first quotient is the conditional probability of that step, and the sum $\sum_k\tilde\omega_k^S$ is the state rescaling.

Let $A(S)$ be the total unnormalized mass over all partial orders whose prefix is $S$, built up from shorter prefixes by inductively placing one more element last:
$$
A(\emptyset)=1,
\quad
A(S)=\sum_{i\in S} A(S\setminus\{i\})\,\ell_{\lambda,\tilde\omega}(i;S).
$$
Then $Z=A([n])$.
Similarly, let $B(S)$ be the total unnormalized mass of all completions from prefix $S$ to the full set:
$$
B([n])=1,
\quad
B(S)=\sum_{k\notin S}\ell_{\lambda,\tilde\omega}(k;S\cup\{k\})\,B(S\cup\{k\}).
$$
Combining the prefix mass, the step that inserts $j$ at the end, and the continuation mass gives, for $i\ne j$,
$$
P^\star(i\prec j)
=\frac{1}{Z}\sum_{\substack{S\subseteq[n]\setminus\{j\}\\ i\in S}}
A(S)\,\ell_{\lambda,\tilde\omega}(j;S\cup\{j\})\,B(S\cup\{j\}).
$$
Each summand is the unnormalized mass of permutations whose prefix immediately before $j$'s insertion is $S$; the condition $i\in S$ picks out the permutations in which $i$ is earlier than $j$.
After the DP, we set $P^\star(i\prec i)=0$ and, for numerical stability at larger $n$, rescale each complementary pair $\{P^\star(i\prec j), P^\star(j\prec i)\}$ to sum to one.

The whole procedure runs in $O(n^2\,2^n)$ time, replacing the $O(n^2 n!)$ enumeration cost.
This is what makes exact ground-truth computation tractable up to $n=24$.
Algorithm~\ref{alg:app-exact-pairwise-dp} summarizes the resulting DP procedure.

\subsubsection{Two Special Regimes and Reference Rates}

Two limiting regimes give useful reference points for reading the mixing plot.
Both correspond to the hard-priority limit $\tilde\omega_{ij}\to 0$ on every present edge, which by Theorem~\ref{thm:hard-limit} concentrates $p^{(\lambda,\omega)}$ on $\widetilde\Pi^{G_\omega}$.

When $G_\omega$ is a DAG and $\lambda_i\equiv 1$, $\widetilde\Pi^{G_\omega}=\Pi^{G_\omega}$ and the limit is uniform on the linear extensions of $G_\omega$, which is exactly the target distribution of PSV.
When $G_\omega$ is the saturated directed graph ($p_{\mathrm{edge}}=1$) with $\lambda_i\equiv 1$, every $\pi$ incurs the same total raw violation, so the limit becomes uniform on the full $\Pi$, which is the target of the standard Shapley value.
These two regimes bound the sampling difficulty within our grid and motivate the reference growth lines $n^2$ and $(4/\pi^2)\,n^3\log n$ that we overlay in the mixing figures; the former is the practical mixing rate reported by \citet{talvitie2017mixing}, and the latter is the known best theoretical rate \cite{bubley1999faster,wilson2004mixing}.

\subsubsection{Full Results}

Figure~\ref{fig:app-simulation-mixing} reports practical mixing times across the full grid described above.
Panels separate the DAG and general directed graph families, the density $p_{\mathrm{edge}}$, the node heterogeneity $U_\lambda$, and the weak/strong edge-priority regime.
Greedy initialization tracks or beats random initialization throughout the grid, and the gap is largest in dense and saturated regimes, where random initialization often fails to cross the threshold within the $T(n)$ horizon.
The greedy-initialized curves remain close to the $n^2$ reference in saturated regimes and close to the $n^3\log n$ reference in sparse DAG regimes, matching the two limiting targets discussed above.

\begin{figure}[h!]
\centering
\includegraphics[width=\textwidth]{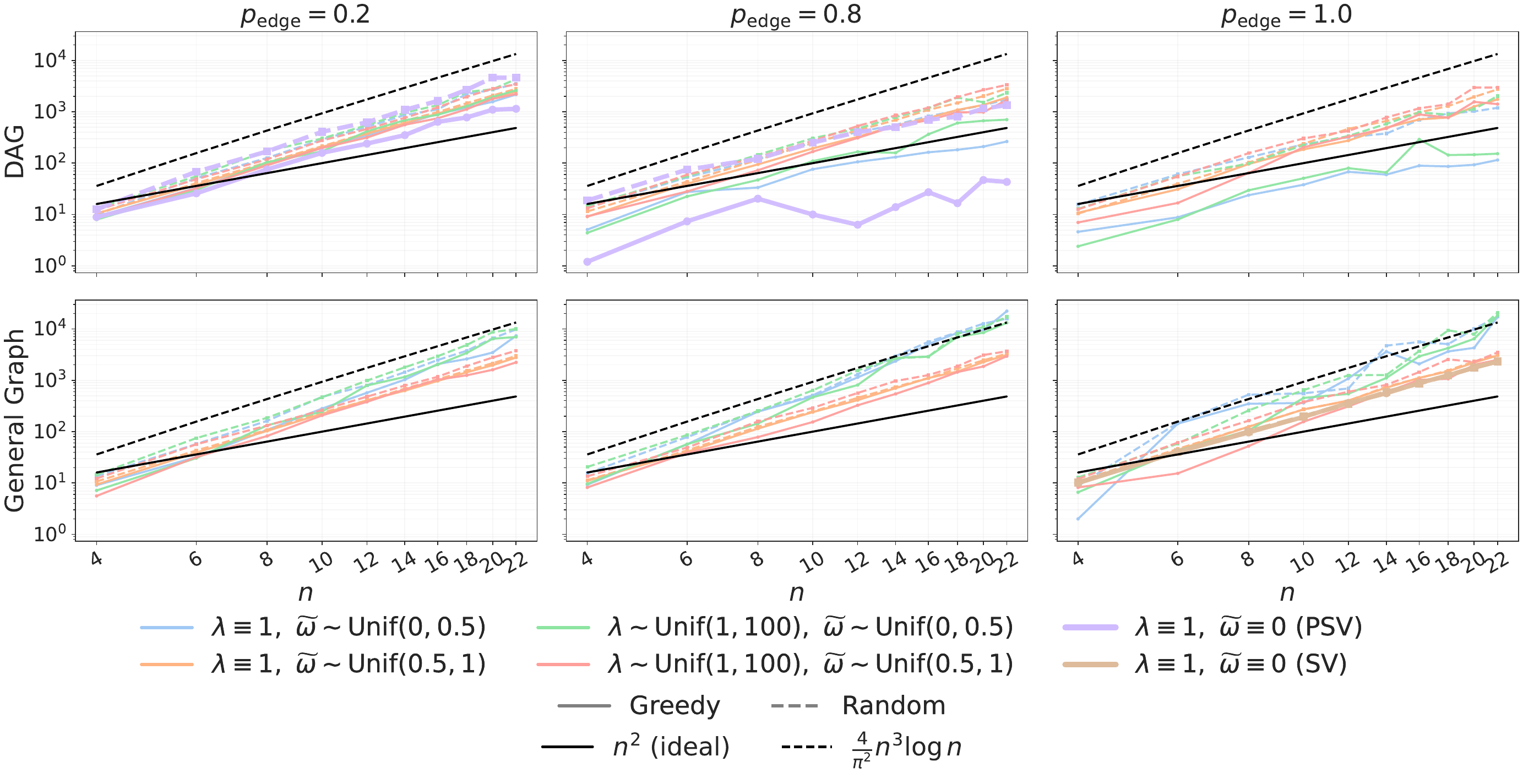}
\caption{Practical mixing-time diagnostic for the adjacent-swap MH chain on GPASV.
Solid curves use greedy initialization (Algorithm~\ref{alg:gpasv-init}); dashed curves use uniform random initialization.
Both axes are in log-scale, so the slope of each curve corresponds to the empirical growth rate of the mixing time.
Panels vary graph family (DAG vs.\ general directed), density $p_{\mathrm{edge}}$.
Reference lines show the $n^2$ (practical mixing time reported by \citet{talvitie2017mixing}) and $(4/\pi^2)n^3\log n$ (best known theoretical mixing time reported by \citet{bubley1999faster,wilson2004mixing}) growth rates.}
\label{fig:app-simulation-mixing}
\end{figure}

\subsection{Simulation 2: Monte Carlo Accuracy and Estimator Comparison}
\label{app:sim2-accuracy}

Once the adjacent-swap MH chain is mixed, the remaining source of error is Monte Carlo noise in the ROV expectation \eqref{eq:random-order-value-prelim}.
This subsection isolates that error by running the direct permutation estimator and the surrogate-assisted estimator from Section~\ref{sec:computation-main} on two synthetic families for which the exact GPASV value $\psi^\star$ is known in closed form.
We also compare the three estimators, including surrogate-assisted estimators introduced in Appendix~\ref{sec:two-stage-surrogate} under matched utility-evaluation budgets.

As in Section~\ref{app:sim1-mixing}, we work in the multiplicative representation of Appendix~\ref{app:representation}: the edge weight $\tilde\omega_{ij}=\exp(-\beta\omega_{ij})\in(0,1]$, with $\tilde\omega_{ij}=1$ corresponding to no pairwise priority, $\tilde\omega_{ij}\to 0$ to the hard-priority limit, and non-edge pairs set to $\tilde\omega_{ij}=1$.

\subsubsection{Closed-Form Targets for Benchmark Scenarios}

Both benchmarks belong to a common \emph{sum-of-unanimity} family:
\begin{equation}
\label{eq:app-sum-of-unanimity}
U(S)=\sum_{j=1}^{d} c_j\,\mathbbm{1}\{T_j\subseteq S\},
\quad T_j\subseteq[n],
\quad c_j\ge 0.
\end{equation}
Under \eqref{eq:app-sum-of-unanimity}, the ROV marginal contribution of player $i\in T_j$ to the $j$-th term is $c_j$ exactly when $i$ is the \emph{last} member of $T_j$ to arrive in $\pi$, and zero otherwise.
By linearity, the ROV expectation \eqref{eq:random-order-value-prelim} therefore reduces to
$$
\psi_i(U)
=\sum_{j:\,i\in T_j} c_j\cdot
\mathbb{P}_{\pi\sim p^{(\lambda,\omega)}}\bigl(\,i \text{ is the last member of } T_j \text{ in }\pi\,\bigr),
$$
so the two scenarios differ only in how the graph and the subset family $\{T_j\}$ are chosen to make this last-arrival probability tractable.

\paragraph{Scenario 1: line DAG with contiguous-interval utilities.}
Players are totally ordered $1\prec 2\prec\cdots\prec n$, and every precedence edge $i<j$ shares a common multiplicative weight $\tilde\omega_{ij}\equiv\tilde\omega_0\in(0,1]$; non-edge pairs (i.e.\ $i>j$) take $\tilde\omega_{ij}\equiv 1$.
The subsets $T_j$ in \eqref{eq:app-sum-of-unanimity} are contiguous intervals $T_j=\{\ell_j,\ell_j+1,\ldots,r_j\}$ of $[n]$.
Under this structure, the last-arrival probability within each interval reduces to a $\lambda$-weighted softmax in which the total-order weight $\tilde\omega_0^{\,r_j-i}$ discounts position $i\in T_j$.

\begin{proposition}[Closed form on Scenario 1]
\label{prop:scenario1}
Let the players be arranged in the total order $1\prec 2\prec\cdots\prec n$ with $\tilde\omega_{ij}\equiv\tilde\omega_0\in(0,1]$ for $i<j$ and $\tilde\omega_{ij}\equiv 1$ on non-edge pairs.
Consider a sum-of-unanimity utility \eqref{eq:app-sum-of-unanimity} with $d=n^2$ and $T_j=\{\ell_j,\ell_j+1,\dots,r_j\}$ a contiguous interval for each $j$.
Then the exact GPASV value is
$$
\psi_i(U)
=\sum_{j:\,i\in T_j}
c_j
\,\frac{\lambda_i\,\tilde\omega_0^{\,r_j-i}}
       {\sum_{q=\ell_j}^{r_j}\lambda_q\,\tilde\omega_0^{\,r_j-q}},
\quad i\in[n].
$$
\end{proposition}

For each interval term, player $i\in T_j$ receives $c_j$ exactly when $i$ is the last element of $T_j$ to appear in $\pi$.
Under the total-order priority structure, the last-arrival probability within $T_j$ is proportional to $\lambda_i\,\tilde\omega_0^{\,r_j-i}$; normalizing over $T_j$ and summing over $j$ by linearity gives the stated form.
The full proof is given in Appendix~\ref{app:proof-scenario1}.

\paragraph{Scenario 2: block DAG with cyclic blocks and block-completion utilities.}
Partition $[n]=B_1\cup\cdots\cup B_K$ into disjoint equal-sized blocks, and build a directed graph in which each block $B_k$ carries an internal directed cycle, and whenever $B_k\to B_\ell$ is present in an underlying DAG on $\{B_1,\ldots,B_K\}$, every ordered pair $(i,j)\in B_k\times B_\ell$ is an edge.
We take $\lambda_i$ constant within each block; on each cycle within a block $B_k$ we take $\tilde\omega_{ij}$ to be a single common value, and on each present ordered block pair $B_k\to B_\ell$ we take $\tilde\omega_{ij}$ to be a single common value over all $(i,j)\in B_k\times B_\ell$.
Non-edge pairs take $\tilde\omega_{ij}\equiv 1$.
The subsets in \eqref{eq:app-sum-of-unanimity} are the blocks themselves, $T_j=B_j$ for $j=1,\ldots,K$, so $U$ rewards the completion of entire blocks.
Within-block cyclic symmetry, combined with equal $\lambda$'s inside a block, makes every member of $B_j$ equally likely to be the last to arrive among $B_j$.

\begin{proposition}[Closed form on Scenario 2]
\label{prop:scenario2}
Let $[n]=B_1\cup\cdots\cup B_K$ be a partition into disjoint blocks, and let the directed graph satisfy the block-structured assumptions above: each block induces a directed cycle, every present block edge $B_k\to B_\ell$ contributes all pairs $(i,j)\in B_k\times B_\ell$ as edges, $\lambda_i$ is constant within each block, each block's cycle edges share a common weight, each present block pair shares a common weight, and non-edge pairs take $\tilde\omega_{ij}\equiv 1$.
Consider the block-completion utility \eqref{eq:app-sum-of-unanimity} with $d=K$ and $T_j=B_j$.
Then the exact GPASV value is
$$
\psi_i(U)=\frac{c_j}{|B_j|}
\quad\text{for every } i\in B_j.
$$
\end{proposition}

For the block-completion term associated with $B_j$, only the final member of $B_j$ to appear can receive $c_j$.
Because node weights are constant within each block and every priority factor affecting $B_j$'s members is symmetric under cyclic relabeling inside $B_j$, each member of $B_j$ is equally likely to be this final member.
The block reward thus splits uniformly within the block, and linearity over $j$ gives the stated form.
The full proof is given in Appendix~\ref{app:proof-scenario2}.

\subsubsection{Experimental Design}

Across both scenarios we draw iid coefficients $c_j\sim\mathrm{Unif}(0.5,1.5)$.
The graph and priority choices specialize the closed-form setups of Proposition~\ref{prop:scenario1} and Proposition~\ref{prop:scenario2} to the two-dimensional $(\lambda,\tilde\omega)$-regime grid used in the main text.

For \textbf{Scenario~1}, we fix $d=n^2$ and sample each interval $T_j$ independently with replacement from the set of contiguous intervals of $[n]$.
Soft priorities follow either the homogeneous regime $\lambda_i\equiv 1$ or the heterogeneous regime $\lambda_i\stackrel{\mathrm{iid}}{\sim}\mathrm{Unif}(1,10)$.
The common precedence-edge weight is $\tilde\omega_0\in\{0.3,\,0.7\}$, with $\tilde\omega_0=0.7$ corresponding to the weak-priority regime (close to no priority) and $\tilde\omega_0=0.3$ to the strong-priority regime.

For \textbf{Scenario~2}, we take $K=n/16$ equal-sized blocks.
The between-block DAG is generated by including each ordered block pair $B_k\to B_\ell$ with $k<\ell$ independently with probability $0.8$, and every block carries a fixed internal directed cycle.
Soft priorities are either $\lambda_i\equiv 1$ or block-constant $\lambda_i=\lambda_{0k}$ for $i\in B_k$ with $\lambda_{0k}\stackrel{\mathrm{iid}}{\sim}\mathrm{Unif}(1,10)$.
Edge priorities are block-structured: for each present block pair $B_k\to B_\ell$, the pair-level common weight $\tilde\omega_{ij}\equiv\tilde\omega^{\mathrm{bet}}_{k\ell}$ over all $(i,j)\in B_k\times B_\ell$, and within each block $B_k$ the cycle edges share a common weight $\tilde\omega^{\mathrm{cyc}}_k$.
Both $\tilde\omega^{\mathrm{bet}}_{k\ell}$ and $\tilde\omega^{\mathrm{cyc}}_k$ are drawn either from $\mathrm{Unif}(0.5,1)$ (weak-priority regime) or from $\mathrm{Unif}(0,0.5)$ (strong-priority regime).

\paragraph{Case 1--4 regimes.}
The two $\lambda$-regimes and the two $\tilde\omega$-regimes factor into four cases, shared between the two scenarios so that the ARE, AUCC, unique-subset, and runtime reports use the same column layout.
The homogeneous $\lambda$-regime is identical across scenarios, while the heterogeneous one differs: Scenario~1 uses a per-player draw, whereas Scenario~2 uses the block-constant version $\lambda_i=\lambda_{0k}$ for $i\in B_k$.
\begin{center}
\small
\begin{tabular}{c|cc|cc}
\toprule
     & \multicolumn{2}{c|}{Scenario 1} & \multicolumn{2}{c}{Scenario 2} \\
Case & $\lambda_i$ & $\tilde\omega_0$ ($i<j$) & $\lambda_i$ (for $i\in B_k$) & $\tilde\omega^{\mathrm{bet}}_{k\ell},\ \tilde\omega^{\mathrm{cyc}}_k$ \\
\midrule
1 & $\equiv 1$                & $\equiv 0.7$ & $\equiv 1$                                   & $\sim\mathrm{Unif}(0.5,1)$ \\
2 & $\equiv 1$                & $\equiv 0.3$ & $\equiv 1$                                   & $\sim\mathrm{Unif}(0,0.5)$ \\
3 & $\sim\mathrm{Unif}(1,10)$ & $\equiv 0.7$ & $\equiv\lambda_{0k}\sim\mathrm{Unif}(1,10)$   & $\sim\mathrm{Unif}(0.5,1)$ \\
4 & $\sim\mathrm{Unif}(1,10)$ & $\equiv 0.3$ & $\equiv\lambda_{0k}\sim\mathrm{Unif}(1,10)$   & $\sim\mathrm{Unif}(0,0.5)$ \\
\bottomrule
\end{tabular}
\end{center}

\subsubsection{Accuracy Metrics and MH Protocol}

We report results for $n\in\{32,128,512,2048\}$ with a total Monte Carlo budget of $N_{\mathrm{MC}}=20{,}000$ post-burn-in permutations.
For each $(n,\text{Case})$, we fix one synthetic instance and run $10$ independent MH chains with greedy initialization (Algorithm~\ref{alg:gpasv-init}), burn-in $\lceil n^{2.5}\rceil$, and thinning interval $1000$.
Standard deviations reported throughout Section~\ref{app:sim2-accuracy} are taken across these $10$ repetitions on the fixed instance.

Every $100$ samples we record the absolute relative error
$$
\mathrm{ARE}(m)
:=\frac{\|\widehat\psi^{(m)}-\psi^\star\|_2}{\|\psi^\star\|_2},
$$
where $\widehat\psi^{(m)}$ is the direct Monte Carlo estimate of $\psi$ after $m$ post-burn-in samples and $\psi^\star$ is the closed-form target from Proposition~\ref{prop:scenario1} or~\ref{prop:scenario2}.
We summarize the whole convergence path by the area under the convergence curve,
$$
\mathrm{AUCC}
:=\frac{1}{200}\sum_{\ell=1}^{200}\mathrm{ARE}(100\ell),
$$
which penalizes slow initial convergence and high terminal error on a common scale.
Every $100$ samples we also record the number of distinct subsets whose utility $U(S)$ has been evaluated up to that point, which measures cache reuse, and the cumulative non-utility runtime, which isolates sampling and bookkeeping cost from the utility evaluation.

Concretely, we fix $\lambda_i\equiv 1$ and sweep two axes: the common edge weight $\tilde\omega_0\in\{0.1,0.3,0.5,0.7,0.9\}$ and the burn-in exponent $c$ in $n^c$ with $c\in\{0,0.5,1,1.5,2,2.5\}$.
All other MH settings match the main Section~\ref{app:sim2-accuracy} protocol.
For each $(\tilde\omega_0,c)$ we run both random and greedy initialization and report AUCC.
This ablation separates two questions: how much of the greedy benefit is simply a shorter burn-in substitute, and how that tradeoff depends on the pairwise-priority strength.

\subsubsection{Matched-Budget Surrogate Comparison Protocol}

Beyond the direct permutation estimator, GPASV admits a surrogate-assisted variant that first fits a cheap surrogate $\hat h\approx U$ on a limited set of evaluated subsets and then corrects the surrogate's GPASV by a residual term; see Section~\ref{sec:two-stage-surrogate} for the full construction.
We compare three estimators:
\begin{enumerate}
\setlength{\itemsep}{0pt}
\item[(i)] the direct permutation estimator of Section~\ref{sec:computation-main};
\item[(ii)] a linear surrogate $\hat h(S)=U(\emptyset)+\sum_{i\in S}\hat{a}_i$, in which each player contributes additively;
\item[(iii)] a quadratic surrogate $\hat h(S)=U(\emptyset)+\sum_{i\in S}\hat{a}_i+\sum_{i,j\in S}\hat{b}_{ij}$ that includes a randomly chosen $10\%$ of the pairwise interactions.
\end{enumerate}

Since the dominant cost of GPASV is direct utility evaluation, we compare the estimators by matching them on the number of distinct evaluated subsets.
For each problem instance we first run the direct estimator and record the number $K_{\mathrm{eval}}$ of distinct subsets whose $U(S)$ was evaluated.
The surrogate estimators are then given the same budget: $K_{\mathrm{train}}:=\min\{\lfloor 0.2\,K_{\mathrm{eval}}\rfloor,\,200{,}000\}$ evaluations go toward fitting $\hat h$, and the remaining $K_{\mathrm{adjust}}:=K_{\mathrm{eval}}-K_{\mathrm{train}}$ evaluations go toward the residual correction.
The two phases share the same utility cache, so $K_{\mathrm{eval}}$ is counted over the union of evaluated subsets rather than the sum.
Permutation samples that do not require utility evaluation, namely those used to estimate subset coefficients and pairwise-order probabilities, do not count toward $K_{\mathrm{eval}}$; we use $M=100{,}000$ such utility-free permutation samples.

Problem sizes are $n\in\{64,128,256,512,1024\}$, and all other MH settings (burn-in, thinning, greedy initialization) match the main Section~\ref{app:sim2-accuracy} setup.

\subsubsection{Full Results}

\paragraph{Direct permutation estimator.}
Figure~\ref{fig:app-simulation-are} reports $\mathrm{ARE}(m)$ over the sampling budget on the shared $2\times 4$ Scenario$\times$Case grid;
Table~\ref{tab:app-simulation-aucc} reports the final AUCC;
Figure~\ref{fig:app-simulation-unique} reports the number of distinct coalition subsets evaluated;
Table~\ref{tab:app-simulation-runtime} reports the final non-utility runtime.
Error drops with the sampling budget throughout the grid, and the slope flattens as $n$ grows.
Scenario~2 is harder than Scenario~1 at matched $n$, consistent with the extra mixing burden from within-block directed cycles.
The unique-subset curves flatten as $m$ grows, reflecting increasing cache reuse along each chain.

\paragraph{Surrogate-assisted estimators.}
Figure~\ref{fig:app-surrogate-are} compares the direct permutation estimator with the two surrogate variants under matched utility-evaluation budgets.
Surrogate methods reduce ARE in most settings, and the quadratic surrogate tends to dominate on Scenario~2 where pairwise interactions are informative.
Table~\ref{tab:app-surrogate-runtime} and Table~\ref{tab:app-surrogate-mem} report the non-utility runtime and the surrogate training memory.
The gain in ARE is bought at a substantial computational cost: quadratic-surrogate runtime and memory both grow quickly with $n$.
Unlike PASV, which can zero out $b_{ij}$ for any pair that violates the DAG, the GPASV quadratic surrogate has to retain all sampled pairwise coefficients, so memory is harder to trim.

\subsection{Simulation 3: Priority Sweeping}
\label{app:sim3-sweep}

Simulations~1 and~2 study the computational side of GPASV.
Simulation~3 asks a different question: how does the attribution respond when we move a single soft-priority scalar along a sweep, and how does that response differ between GPASV and PASV?
This is the empirical counterpart to the priority-sweeping diagnostic introduced in Section~\ref{subsec::method::sweep}, and it illustrates an expected consequence of GPASV's full support: $\omega$ at finite $\beta$ leaves more room for $\lambda$ to reshape the order distribution than PASV's binary DAG does.

\subsubsection{Setup: Group Selection and Utility Functions}

We fix $n=32$ and for each $p_{\mathrm{edge}}\in\{0.2,0.5,0.8\}$, draw a single DAG on $[n]$ by including each precedence edge $i<j$ independently with probability $p_{\mathrm{edge}}$.
We then fix a single group $H\subseteq[n]$ with $|H|=n/2$, and sweep two parameters: the soft-priority level $\lambda_0$ on $H$ and the hard-priority strength $\beta$.
Throughout, we write $\pi^{-1}_i$ for the position of player $i$ in the permutation $\pi$, i.e.\ the unique $t\in[n]$ with $\pi_t=i$.

\paragraph{Soft-priority sweep.}
We set $\lambda_i=\lambda_0$ for $i\in H$ and $\lambda_i=1$ otherwise, and sweep
$$
\lambda_0\in\{1,2,4,8,16,32,64,128\}.
$$
At $\lambda_0=1$ all players share the same soft priority, so GPASV and PASV both assign the same value to $H$ as to $[n]\setminus H$ up to graph effects; larger $\lambda_0$ progressively up-weights the members of $H$ in the backward stage-wise softmax, making them more likely to be selected for later forward positions, hence pushing them toward the tail of $\pi$.

\paragraph{Hard-priority sweep.}
We vary the temperature $\beta$ in \eqref{eq:gpasv} over
$$
\beta\in\{1,2,4,8\},
$$
and include the hard-priority extreme $\beta\to\infty$ (which, on a DAG, reduces GPASV to PASV by Theorem~\ref{thm:hard-limit}) as a reference curve.

\paragraph{Utility functions.}
We again use the sum-of-unanimity (SOU) family from Section~\ref{app:sim2-accuracy}, but now with $T_k$ ranging over general nonempty subsets of $[n]$ rather than contiguous intervals.
Concretely, we draw $d=n^2$ subsets $T_1,\ldots,T_d$ iid uniformly from the nonempty subsets of $[n]$ together with iid coefficients $c_k\sim\mathrm{Unif}(0.5,1.5)$, and set
$$
U^{\mathrm{SOU}}(S):=\sum_{k=1}^{d}c_k\,\mathbbm{1}\{T_k\subseteq S\}.
$$
For each $k$, the marginal value of player $i\in T_k$ under the indicator $\mathbbm{1}\{T_k\subseteq S\}$ is
$$
\mathbb{P}_{\pi\sim p^{(\lambda,\omega)}}\!\bigl(i \text{ is the last member of } T_k \text{ in }\pi\bigr).
$$
Hence $U^{\mathrm{SOU}}$ rewards late-positioned players.

To probe the opposite end of the position axis, we additionally introduce the \emph{sum-of-race} (SOR) family, obtained by replacing the indicator of $T_k\subseteq S$ with that of $T_k\cap S\ne\emptyset$:
$$
U^{\mathrm{SOR}}(S):=\sum_{k=1}^{d}c_k\,\mathbbm{1}\{T_k\cap S\neq\emptyset\}.
$$
For each $k$, the marginal value of player $i\in T_k$ under $\mathbbm{1}\{T_k\cap S\neq\emptyset\}$ is
$$
\mathbb{P}_{\pi\sim p^{(\lambda,\omega)}}\!\bigl(i \text{ is the first member of } T_k \text{ in }\pi\bigr),
$$
so $U^{\mathrm{SOR}}$ rewards early-positioned players.
We write $U^{\mathrm{SOU}}$ and $U^{\mathrm{SOR}}$ throughout to distinguish the two utilities.
Running the same sweep on $U^{\mathrm{SOU}}$ and $U^{\mathrm{SOR}}$ therefore exposes $\lambda_0$-sensitivity from two opposite ends.

\subsubsection{Connection to Limiting Cases}

Two extremes of $\beta$ bracket the sweep curves and help interpret the intermediate $\beta$'s.

\paragraph{$\beta\to\infty$ (hard-priority limit).}
By Theorem~\ref{thm:hard-limit}, the GPASV distribution concentrates on $\widetilde\Pi^{G_\omega}$, and when $G_\omega$ is a DAG we have $\widetilde\Pi^{G_\omega}=\Pi^{G_\omega}$ together with $M_\omega(S)=\max(S)$.
The stage-wise factors in \eqref{eq:gpasv} then reduce to PASV's stage-wise factors in \eqref{eq:pasv-prelim}, so the group-sum curve at $\beta\to\infty$ coincides with the PASV group-sum on the same DAG.
This is the dashed reference shown in the sweep plots.
Under PASV, the impact of $\lambda_0$ is bounded by the hard constraints of $G_\omega$: once the DAG order is satisfied, $\lambda$ can only redistribute within admissible sub-permutations, which already fixes most of the relative order of $H$ against $[n]\setminus H$.

\paragraph{$\beta=0$ (no pairwise priority).}
When $\beta=0$, the graph term $V_\omega$ drops out of \eqref{eq:gpasv} and the distribution reduces to the Plackett--Luce distribution with worths $(\lambda_i)$ (Section~\ref{subsec::method::connections}).
In this regime, $\lambda_0$ has the largest room to move the group-sum attribution because there is no hard constraint to pull the order back; the position of each member of $H$ is determined by the worths alone.

Intermediate $\beta\in\{1,2,4,8\}$ interpolates between the two extremes.
Because GPASV softens the hard priority at finite $\beta$, $\lambda_0$ can push members of $H$ later even when such moves violate an edge of $G_\omega$; the softness is controlled by $\beta$.

\subsubsection{Implementation and Baseline Computation}

For each $(\lambda_0,\beta)$ pair with finite $\beta$, we sample $N_{\mathrm{MC}}=10{,}000$ after $10,000$ burn-in periods from $p^{(\lambda,\omega)}$ using the adjacent-swap MH sampler (Algorithm~\ref{alg:gpasv-adjacent-mh}) with greedy initialization (Algorithm~\ref{alg:gpasv-init}).
We replicate each setting over $10$ independent chains and report mean and standard deviation.
Along the $\lambda_0$ sweep at fixed $\beta$, we reuse previous samples through the SNIS rule described in Section~\ref{sec:snis-reuse}, monitoring the effective sample size and drawing additional fresh samples when it falls below target.
This keeps the number of utility evaluations manageable across the sweep.

The PASV reference is obtained by sampling PASV permutations directly on $\Pi^{G_\omega}$ following the adjacent-swap MH sampler of \citet{lee2026priority} and estimating the group-sum PASV by the same Monte Carlo averaging, with matched sample count.

For each sample we record both the position $\pi^{-1}_i$ of every member $i\in H$, used to compute the mean forward position
$$
\bar r_H(\lambda_0,\beta)
:=\frac{1}{|H|}\sum_{i\in H}\mathbb{E}_{\pi\sim p^{(\lambda,\omega)}}[\pi^{-1}_i],
$$
by its empirical estimate and the marginal contributions under $U^{\mathrm{SOU}}$ and $U^{\mathrm{SOR}}$, used to form the group-sum GPASV $\sum_{i\in H}\psi_i(U)$.

\subsubsection{Full Results}

Figure~\ref{fig:app-sweep-value} shows the full sweep.
The left panel reports $\bar r_H(\lambda_0,\beta)$, the middle and right panels report $\sum_{i\in H}\psi_i(U)$ under $U^{\mathrm{SOU}}$ and $U^{\mathrm{SOR}}$, and the dashed black curve in each panel is the PASV reference.

Three patterns are visible.
First, the mean forward position is monotone in $\lambda_0$: larger $\lambda_0$ pushes members of $H$ later in the sampled order under both GPASV and PASV.
Second, the slope of the position curve is steeper at smaller $\beta$, reflecting softer pairwise priority; at $\beta=1$, $\lambda_0$ moves the mean position over a wider range than at $\beta=8$ or at $\beta\to\infty$ (PASV).
Third, the group-sum curves under $U^{\mathrm{SOU}}$ and $U^{\mathrm{SOR}}$ move in opposite directions, as expected from their definitions: the $U^{\mathrm{SOU}}$ group sum rises with $\lambda_0$ (more $H$-members arrive late and thus collect unanimity rewards), whereas the $U^{\mathrm{SOR}}$ group sum falls with $\lambda_0$ (fewer $H$-members arrive early and thus collect first-arrival rewards).
The transition sharpens monotonically in $\beta$, and the $\beta\to\infty$ reference tracks the $\beta=8$ curve closely throughout, consistent with Theorem~\ref{thm:hard-limit}.

\clearpage

\begin{figure}[t]
\centering
\includegraphics[width=\textwidth]{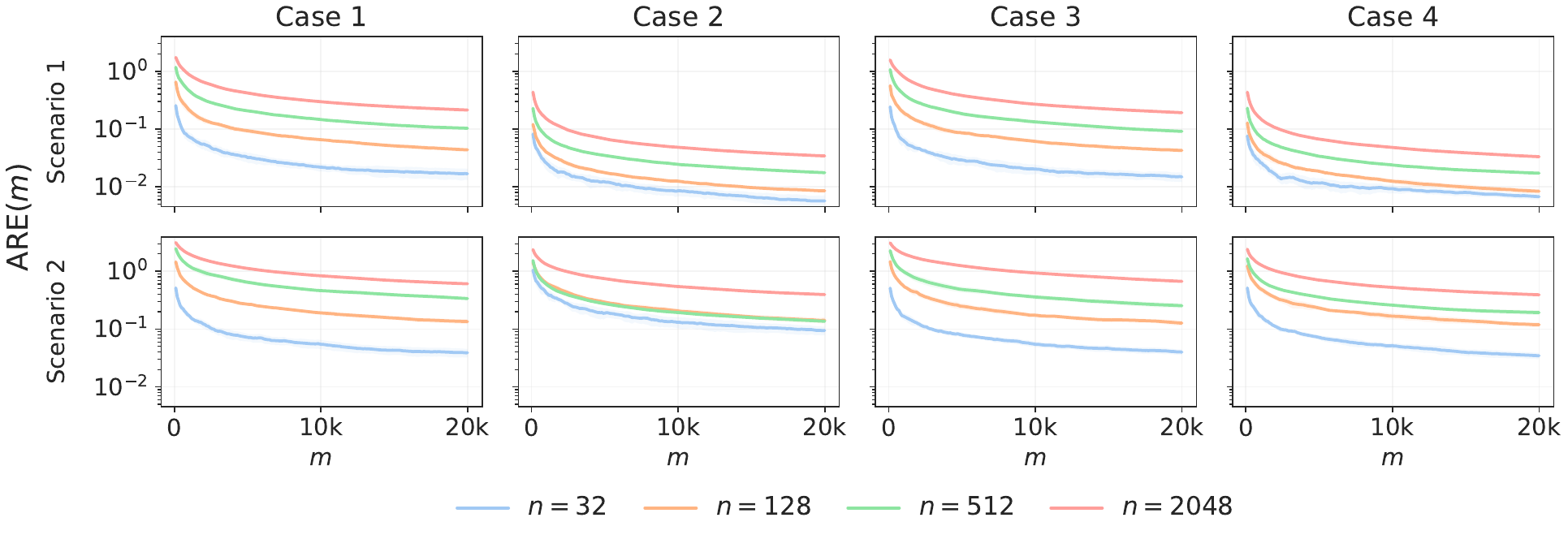}
\caption{Monte Carlo accuracy of the direct permutation estimator, reported as $\mathrm{ARE}(m)$ over the post-burn-in sampling budget $m$.
All curves use greedy initialization.
Rows correspond to Scenario~1 and Scenario~2; columns correspond to Cases~1--4.}
\label{fig:app-simulation-are}
\end{figure}

\begin{figure}[t]
\centering
\includegraphics[width=\textwidth]{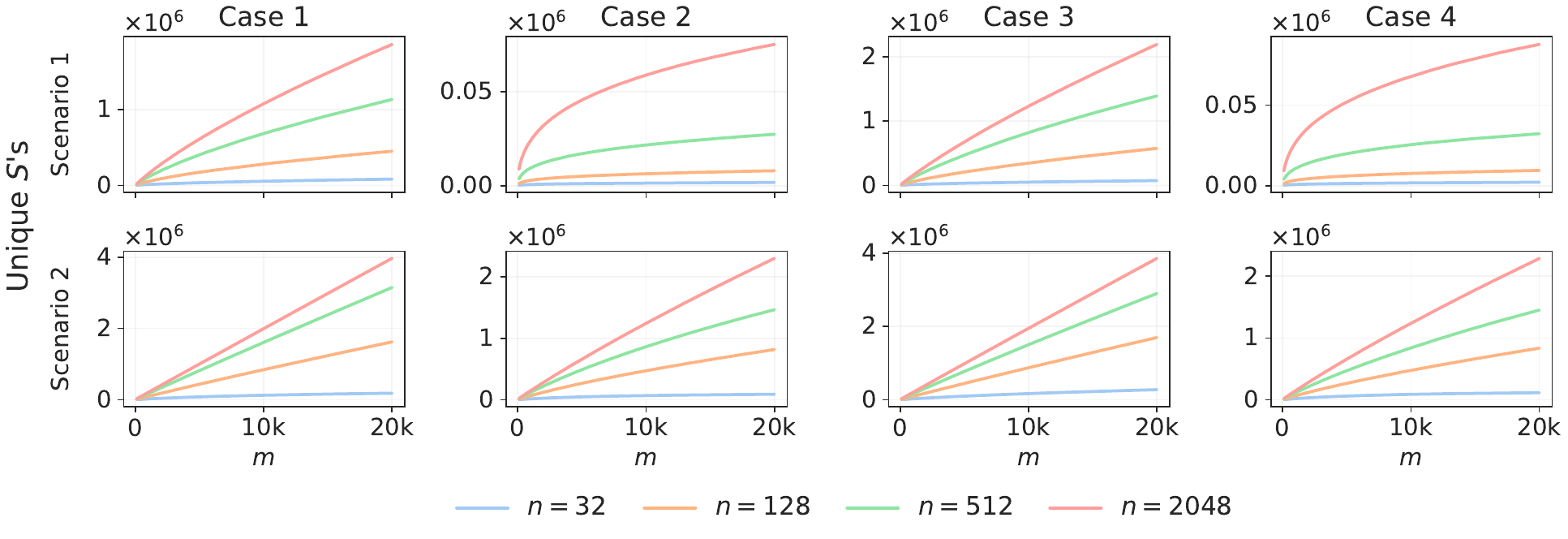}
\caption{Number of distinct coalition subsets evaluated by the direct permutation estimator, over the sampling budget.
Flatter growth indicates stronger cache reuse.}
\label{fig:app-simulation-unique}
\end{figure}

\begin{table}[t]
\centering
\caption{Final AUCC for the direct permutation estimator under greedy initialization.
Rows index $n$; columns index Cases~1--4 within each scenario.
Entries are mean (standard deviation) across $10$ independent MH repetitions on the fixed instance.}
\label{tab:app-simulation-aucc}
\scriptsize
\setlength{\tabcolsep}{3.5pt}
\resizebox{\textwidth}{!}{%
\begin{tabular}{c|cccc|cccc}
\toprule
& \multicolumn{4}{c|}{\textbf{Scenario 1}}
& \multicolumn{4}{c}{\textbf{Scenario 2}} \\
$n$ & Case 1 & Case 2 & Case 3 & Case 4 & Case 1 & Case 2 & Case 3 & Case 4 \\
\midrule
32   & 0.0310 (0.0024) & 0.0111 (0.0012) & 0.0278 (0.0033) & 0.0115 (0.0011) & 0.0708 (0.0075) & 0.1772 (0.0249) & 0.0741 (0.0050) & 0.0681 (0.0048) \\
128  & 0.0875 (0.0052) & 0.0167 (0.0006) & 0.0811 (0.0040) & 0.0170 (0.0010) & 0.2491 (0.0076) & 0.2685 (0.0116) & 0.2361 (0.0144) & 0.2160 (0.0151) \\
512  & 0.1919 (0.0072) & 0.0326 (0.0006) & 0.1726 (0.0085) & 0.0319 (0.0008) & 0.5831 (0.0310) & 0.2540 (0.0078) & 0.4506 (0.0327) & 0.3246 (0.0079) \\
2048 & 0.3799 (0.0129) & 0.0641 (0.0012) & 0.3465 (0.0099) & 0.0634 (0.0009) & 0.9900 (0.0178) & 0.6552 (0.0113) & 1.0791 (0.0274) & 0.6352 (0.0076) \\
\bottomrule
\end{tabular}
}
\end{table}

\begin{table}[t]
\centering
\caption{Final non-utility runtime (seconds) for the direct permutation estimator under greedy initialization.
This excludes the time spent evaluating $U(S)$ and isolates sampling plus bookkeeping cost.
Entries are mean (standard deviation) across $10$ independent MH repetitions.}
\label{tab:app-simulation-runtime}
\scriptsize
\setlength{\tabcolsep}{3.5pt}
\resizebox{\textwidth}{!}{%
\begin{tabular}{c|cccc|cccc}
\toprule
& \multicolumn{4}{c|}{\textbf{Scenario 1}}
& \multicolumn{4}{c}{\textbf{Scenario 2}} \\
$n$ & Case 1 & Case 2 & Case 3 & Case 4 & Case 1 & Case 2 & Case 3 & Case 4 \\
\midrule
32   & 4.37 (0.50) & 3.15 (0.07) & 4.26 (0.01) & 3.31 (0.01) & 2.82 (0.01) & 2.54 (0.01) & 3.14 (0.01) & 2.71 (0.01) \\
128  & 10.41 (0.08) & 8.48 (0.01) & 10.88 (0.05) & 8.72 (0.01) & 10.89 (0.09) & 8.57 (0.02) & 10.61 (0.02) & 8.70 (0.18) \\
512  & 45.52 (0.47) & 38.23 (0.26) & 43.10 (0.23) & 39.08 (0.20) & 47.98 (0.62) & 38.91 (0.34) & 47.34 (0.67) & 38.86 (0.34) \\
2048 & 1240.88 (49.17) & 1199.02 (64.10) & 1195.87 (57.89) & 1232.67 (66.66) & 1503.55 (53.94) & 1045.47 (8.85) & 1205.56 (24.25) & 993.38 (25.95) \\
\bottomrule
\end{tabular}
}
\end{table}

\begin{figure}[t]
\centering
\includegraphics[width=\textwidth]{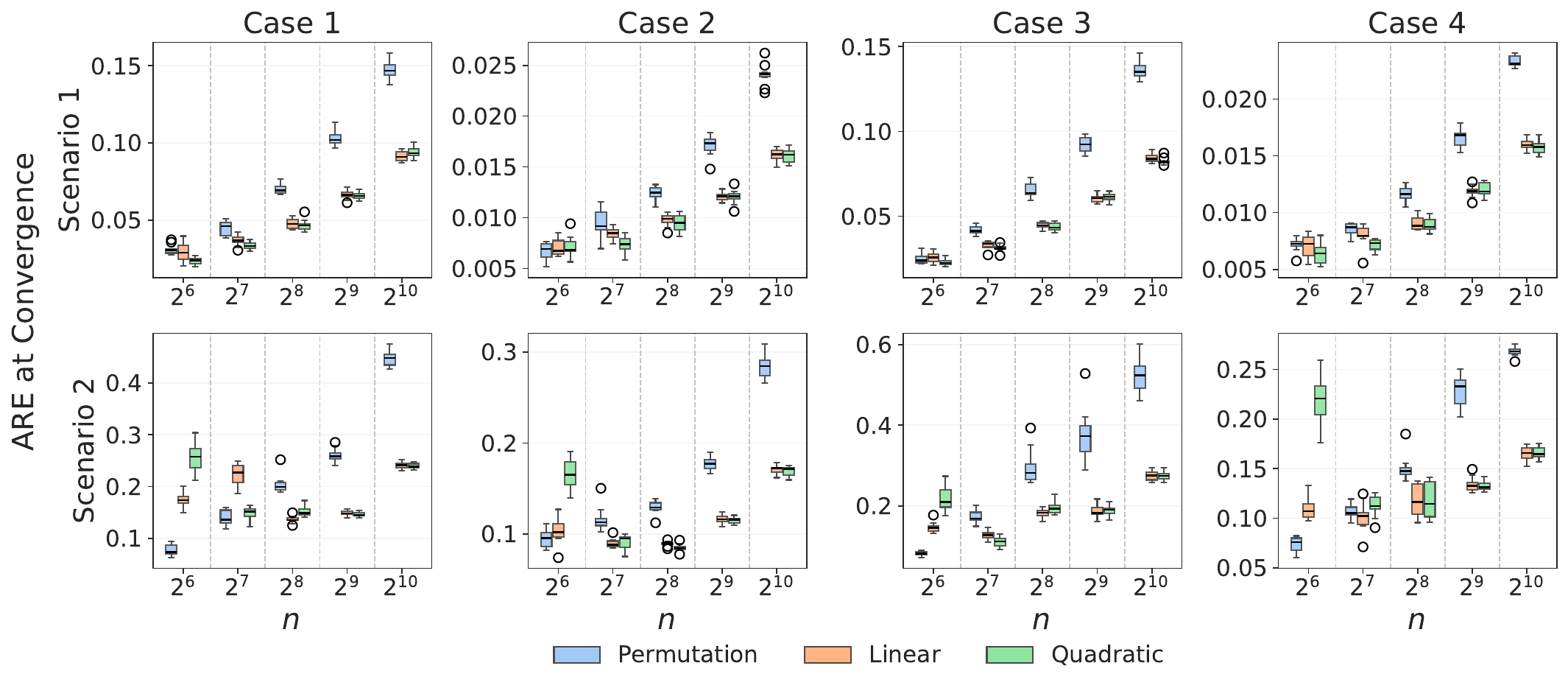}
\caption{Final ARE under matched utility-evaluation budgets for the direct and the two surrogate-assisted estimators.}
\label{fig:app-surrogate-are}
\end{figure}

\begin{table}[t]
\centering
\caption{Non-utility runtime (seconds) for the three estimators under matched utility budgets.
Entries are mean (standard deviation) across independent repetitions.}
\label{tab:app-surrogate-runtime}
\scriptsize
\setlength{\tabcolsep}{3.2pt}
\resizebox{\textwidth}{!}{%
\begin{tabular}{cc|cccc|cccc}
\toprule
& & \multicolumn{4}{c|}{\textbf{Scenario 1}} & \multicolumn{4}{c}{\textbf{Scenario 2}} \\
Method & $n$ & Case 1 & Case 2 & Case 3 & Case 4 & Case 1 & Case 2 & Case 3 & Case 4 \\
\midrule
\multicolumn{10}{l}{\textit{Permutation}} \\
& 64 & 5.98 (0.02) & 4.94 (0.01) & 6.31 (0.01) & 5.17 (0.01) & 5.21 (0.01) & 4.69 (0.02) & 5.28 (0.02) & 4.68 (0.02) \\
& 128 & 9.73 (0.01) & 8.53 (0.02) & 10.13 (0.01) & 8.80 (0.01) & 10.13 (0.02) & 8.15 (0.04) & 10.89 (0.10) & 8.26 (0.02) \\
& 256 & 17.72 (0.03) & 15.94 (0.03) & 18.12 (0.05) & 16.21 (0.02) & 19.37 (0.06) & 15.99 (0.07) & 19.65 (0.16) & 15.94 (0.04) \\
& 512 & 39.77 (0.19) & 36.45 (0.16) & 39.68 (0.10) & 36.67 (0.12) & 41.84 (0.15) & 34.36 (0.06) & 42.44 (0.17) & 34.03 (0.15) \\
& 1024 & 209.25 (3.47) & 195.76 (2.46) & 203.22 (1.79) & 196.92 (1.99) & 254.92 (3.92) & 205.55 (0.51) & 255.72 (3.14) & 206.43 (1.47) \\
\midrule
\multicolumn{10}{l}{\textit{Linear surrogate}} \\
& 64 & 36.40 (0.26) & 26.51 (0.15) & 38.81 (0.23) & 27.21 (0.19) & 44.20 (0.39) & 35.50 (0.38) & 44.38 (0.24) & 35.58 (0.26) \\
& 128 & 66.61 (0.34) & 46.72 (0.19) & 69.92 (0.19) & 48.57 (0.23) & 91.95 (0.24) & 61.27 (0.28) & 98.72 (0.20) & 61.63 (0.36) \\
& 256 & 131.41 (0.51) & 91.73 (0.17) & 133.43 (0.76) & 92.27 (0.53) & 174.25 (1.43) & 120.17 (0.39) & 180.50 (0.31) & 117.79 (0.51) \\
& 512 & 290.04 (1.83) & 191.07 (0.65) & 272.53 (1.29) & 188.95 (0.92) & 309.91 (2.41) & 231.47 (1.18) & 311.12 (1.90) & 226.24 (1.08) \\
& 1024 & 1053.69 (19.83) & 663.81 (6.91) & 835.67 (4.70) & 649.31 (4.14) & 886.68 (10.01) & 725.34 (4.64) & 893.16 (6.54) & 735.43 (3.17) \\
\midrule
\multicolumn{10}{l}{\textit{Quadratic surrogate}} \\
& 64 & 36.77 (0.25) & 26.49 (0.10) & 39.15 (0.25) & 27.15 (0.21) & 45.32 (0.45) & 36.08 (0.24) & 45.28 (0.34) & 36.26 (0.25) \\
& 128 & 68.86 (0.19) & 46.89 (0.26) & 71.97 (0.24) & 48.63 (0.16) & 95.70 (0.36) & 63.95 (0.22) & 103.02 (0.61) & 65.00 (0.33) \\
& 256 & 144.39 (0.74) & 92.30 (0.28) & 152.45 (1.20) & 92.73 (0.18) & 191.26 (1.16) & 138.65 (0.88) & 198.05 (0.71) & 138.68 (1.59) \\
& 512 & 359.25 (3.06) & 194.82 (0.74) & 342.59 (2.79) & 193.55 (0.82) & 382.64 (1.43) & 317.75 (3.33) & 388.41 (2.06) & 315.09 (2.61) \\
& 1024 & 1599.66 (23.32) & 700.01 (6.31) & 1558.86 (18.52) & 701.05 (4.81) & 1520.96 (14.39) & 1378.02 (9.58) & 1575.88 (16.45) & 1409.52 (10.56) \\
\bottomrule
\end{tabular}
}
\end{table}

\begin{table}[t]
\centering
\caption{Training memory (GB) used by the surrogate fit under matched utility budgets.
The permutation estimator does not fit a surrogate and is omitted.}
\label{tab:app-surrogate-mem}
\scriptsize
\setlength{\tabcolsep}{3.2pt}
\resizebox{\textwidth}{!}{%
\begin{tabular}{cc|cccc|cccc}
\toprule
& & \multicolumn{4}{c|}{\textbf{Scenario 1}} & \multicolumn{4}{c}{\textbf{Scenario 2}} \\
Method & $n$ & Case 1 & Case 2 & Case 3 & Case 4 & Case 1 & Case 2 & Case 3 & Case 4 \\
\midrule
\multicolumn{10}{l}{\textit{Linear surrogate}} \\
& 64 & 0.033 (0.000) & 0.001 (0.000) & 0.037 (0.000) & 0.001 (0.000) & 0.138 (0.000) & 0.093 (0.000) & 0.139 (0.000) & 0.092 (0.000) \\
& 128 & 0.132 (0.001) & 0.002 (0.000) & 0.155 (0.002) & 0.003 (0.000) & 0.298 (0.000) & 0.194 (0.000) & 0.289 (0.000) & 0.183 (0.000) \\
& 256 & 0.443 (0.003) & 0.009 (0.000) & 0.547 (0.004) & 0.011 (0.000) & 0.580 (0.001) & 0.592 (0.000) & 0.580 (0.001) & 0.568 (0.006) \\
& 512 & 1.151 (0.001) & 0.031 (0.000) & 1.144 (0.002) & 0.036 (0.000) & 1.146 (0.001) & 1.158 (0.001) & 1.144 (0.002) & 1.155 (0.001) \\
& 1024 & 2.294 (0.003) & 0.106 (0.001) & 2.298 (0.004) & 0.126 (0.001) & 2.304 (0.003) & 2.312 (0.002) & 2.297 (0.003) & 2.318 (0.002) \\
\midrule
\multicolumn{10}{l}{\textit{Quadratic surrogate}} \\
& 64 & 0.094 (0.001) & 0.002 (0.000) & 0.102 (0.001) & 0.002 (0.000) & 0.409 (0.001) & 0.290 (0.001) & 0.409 (0.000) & 0.288 (0.000) \\
& 128 & 0.666 (0.008) & 0.012 (0.000) & 0.772 (0.009) & 0.015 (0.000) & 1.526 (0.003) & 0.986 (0.003) & 1.474 (0.003) & 0.915 (0.002) \\
& 256 & 4.030 (0.024) & 0.081 (0.001) & 4.922 (0.041) & 0.098 (0.001) & 5.305 (0.008) & 5.620 (0.008) & 5.336 (0.011) & 5.394 (0.061) \\
& 512 & 20.592 (0.042) & 0.566 (0.006) & 20.403 (0.041) & 0.657 (0.006) & 20.549 (0.034) & 21.000 (0.027) & 20.477 (0.041) & 20.989 (0.042) \\
& 1024 & 106.818 (0.211) & 3.724 (0.031) & 107.120 (0.217) & 4.409 (0.035) & 107.556 (0.187) & 108.276 (0.155) & 107.107 (0.175) & 108.548 (0.215) \\
\bottomrule
\end{tabular}
}
\end{table}

\begin{figure}[h!]
\centering
\includegraphics[width=\textwidth]{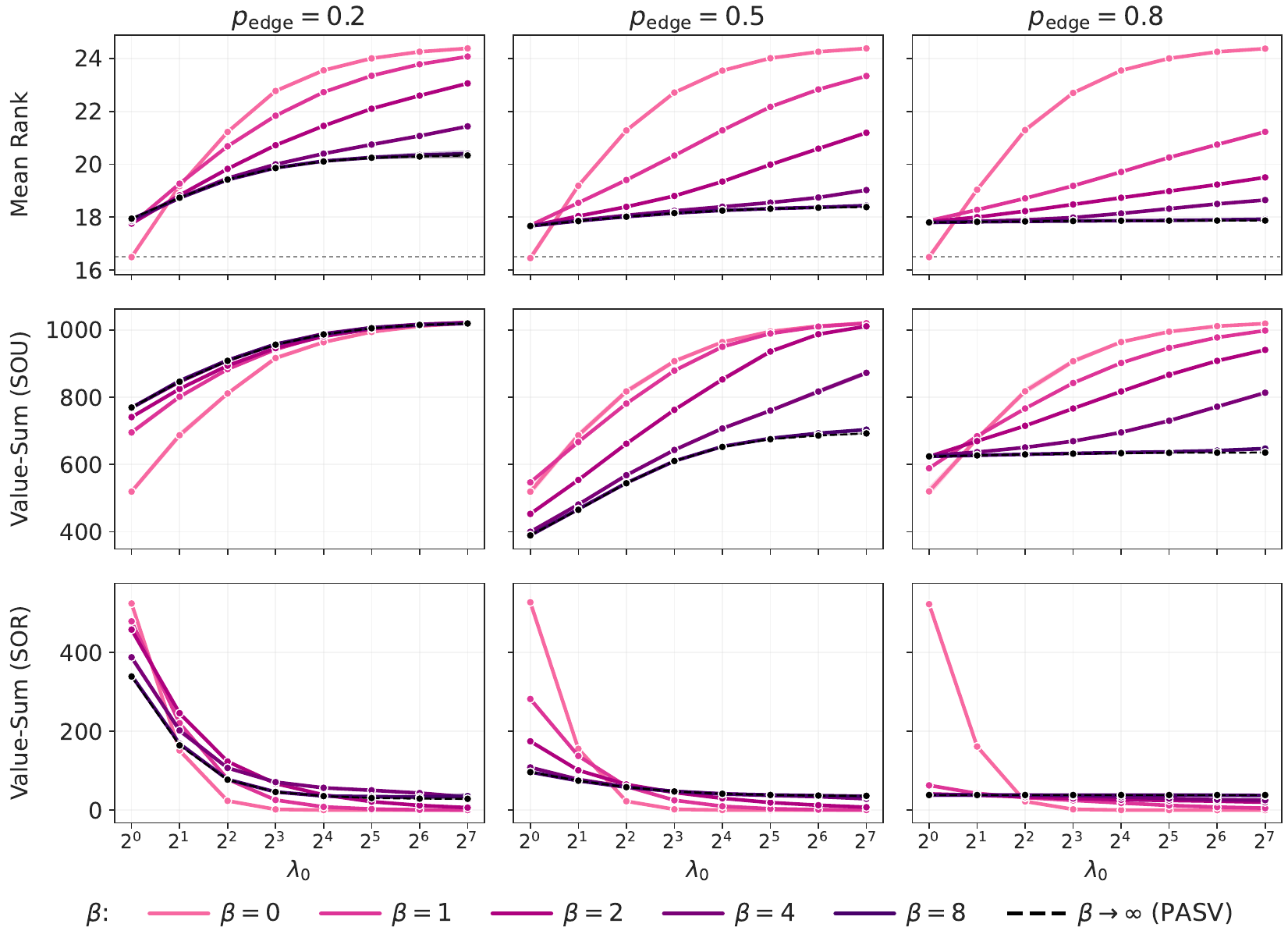}
\caption{Priority sweeping results.
Rows: mean forward position, SOU group-sum, and SOR group-sum.
Colored curves vary $\beta\in\{1,2,4,8\}$; the dashed black curve is the $\beta\to\infty$ (PASV on $G_\omega$) reference.}
\label{fig:app-sweep-value}
\end{figure}

\clearpage

\section{Additional Details and Results for LLM Evaluation in Section~\ref{sec:application}}
\label{app:application}

This appendix collects the implementation details and the full empirical results of the LLM evaluation in Section~\ref{sec:application}.
Throughout, a \emph{player} $i$ is one of the $n=20$ LLMs in Table~\ref{tab:llm-model-set} and a \emph{coalition} $S\subseteq[n]$ is a subset of them; GPASV attributes the coalition utility $U_{\mathrm{ens}}(S)$ across players under the hard-priority weights $\omega_{ij}$ and the soft-priority vector $\lambda_i=\exp(-\alpha z_i)$ from Section~\ref{sec:application}.

The utility-evaluation stage of this experiment requires GPU computing and was run on the same computing cluster as Appendix~\ref{app:simulation}, pinned to NVIDIA H200 GPUs
serving \texttt{Qwen3.5-35B-A3B-FP8} via vLLM, for a total of approximately 30 GPU-days across the $80$ MT-Bench prompts.
All downstream stages are CPU-only, and were run on the same cluster on multiple Intel Xeon CPUs.

\subsection{Data and Priority Construction}
\label{app:application-construction}

\paragraph{Model set and open/paid labels.}
Table~\ref{tab:llm-model-set} lists the $20$ models used as players in Section~\ref{sec:application}, together with the binary label $z_i\in\{0,1\}$ used for the soft priority.
The model set is exactly the intersection of MT-Bench \citep{zheng2023judging} and Chatbot Arena \citep{chiang2024chatbot}: every model we score on MT-Bench must also have a sufficient number of Chatbot Arena comparisons for the hard priority to be well-defined.
We label the five proprietary API models (\texttt{gpt-4}, \texttt{gpt-3.5-turbo}, \texttt{claude-v1}, \texttt{claude-instant-v1}, \texttt{palm-2}) as \emph{paid}, and all remaining MT-Bench/Arena-overlap models as \emph{open-source}.
The label is a coarse deployment signal rather than a license audit; the point of the soft priority in Section~\ref{sec:application} is merely to study how strongly an evaluator who explicitly prefers open-source models over paid APIs reshapes the attribution, so $z_i$ is set to $1$ for open-source and $0$ for paid.

\begin{table}[h!]
\centering
\small
\begin{tabular}{@{}ll@{\hspace{1.5em}}ll@{}}
\toprule
Model & Label & Model & Label \\
\midrule
\texttt{gpt-4} & paid & \texttt{gpt4all-13b-snoozy} & open-source \\
\texttt{claude-v1} & paid & \texttt{mpt-7b-chat} & open-source \\
\texttt{claude-instant-v1} & paid & \texttt{rwkv-4-raven-14b} & open-source \\
\texttt{gpt-3.5-turbo} & paid & \texttt{alpaca-13b} & open-source \\
\texttt{guanaco-33b} & open-source & \texttt{oasst-pythia-12b} & open-source \\
\texttt{vicuna-13b} & open-source & \texttt{fastchat-t5-3b} & open-source \\
\texttt{wizardlm-13b} & open-source & \texttt{chatglm-6b} & open-source \\
\texttt{palm-2} & paid & \texttt{stablelm-tuned-alpha-7b} & open-source \\
\texttt{vicuna-7b} & open-source & \texttt{dolly-v2-12b} & open-source \\
\texttt{koala-13b} & open-source & \texttt{llama-13b} & open-source \\
\bottomrule
\end{tabular}
\caption{Model set and open/paid labels $z_i$ used in the LLM evaluation experiment.}
\label{tab:llm-model-set}
\end{table}

\paragraph{Coalition utility from MT-Bench.}
MT-Bench consists of $Q=80$ two-turn prompts organized into eight categories with ten prompts each:
\texttt{Writing}, \texttt{Roleplay}, \texttt{Reasoning}, \texttt{Math}, \texttt{Coding}, \texttt{Extraction}, \texttt{Knowledge I}, and \texttt{Knowledge II}.
The released benchmark bundles first- and second-turn model responses for every player, and these responses are treated as fixed candidate answers throughout the experiment; we never re-query the $20$ players themselves.
Following the MT-Bench judging protocol, we use two judge variants.
For \texttt{Reasoning}, \texttt{Math}, and \texttt{Coding}, the judge is given MT-Bench's released reference answer and is asked to score correctness relative to it (with-reference variant).
For the remaining five categories the judge scores the answer directly from the user prompt and the assistant response (no-reference variant).
For a coalition $S$ and a prompt $q$, the first-turn candidates $\{y^{(1)}_i\}_{i\in S}$ are passed to an aggregator LLM that synthesizes a single ensemble answer $y^{(1)}_S$; a judge LLM scores $y^{(1)}_S$ on $\{1,\ldots,10\}$, producing a score $r^{(1)}(S,q)$.
For the second turn we feed the first-turn aggregated answer $y^{(1)}_S$ as the previous assistant turn, aggregate the second-turn candidates into $y^{(2)}_S$, and score analogously, producing $r^{(2)}(S,q)$.
The prompt-level score is the mean of the two turn scores, and the coalition utility is the average over prompts,
$$
U_{\mathrm{ens}}(S)
=\frac{1}{Q}\sum_{q=1}^{Q}\frac{r^{(1)}(S,q)+r^{(2)}(S,q)}{2},
\qquad U_{\mathrm{ens}}(\emptyset):=0.
$$

\paragraph{Hard-priority graph from Chatbot Arena.}
Chatbot Arena supplies pairwise human preferences between model responses; we use the public comparison counts to construct the hard priority.
For every ordered pair $(i,j)$ with at least $50$ recorded comparisons, let $\widehat p_{ij}$ denote the empirical win probability of $i$ over $j$, and let $i$ denote the majority-preferred side so that $\widehat p_{ij}\ge 1/2$.
We then set
$$
\omega_{ij}=\widehat p_{ij}-\tfrac12,
\qquad
\omega_{ji}=0,
$$
so that the additive priority on edge $i\to j$ encodes how strongly the Arena crowd prefers $i$ over $j$.
Pairs with fewer than $50$ comparisons are treated as non-edges ($\omega_{ij}=\omega_{ji}=0$).
Larger $\beta$ in \eqref{eq:gpasv} more strongly suppresses permutations that place a majority-losing model before its majority-preferred counterpart.
The resulting graph $G_\omega$ is \emph{not} a DAG: majority preferences form several intransitive triangles among mid-tier models, which is precisely the regime that motivates GPASV over PASV.

\paragraph{Input scales.}
Following the pre-processing recommendation in Section~\ref{subsec::method::GPASV}, Table~\ref{tab:llm-input-scale} reports the numerical scales of the two priority inputs.
Both lie in $[0,1]$, and we use them without further rescaling.

\begin{table}[h!]
\centering
\small
\begin{tabular}{@{}lrrr@{}}
\toprule
 & SD & Min & Max \\
\midrule
Raw soft priority score $z_i$ & $0.433$ & $0$ & $1$ \\
Edge weight $\omega_{ij}$ & $0.110$ & $0$ & $0.446$ \\
\bottomrule
\end{tabular}
\caption{Numerical scales of the priority inputs in the LLM evaluation experiment.}
\label{tab:llm-input-scale}
\end{table}

\subsection{Aggregator–Judge Pipeline}
\label{app:application-llm-pipeline}

\paragraph{Model choice and inference setup.}
Both the aggregator and the judge are instantiated with \texttt{Qwen3.5-35B-A3B-FP8} \citep{qwen3.5}.
The model is chosen because (i) it is strong enough to aggregate and score MT-Bench responses reliably yet (ii) small enough in FP8 form to be served locally with batched \texttt{vLLM} inference \citep{kwon2023efficient}, which is essential at the cached-computation scale described in Section~\ref{app:application-llm-sampling}.
Using a single model for both roles removes judge--aggregator style mismatches.
We cap generation at $4096$ new tokens per call, and sample with temperature $1.0$, top-$p=0.95$, top-$k=20$, following the official Qwen3.5 recommendation for instruct (non-thinking) mode on reasoning tasks.
Every realized $(S,q)$ output is cached and reused across all priority regimes, so all priority comparisons are made against the same fixed utility realization.

\paragraph{Prompt templates.}
The aggregator and judge prompt templates are shown in Appendix~\ref{app:application-llm-prompts}; we summarize their structure here.
The aggregator receives the user prompt together with a randomly shuffled list of candidate responses, each delimited by explicit start/end markers, and is instructed to merge, edit, and reorganize the candidates while preserving the user's requested format and without adding external facts.
For the second turn, the first-turn aggregated answer $y^{(1)}_S$ is inserted as the previous assistant message before presenting the second-turn candidates.
The no-reference judge prompt consists of an \texttt{[Instruction]} block, the \texttt{[Question]} block, and the assistant answer delimited by start/end markers; the with-reference variant additionally inserts a reference-answer block and asks the judge to evaluate correctness against it.
For second-turn judging, the full two-turn conversation is presented in the \texttt{Assistant A} format, so that coherence across turns is scored rather than each turn in isolation.
In all variants, the judge is asked for a brief explanation followed by a final \texttt{[[rating]]} on the $1$--$10$ scale.
Placeholders \texttt{\{x1\}}, \texttt{\{x2\}}, \texttt{\{y1S\}}, \texttt{\{y2S\}}, \texttt{\{r1\}}, \texttt{\{r2\}}, and \texttt{\{candidate answer $k$\}} are filled by the corresponding MT-Bench prompt, aggregated response, reference answer, or candidate-model response at runtime.

\subsection{Sampling and Computation Protocol}
\label{app:application-llm-sampling}

\paragraph{MH hyperparameters.}
For every priority regime $(\alpha,\beta)$ with $\beta>0$, we draw $N_{\mathrm{MC}}=1000$ permutations from $p^{(\lambda,\omega)}$ using the adjacent-swap Metropolis--Hastings sampler of Section~\ref{sec:adjacent-swap-sampler} with greedy initialization (Algorithm~\ref{alg:gpasv-init}).
We use burn-in $10^5$ and thinning interval $10^3$; at $n=20$ the adjacent-swap chain mixes quickly relative to the simulation-study scales, so this conservative schedule leaves the effective sample size close to $N_{\mathrm{MC}}$.
At the baseline $(\alpha,\beta)=(0,0)$ the sampler reduces to uniform sampling over $\Pi$, which we implement by direct iid uniform draws.

\paragraph{Utility cache design.}
Evaluating one coalition on the full MT-Bench set requires four LLM calls per prompt: aggregation and judging for each of the two turns.
A naive direct Monte Carlo over $N_{\mathrm{MC}}$ permutations of $n$ players with $Q$ prompts would therefore require $4N_{\mathrm{MC}}nQ$ calls, which at $N_{\mathrm{MC}}=1000$, $n=20$, $Q=80$ amounts to $6.4\times 10^6$ calls for a \emph{single} priority regime, and across the $19$ regimes of Section~\ref{sec:application} it would scale linearly in the number of regimes.
We avoid this by caching at the level of (subset, prompt) pairs: each cache row is keyed by a $20$-bit subset mask and an MT-Bench prompt index, and stores both aggregated answers $y^{(1)}_S,y^{(2)}_S$, both raw judge responses, and the two parsed ratings.
Because GPASV only evaluates $U_{\mathrm{ens}}$ on prefix coalitions of sampled permutations, and because neighboring priority regimes along each sweep share most prefixes through the SNIS permutation reuse described next, the realized number of distinct subset evaluations is far below the worst-case $N_{\mathrm{MC}}n$.
Once the cache has been populated along the first sweep, all subsequent priority regimes read the same cache: changing $(\alpha,\beta)$ changes only the sampled order distribution, not $U_{\mathrm{ens}}$.

\paragraph{Sweep traversal and SNIS/ESS reuse.}
The three sweeps of Section~\ref{sec:application} are traversed in increasing temperature order from the shared baseline:
$$
(0,0)\to(1,0)\to(2,0)\to(4,0)\to(8,0)\to(16,0)\to(32,0),
$$
and analogously for the $\beta$-only and joint $\alpha=\beta$ sweeps.
Between neighboring regimes we apply the SNIS/ESS reuse rule from Section~\ref{sec:snis-reuse}.
Given the reweighted effective sample size $\mathrm{ESS}$ of the previous-regime samples under the new target, we add
$$
N_{\mathrm{new}}
=\min\!\left\{N_{\mathrm{MC}},\;\max\!\left(500,\,\lceil N_{\mathrm{MC}}-\mathrm{ESS}\rceil\right)\right\}
$$
fresh permutations from the new target and combine the reweighted estimate with the fresh Monte Carlo estimate using weights proportional to $\mathrm{ESS}$ and $N_{\mathrm{new}}$.
The lower bound of $500$ ensures that even near-overlapping regimes refresh a meaningful fraction of the sample; the upper bound of $N_{\mathrm{MC}}$ caps the cost of any single transition at a full fresh draw.

\paragraph{Realized acceleration.}
Table~\ref{tab:llm-reuse-comparison} reports the realized utility-evaluation cost under the four combinations of the two reuse mechanisms: subset reuse via the cache and permutation reuse via SNIS.
The cache provides the dominant fraction of the saving, reducing the per-prompt distinct subset count by a factor of $\approx 4.9$ on its own (B vs.~A); SNIS trims the effective permutation count from $19{,}000$ to $15{,}840$ along the three sweeps (Table~\ref{tab:llm-snis-savings}) but on its own only yields a $\approx 1.2\times$ reduction (C vs.~A).
Combining both gives a realized $\approx 5.4\times$ reduction over the naive baseline (D vs.~A).

\paragraph{Wall-clock cost and scaling.}
One un-cached evaluation of $U_{\mathrm{ens}}(S)$ requires $320$ LLM calls and takes $32.21$ seconds on average.
We used two NVIDIA H200 GPUs in parallel; almost all runtime is spent on these evaluations, while permutation sampling is negligible.
This bottleneck is shared by any Shapley-type method with an expensive utility: for $N$ permutations and $n$ players, direct GPASV needs $Nn$ utility calls, the same as permutation-based Shapley value or PASV, and the MH sampling itself requires no utility calls.
Note also that the total cost covers the full sweep over all $19$ settings, not a single valuation.
To scale to more players, one can run a sparser sweep grid or reduce the number of Monte Carlo permutations per setting.

\begin{table}[h!]
\centering
\caption{Realized utility-evaluation cost under the four combinations of subset reuse (cache) and permutation reuse (SNIS).
``$\#$ of Permutations'' is the total number of permutations sampled across the $19$ regimes after SNIS thinning along the three sweeps (with the baseline $(0,0)$ shared once);
``$\#$ of Subsets'' is the number of distinct prefix coalitions evaluated per MT-Bench prompt.}
\label{tab:llm-reuse-comparison}
\small
\setlength{\tabcolsep}{8pt}
\begin{tabular}{@{}lccrr@{}}
\toprule
Scenario & Cache & SNIS & $\#$ of Permutations & $\#$ of Subsets \\
\midrule
A.~No reuse               & \xmark & \xmark & $19{,}000$ & $380{,}000$ \\
B.~Subset reuse only      & \cmark & \xmark & $19{,}000$ & $78{,}220$  \\
C.~Permutation reuse only & \xmark & \cmark & $15{,}840$ & $316{,}800$ \\
D.~Both                   & \cmark & \cmark & $15{,}840$ & $70{,}511$  \\
\bottomrule
\end{tabular}
\end{table}

\begin{table}[h!]
\centering
\caption{Permutations saved by SNIS along the three sweep traversals.
``Without SNIS'' is the $6\,N_{\mathrm{MC}}$ fresh permutations that would be drawn per sweep ($6$ non-baseline temperatures $\times$ $N_{\mathrm{MC}}=1{,}000$); ``with SNIS'' is the realized $\sum N_{\mathrm{new}}$ under the rule above.
The shared baseline $(0,0)$ contributes a single $N_{\mathrm{MC}}=1{,}000$ draw across all three sweeps and is excluded from this table.}
\label{tab:llm-snis-savings}
\small
\setlength{\tabcolsep}{8pt}
\begin{tabular}{@{}lrrr@{}}
\toprule
Sweep & Without SNIS & With SNIS & Saved \\
\midrule
$\alpha$-only ($\beta=0$)  & $6{,}000$  & $4{,}007$  & $1{,}993$ \\
$\beta$-only ($\alpha=0$)  & $6{,}000$  & $5{,}403$  & $597$     \\
joint ($\alpha=\beta$)     & $6{,}000$  & $5{,}430$  & $570$     \\
\midrule
Three-sweep total          & $18{,}000$ & $14{,}840$ & $3{,}160$ \\
\bottomrule
\end{tabular}
\end{table}

\FloatBarrier
\subsection{Full Results}
\label{app:application-llm-results}

Section~\ref{sec:application} showed results for representative regimes;
this subsection reports the full results across all $19$ settings generated by the three one-dimensional sweeps on $(\alpha,\beta)$, in three views:
per-player value trajectories along each sweep (Figure~\ref{fig:llm-player-sweep}),
per-player rank under every $(\alpha,\beta)$ cell (Figure~\ref{fig:llm-rank-heatmap}),
and the top-$8$ ranked players at every individual regime (Figure~\ref{fig:llm-top8-all19}).

\paragraph{Per-player value trajectories.}
Figure~\ref{fig:llm-player-sweep} plots $\psi_i(U_{\mathrm{ens}})$ for each of the $20$ players as a function of the sweep temperature, one panel per sweep ($\alpha$-only, $\beta$-only, joint $\alpha=\beta$).
Solid lines are open-source models and dashed lines are paid.
This view emphasizes the \emph{trajectory} of each player as the priority strength is increased and is the curve-level companion to the group-sum view in Figure~\ref{fig:llm-group}.

\paragraph{Per-player rank on the $(\alpha,\beta)$ grid.}
Figure~\ref{fig:llm-rank-heatmap} reports, for each of the $20$ players, the rank of that player among all $20$ models on every supported $(\alpha,\beta)$ cell of the $7\times 7$ grid (the $19$ cells along the three sweeps; off-support cells are masked white).
Each per-player heatmap shares the same color scale: rank $1$ (best) is dark and rank $20$ (worst) is light.
Per-panel titles are colored by group (paid in dark blue, open-source in burnt orange) for quick reading; the same color convention is used in the bar charts of Figure~\ref{fig:llm-top8-all19}.
This view is the cell-level dual of Figure~\ref{fig:llm-player-sweep}: the trajectory in the latter corresponds to a row or column of cells in the former.

\paragraph{Top-$8$ ranking under every regime.}
Figure~\ref{fig:llm-top8} in the main text visualizes the top-$8$-valued models under representative regimes.
Figure~\ref{fig:llm-top8-all19} extends the same view to all $19$ regimes, arranged as a $7\times 3$ grid: rows index the temperature value ($0$ baseline together with $\{1,2,4,8,16,32\}$) and columns index the sweep type.
Walking down a column traces how the top-$8$ shifts as that sweep is strengthened from the baseline; walking across a row contrasts the three sweep types at a fixed temperature.

\begin{figure}[h!]
\centering
\includegraphics[width=\textwidth]{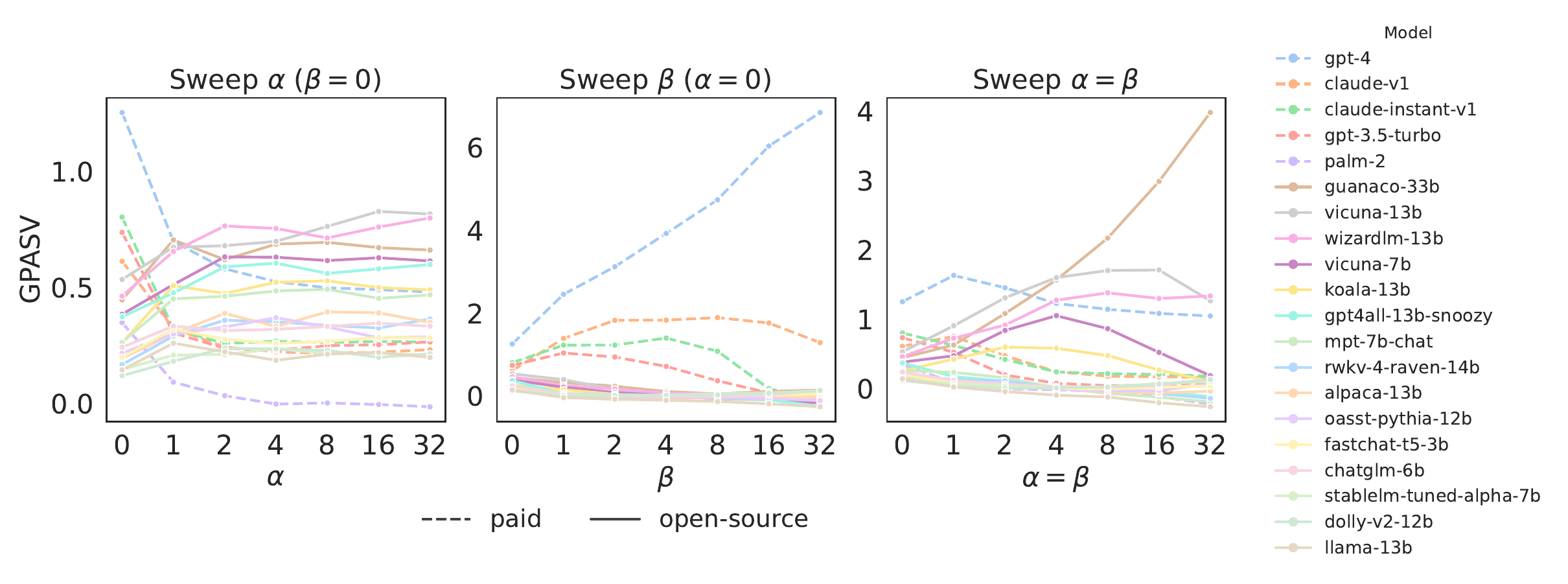}
\caption{Per-player GPASV $\psi_i(U_{\mathrm{ens}})$ along the three sweeps, averaged over the $80$ MT-Bench prompts.
Left: $\alpha$-only sweep ($\beta=0$).
Middle: $\beta$-only sweep ($\alpha=0$).
Right: joint sweep ($\alpha=\beta$).
Solid lines are open-source models; dashed lines are paid (\texttt{gpt-4}, \texttt{gpt-3.5-turbo}, \texttt{claude-v1}, \texttt{claude-instant-v1}, \texttt{palm-2}).}
\label{fig:llm-player-sweep}
\end{figure}

\begin{figure}[h!]
\centering
\includegraphics[width=\textwidth]{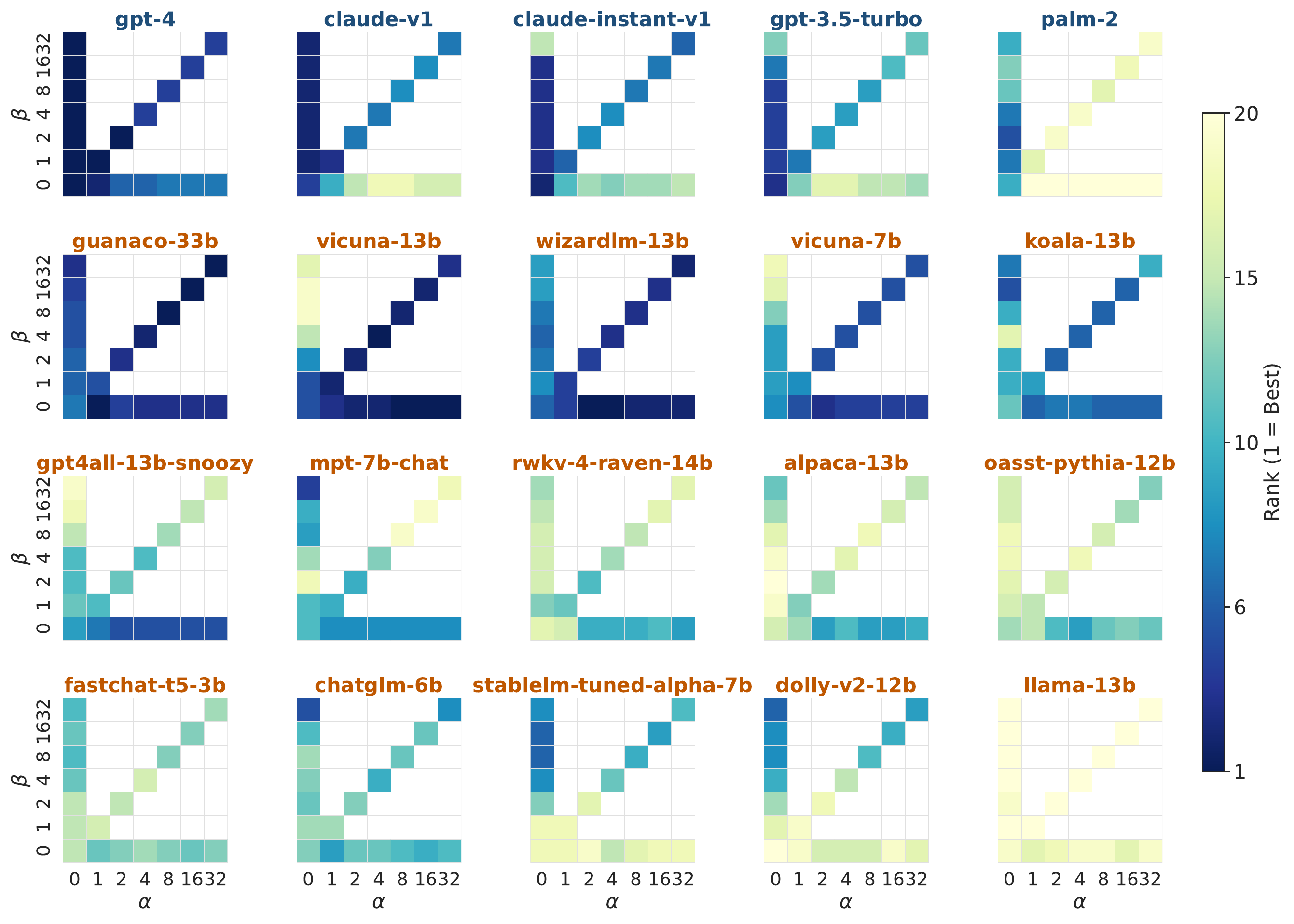}
\caption{Per-player rank heatmaps over the supported $(\alpha,\beta)$ cells.
Each subplot is one player; rows index $\beta$ and columns index $\alpha$ (both on the grid $\{0,1,2,4,8,16,32\}$).
Cells off the three sweep paths are masked.
Per-panel title color encodes paid (blue) vs.\ open-source (orange), matching Figure~\ref{fig:llm-top8-all19}.}
\label{fig:llm-rank-heatmap}
\end{figure}

\begin{figure}[h!]
\centering
\includegraphics[width=\textwidth]{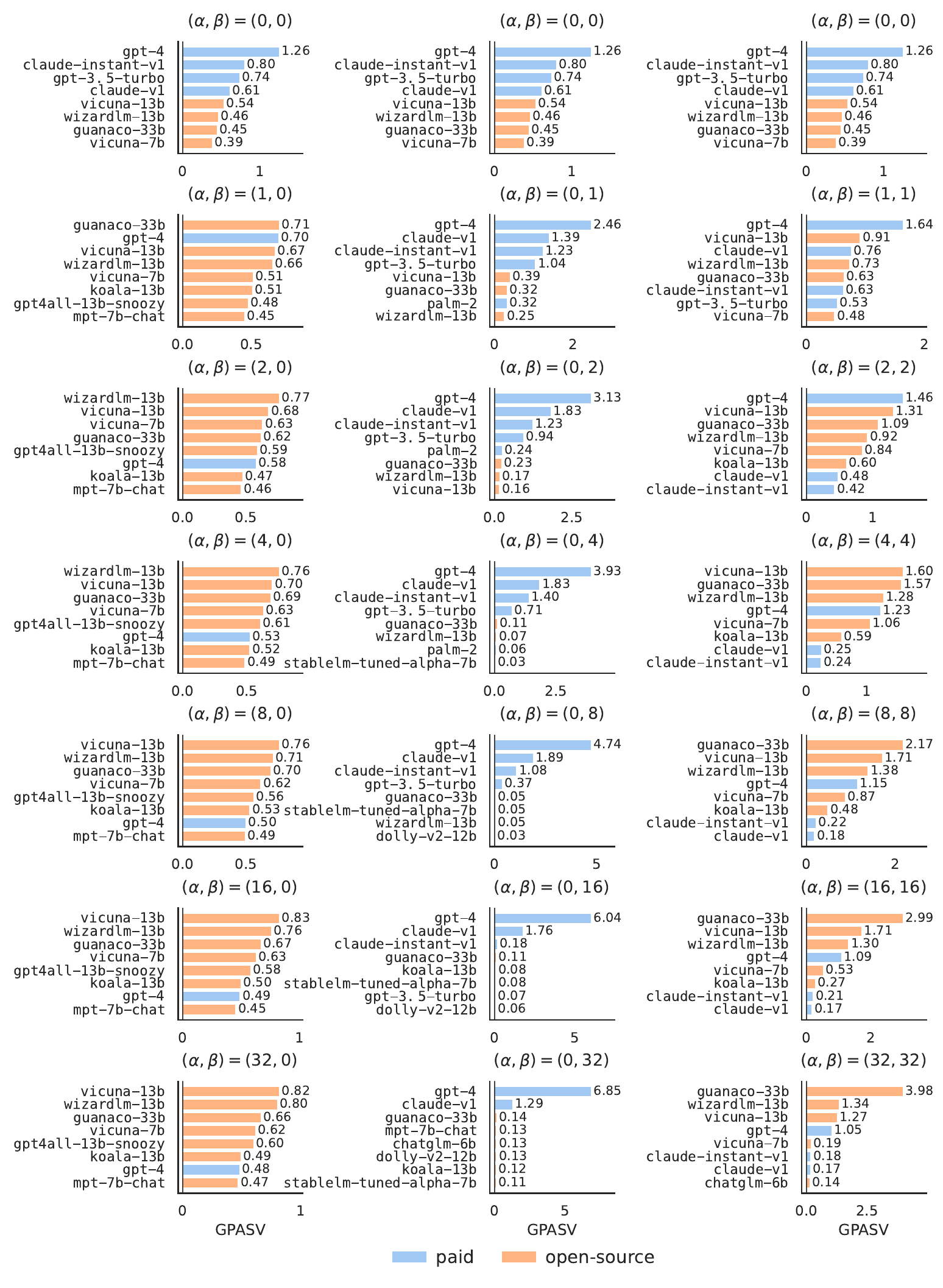}
\caption{Top-$8$-valued models across all $19$ priority regimes.
Rows: temperature values (baseline $(0,0)$ together with $\{1,2,4,8,16,32\}$).
Columns: $\alpha$-only sweep ($\beta=0$), $\beta$-only sweep ($\alpha=0$), joint sweep ($\alpha=\beta$).
Bar color encodes paid vs.\ open-source, matching the color scheme of Figure~\ref{fig:llm-top8} in the main text.}
\label{fig:llm-top8-all19}
\end{figure}

\FloatBarrier
\subsection{Robustness to Priority-Graph Uncertainty and a Ranking-then-PASV Baseline}
\label{app:application-llm-robustness}

\paragraph{Graph bootstrap.}
The Chatbot Arena graph is estimated from finite comparisons.
To assess how this affects GPASV, for each model pair we resample its win/loss/tie counts from a multinomial distribution with the observed proportions and the observed number of comparisons, and rebuild $\omega$ by the same rule as in Appendix~\ref{app:application-construction}; we repeat this $B=100$ times.
For each bootstrap graph, we recompute GPASV at $\alpha=0$ by self-normalized importance reweighting of a fixed pool of $6{,}000$ permutations, sampled on the original graph at the six nonzero $\beta$'s ($1{,}000$ each), with the cached utilities kept fixed.
Table~\ref{tab:llm-bootstrap} reports the rank agreement between each bootstrap GPASV and the original-graph GPASV computed from the same pool.
The rankings are stable for $\beta\le 16$ (mean Kendall $\tau\ge 0.91$).
The agreement decreases at $\beta=32$, where fewer pooled permutations can be reused (low ESS), so we read $\beta=16,32$ as descriptive.
These intervals capture priority-graph uncertainty only, conditional on the fixed utility and permutation pool.

\begin{table}[h!]
\centering
\small
\caption{Graph-bootstrap stability of GPASV rankings ($\alpha=0$, $B=100$). Each cell reports the bootstrap mean and the $[5\%,95\%]$ quantile interval.}
\label{tab:llm-bootstrap}
\begin{tabular}{@{}rccc@{}}
\toprule
$\beta$ & ESS & Kendall $\tau$ & Spearman $\rho$ \\
\midrule
$1$  & $1369$ [$1328$, $1412$] & $0.991$ [$0.979$, $1.000$] & $0.999$ [$0.997$, $1.000$] \\
$2$  & $1770$ [$1679$, $1848$] & $0.981$ [$0.958$, $1.000$] & $0.996$ [$0.992$, $1.000$] \\
$4$  & $1345$ [$1154$, $1520$] & $0.957$ [$0.926$, $0.979$] & $0.991$ [$0.983$, $0.997$] \\
$8$  & $805$ [$481$, $1088$]   & $0.935$ [$0.884$, $0.968$] & $0.985$ [$0.965$, $0.995$] \\
$16$ & $306$ [$112$, $553$]    & $0.912$ [$0.852$, $0.958$] & $0.976$ [$0.949$, $0.994$] \\
$32$ & $80$ [$13$, $209$]      & $0.841$ [$0.662$, $0.958$] & $0.939$ [$0.830$, $0.991$] \\
\bottomrule
\end{tabular}
\end{table}

\paragraph{Ranking-then-PASV baseline.}
A natural alternative to GPASV is to first reduce the Arena comparisons to a total order and then apply PASV.
We consider two such orders.
Let $c_{ij}$ be the number of comparisons between models $i$ and $j$, and $w_{ij}$ the points of $i$ against $j$ (one per win and one half per tie), so that $c_{ij}=w_{ij}+w_{ji}$.
The Bradley--Terry ranking \citep{bradley1952rank} orders models by decreasing $\widehat s_i$, where
$$
\widehat{\mathbf s}
=\operatorname*{arg\,max}_{\sum_i s_i=0}
\sum_{\{i,j\}\in\mathcal E}
\Bigl[w_{ij}\log\sigma(s_i-s_j)+w_{ji}\log\sigma(s_j-s_i)\Bigr],
\qquad \sigma(x)=\frac{1}{1+e^{-x}},
$$
and $\mathcal E$ is the set of retained pairs.
The round-robin ranking orders models by decreasing average win share over all observed pairs, $r_i=\frac{1}{n-1}\sum_{j\neq i}w_{ij}/c_{ij}$.
Each order defines a total-order DAG, on which PASV with $\lambda_i\equiv 1$ reduces to the marginal contributions along that single order.
Since there is no ground-truth attribution, this comparison measures how well these baselines approximate GPASV, rather than which method is correct.
Table~\ref{tab:llm-bt-pasv} shows that the two baselines disagree with each other (Kendall $\tau=0.418$), and both agree poorly with GPASV at $\beta=1$ ($\tau\le 0.174$); the agreement improves at $\beta=32$ but remains moderate ($\tau\le 0.459$).
That is, forcing the cyclic graph into a total order changes the valuation, and the result also depends on which projection is chosen.

\begin{table}[h!]
\centering
\small
\caption{Rank agreement between ranking-then-PASV baselines and GPASV at $\alpha=0$. GPASV$(\alpha,\beta)$ is computed on the original graph from the same permutation pool as in Table~\ref{tab:llm-bootstrap}.}
\label{tab:llm-bt-pasv}
\begin{tabular}{@{}lccc@{}}
\toprule
Comparison & Kendall $\tau$ & Spearman $\rho$ & Common top-5 \\
\midrule
BT--PASV vs.\ RR--PASV & $0.418$ & $0.398$ & $2/5$ \\
BT--PASV vs.\ GPASV$(0,1)$ & $0.174$ & $0.271$ & $3/5$ \\
RR--PASV vs.\ GPASV$(0,1)$ & $0.164$ & $0.230$ & $2/5$ \\
BT--PASV vs.\ GPASV$(0,32)$ & $0.459$ & $0.590$ & $3/5$ \\
RR--PASV vs.\ GPASV$(0,32)$ & $0.364$ & $0.554$ & $3/5$ \\
\bottomrule
\end{tabular}
\end{table}

\FloatBarrier
\subsection{Prompt Templates}
\label{app:application-llm-prompts}

\begin{qabox}[Aggregator prompt template: 1st turn]
[Instruction]
You are an impartial aggregator. Given a user prompt and multiple candidate answers, produce one final answer.

Your aggregation is merge-leaning:
- Primarily combine, edit, and reorganize what is already in the candidates.
- Do not add new substantive content or external facts.

Rules:
1) Follow the user's requested format, constraints, and style exactly.
2) Use only information explicitly stated in the candidates. You may add minimal connective wording for readability.
3) If candidates disagree and you cannot resolve it from the candidates, omit the claim or mark it as uncertain.
4) Remove redundancy and irrelevant parts; choose the clearest phrasing among candidates.
5) If JSON/code/strict format is required, keep it valid and do not introduce new APIs/libraries not present in candidates.
6) Two-turn consistency: if the prompt is multi-turn, treat the turn-1 aggregated answer as the assistant's previous message when producing the turn-2 aggregated answer.

After synthesizing, output ONLY the final aggregated answer, with no extra commentary.

[Question]
{x1}

<|The Start of Candidates|>
[The Start of Candidate 1 Answer]
{candidate answer 1}
[The End of Candidate 1 Answer]

...

[The Start of Candidate K Answer]
{candidate answer K}
[The End of Candidate K Answer]
<|The End of Candidates|>
\end{qabox}

\newpage
\begin{qabox}[Aggregator prompt template: 2nd turn]
[Instruction]
You are an impartial aggregator. Given a user prompt and multiple candidate answers, produce one final answer.

Your aggregation is merge-leaning:
- Primarily combine, edit, and reorganize what is already in the candidates.
- Do not add new substantive content or external facts.

Rules:
1) Follow the user's requested format, constraints, and style exactly.
2) Use only information explicitly stated in the candidates. You may add minimal connective wording for readability.
3) If candidates disagree and you cannot resolve it from the candidates, omit the claim or mark it as uncertain.
4) Remove redundancy and irrelevant parts; choose the clearest phrasing among candidates.
5) If JSON/code/strict format is required, keep it valid and do not introduce new APIs/libraries not present in candidates.
6) Two-turn consistency: if the prompt is multi-turn, treat the turn-1 aggregated answer as the assistant's previous message when producing the turn-2 aggregated answer.

After synthesizing, output ONLY the final aggregated answer, with no extra commentary.

<|The Start of Previous Conversation with User|>
### User:
{x1}

### Assistant:
{y1S}

### User:
{x2}
<|The End of Previous Conversation with User|>

<|The Start of Candidates|>
[The Start of Candidate 1 Answer]
{candidate answer 1}
[The End of Candidate 1 Answer]

...

[The Start of Candidate K Answer]
{candidate answer K}
[The End of Candidate K Answer]
<|The End of Candidates|>
\end{qabox}

\newpage
\begin{qabox}[Judge prompt template without reference: 1st turn]
[Instruction]
You are an impartial judge evaluating the quality of an assistant's response.
Evaluate the response based on the following criteria:
- helpfulness
- relevance
- accuracy
- clarity
- completeness

Think carefully about the response before rating it.
Provide at most three short sentences summarizing the main reason for your score.
Then, on a new final line, output the rating strictly in the format "[[rating]]", where rating is an integer from 1 to 10, for example: "[[5]]".
Do not use any other rating format.
Do not output anything after the final rating line.

[Question]
{x1}

[The Start of Assistant's Answer]
{y1S}
[The End of Assistant's Answer]
\end{qabox}

\begin{qabox}[Judge prompt template without reference: 2nd turn]
[Instruction]
You are an impartial judge evaluating the quality of an assistant's response.
Evaluate the response based on the following criteria:
- helpfulness
- relevance
- accuracy
- clarity
- completeness

Think carefully about the response before rating it.
Provide at most three short sentences summarizing the main reason for your score.
Then, on a new final line, output the rating strictly in the format "[[rating]]", where rating is an integer from 1 to 10, for example: "[[5]]".
Do not use any other rating format.
Do not output anything after the final rating line.

<|The Start of Assistant A's Conversation with User|>
### User:
{x1}

### Assistant A:
{y1S}

### User:
{x2}

### Assistant A:
{y2S}
<|The End of Assistant A's Conversation with User|>
\end{qabox}

\newpage
\begin{qabox}[Judge prompt template with reference: 1st turn]
[Instruction]
You are an impartial judge evaluating the quality of an assistant's response using a reference answer.
Evaluate the response by comparing it with the reference answer based on the following criteria:
- correctness
- helpfulness
- relevance
- clarity
- completeness

Think carefully about the response before rating it.
If the assistant's response misses important points or contains mistakes compared with the reference answer, take that into account.
Provide at most three short sentences summarizing the main reason for your score.
Then, on a new final line, output the rating strictly in the format "[[rating]]", where rating is an integer from 1 to 10, for example: "[[5]]".
Do not use any other rating format.
Do not output anything after the final rating line.

[Question]
{x1}

[The Start of Reference Answer]
{r1}
[The End of Reference Answer]

[The Start of Assistant's Answer]
{y1S}
[The End of Assistant's Answer]
\end{qabox}

\newpage
\begin{qabox}[Judge prompt template with reference: 2nd turn]
[Instruction]
You are an impartial judge evaluating the quality of an assistant's response using a reference answer.
Evaluate the response by comparing it with the reference answer based on the following criteria:
- correctness
- helpfulness
- relevance
- clarity
- completeness

Think carefully about the response before rating it.
If the assistant's response misses important points or contains mistakes compared with the reference answer, take that into account.
Provide at most three short sentences summarizing the main reason for your score.
Then, on a new final line, output the rating strictly in the format "[[rating]]", where rating is an integer from 1 to 10, for example: "[[5]]".
Do not use any other rating format.
Do not output anything after the final rating line.

<|The Start of Reference Answer|>
### User:
{x1}

### Reference answer:
{r1}

### User:
{x2}

### Reference answer:
{r2}
<|The End of Reference Answer|>

<|The Start of Assistant A's Conversation with User|>
### User:
{x1}

### Assistant A:
{y1S}

### User:
{x2}

### Assistant A:
{y2S}
<|The End of Assistant A's Conversation with User|>
\end{qabox}

\end{document}